\documentclass[11pt]{article}

\usepackage{microtype}
\usepackage{graphicx}
\usepackage{booktabs}

\usepackage{hyperref}
\usepackage[table]{xcolor}
\usepackage{float}

\usepackage[T1]{fontenc}
\usepackage{lmodern}

\usepackage{multirow}

\newcommand{\declarecolor}[2]{\definecolor{#1}{RGB}{#2}\expandafter\newcommand\csname #1\endcsname[1]{\textcolor{#1}{##1}}}
\declarecolor{White}{255, 255, 255}
\declarecolor{Black}{0, 0, 0}
\declarecolor{Maroon}{128, 0, 0}
\declarecolor{Coral}{255, 127, 80}
\declarecolor{Red}{182, 21, 21}
\declarecolor{LimeGreen}{50, 205, 50}
\declarecolor{DarkGreen}{0, 80, 0}
\declarecolor{Purple}{180, 60, 100}
\declarecolor{Navy}{0, 0, 128}
\declarecolor{LightBlue}{84, 101, 202}

\definecolor{mydarkblue}{rgb}{0.10,0.18,0.42}

\declarecolor{DarkCrimson}{120, 8, 28}
\declarecolor{DarkTeal}{0, 85, 85}
\declarecolor{DeepBurgundy}{140, 20, 40}
\declarecolor{SteelBlue}{55, 70, 130}
\declarecolor{Gold}{212, 175, 55}
\declarecolor{MutedBronze}{140, 120, 83}
\declarecolor{GoldInk}{51, 42, 12}
\declarecolor{BurnishedGold}{179, 119, 0}
\hypersetup{
    pdftitle={},
    pdfkeywords={},
    pdfborder=0 0 0,
    pdfpagemode=UseNone,
    colorlinks=true,
    linkcolor=darkred,
    citecolor=deepteal,
    filecolor=darkred,
    urlcolor=Purple,
}

\usepackage[longnamesfirst]{natbib}

\usepackage{amsmath}
\usepackage{amssymb}
\usepackage{mathtools}
\usepackage{amsthm}
\usepackage{bbm}
\usepackage{bm}
\usepackage{amsfonts}
\usepackage{frcursive}
\usepackage{alphabeta} 

\usepackage{booktabs}
\usepackage[table]{xcolor}
\usepackage{siunitx}

\usepackage{tocloft}

\usepackage{enumitem}
\setlist[enumerate]{itemsep=1mm,parsep=0mm,topsep=.5mm}
\setlist[itemize]{itemsep=1mm,parsep=0mm,topsep=.5mm}

\usepackage{nicefrac}
\usepackage{fullpage}

\usepackage[capitalize,noabbrev,nameinlink]{cleveref}

\usepackage{url}
\usepackage[ruled,vlined,linesnumbered]{algorithm2e}
\usepackage{tabularx}
\usepackage{array}
\usepackage{titletoc}

\definecolor{deepteal}{RGB}{0,95,93}
\definecolor{darkred}{RGB}{139,0,0}
\definecolor{darkgreen}{RGB}{0,80,0}
\definecolor{figgreen}{RGB}{26,74,46}
\definecolor{figwine}{RGB}{107,26,42}
\definecolor{figgreenl}{RGB}{109,184,138}

\theoremstyle{plain}
\newtheorem{theorem}{Theorem}[section]
\newtheorem{lemma}[theorem]{Lemma}
\newtheorem{proposition}[theorem]{Proposition}

\theoremstyle{definition}
\newtheorem{assumption}{Assumption}
\newtheorem{definition}{Definition}
\newtheorem{remark}{Remark}

\newcolumntype{Y}{>{\raggedright\arraybackslash}X}
\newcolumntype{L}[1]{>{\raggedright\arraybackslash}p{#1}}

\newcommand{\E}{\mathbb{E}}
\newcommand{\R}{\mathbb{R}}
\newcommand{\pr}{\mathbb{P}}
\newcommand{\indep}{\perp\!\!\!\perp}
\DeclareMathOperator{\rank}{rank}
\DeclareMathOperator{\diag}{diag}
\DeclareMathOperator{\range}{\mathcal{R}}
\DeclareMathOperator{\spec}{spec}

\title{The Spectral Structure of Latent Treatment Effects}

\author{
  \large Hamza Virk$^{1}$\thanks{Corresponding Author.} \qquad \large Bijan Mazaheri$^{1,2}$ \qquad \large Yihren Wu$^{3}$ \\[0.95em]
  \normalsize $^{1}$Dartmouth College \\[0.1em]
  \normalsize $^{2}$Broad Institute of MIT and Harvard \\[0.1em]
  \normalsize $^{3}$Hofstra University \\[0.2em]
  \normalsize \texttt{\{hamza.a.virk.th, bijan.h.mazaheri\}@dartmouth.edu, yihren.wu@hofstra.edu}
}

\date{}

\begin{document}

\maketitle

\begin{abstract}
Identifying heterogeneous treatment effects under unobserved confounding is central in observational causal inference. In proxy models with a discrete latent confounder, prior Synthetic Potential Outcomes (SPO) \citep{mazaheri2024synthetic} recover the mixture of treatment effects through recursively constructed scalar moments. We show that this sequence is one projection of a more fundamental object. Under the same population factorization assumptions, there is an exact compressed observable operator: after projecting onto the shared proxy signal subspace, the difference of two treatment-arm quotient operators is similar to the diagonal matrix of latent treatment effects. Its eigenvalues are the latent effects; its lifted left eigenvectors, after anchor normalization, recover the target-proxy feature matrix and then the latent mixture proportions. Every scalar SPO moment is a bilinear functional of a power of this operator. The resulting estimator handles overcomplete proxy systems, replaces high-order scalar inversion with finite-dimensional spectral analysis, and admits high-probability first-order perturbation bounds for treatment effects, feature rows, and simplex-projected mixture weights.
\end{abstract}

\section{Introduction}

\begin{figure}[t]
    \centering
    \includegraphics[width=\textwidth]{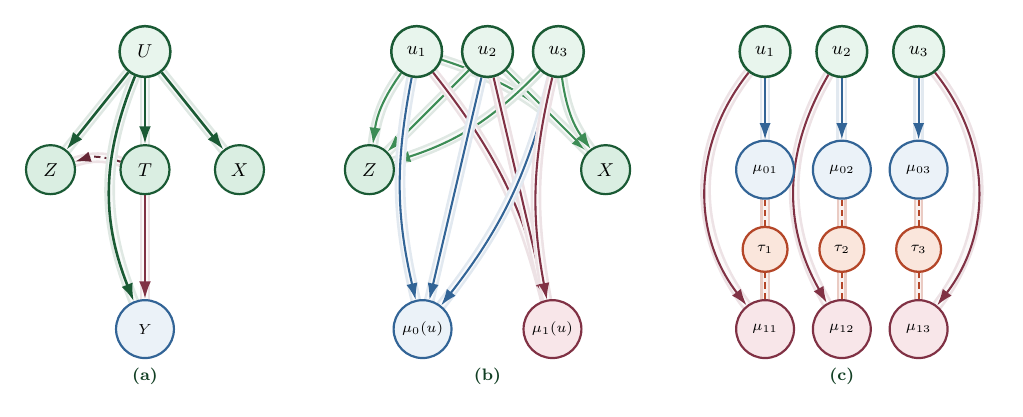}
    \caption[Causal triptych for the latent proxy mixture model]{
    Causal triptych for the latent proxy mixture model. (a) depicts the proxy causal structure: a latent confounder \(U\) drives treatment \(T\), the target proxy \(X\), the reference proxy \(Z\), and the observed outcome \(Y\). The dashed arrow from \(T\) to \(Z\) indicates that the reference-proxy feature law may be treatment-arm-specific, as encoded by \(\mathbf A_t(:,u)=\mathbb E[Z\mid U=u,T=t]\). (b) unrolls the latent classes \(u_j\), each of which generates shared target-proxy features, treatment-arm-specific reference-proxy features, and conditional potential-outcome means \(\mu_t(u_j)=\mathbb E[Y^{(t)}\mid U=u_j]\). (c) shows the classwise counterfactual pairing recovered by the compressed spectral operator: each latent treatment effect is the paired contrast \(\tau_j=\mu_1(u_j)-\mu_0(u_j)\).
    }
    \label{fig:causal_triptych}
\end{figure}

\textcursive{A brief glimpse into the past}.\vspace{-0.2em} The problem of drawing causal conclusions from non-experimental data is
as old as empirical science itself. In the 1840s, John Snow traced the
Broad Street cholera outbreak to a contaminated water pump in London, a
landmark feat of causal reasoning conducted entirely through
observational data and without a randomized trial in sight
\citep{snow1855mode}. Decades later, Karl Pearson formalized the
correlation coefficient and yet famously conflated correlation with
causation, a confusion that would haunt empirical research for
generations \citep{pearson1896regression}. The mid-twentieth century
placed the distinction between association and causation on rigorous
mathematical footing. Jerzy Neyman's potential-outcomes framework
\citep{neyman1923application} and, independently, Donald Rubin's later
formalization \citep{rubin1974estimating} gave statisticians a language
for asking what would have happened had the treatment been different.
The simultaneous-equations tradition in econometrics, crystallized at
the Cowles Commission \citep{haavelmo1944probability}, attacked the
same problem from the angle of structural identification and produced
the instrumental-variable methods still central to applied economics
today. Pearl's graphical reformulation \citep{pearl2009causality}
unified the potential-outcomes and structural-equations traditions into
the modern theory of causal inference. One adversary survived all of
these frameworks. The unobserved confounder, a hidden variable that
simultaneously influences who receives treatment and what outcome they
experience, biases estimates and produces the statistical paradoxes
that observational reasoning has spent more than a century trying to
neutralize \citep{simpson1951interpretation}.

Modern proximal causal inference confronts this adversary structurally.
Auxiliary proxy variables, often called negative controls of the
exposure or outcome, are leveraged to bypass unmeasured confounding
without ever observing the confounder itself \citep{lipsitch2010negative,
miao2018identifying, tchetgen2020introduction}. On the marginalized
Average Treatment Effect (ATE), this program has largely succeeded.
Contemporary applications such as precision medicine, targeted economic
policy, and vaccine-response heterogeneity demand more than an average.
They demand to know how treatment effects vary across unobserved
strata. Recovering this Mixture of Treatment Effects (MTE) is strictly
harder, because the strata themselves are latent. The target is a
distribution over conditional potential outcomes whose indexing
variable is never observed.

\citet{mazaheri2024synthetic} resolved the central identifiability
question for this regime. Under proxy conditional independencies and a
discrete latent confounder, they showed that the MTE distribution is
recoverable from observable cross-moments alone, with no need to
reconstruct the joint mixture. Their Synthetic Potential Outcomes (SPO)
framework realizes this identifiability through a constructive route.
A scalar sequence of synthetic treatment-effect moments is built
recursively from observable matrices, and the latent treatment-effect
support is then extracted from that sequence by a Hankel pencil.

We argue that the scalar moment sequence is not the structure of the
problem. It is one projection of it. The same causal model admits an
exact observable-operator representation under the same assumptions,
and the identification result lives most naturally inside it. After
extracting the $k$-dimensional shared signal subspace of the proxy
system, the treated and control moment factorizations collapse into
compressed operators that share a common similarity basis $\mathbf{R}$. Their difference
$\Delta\widetilde{\mathbf{Q}} = \widetilde{\mathbf{Q}}_1 - \widetilde{\mathbf{Q}}_0$ is similar to the diagonal of latent treatment effects, and its
eigenvalues are exactly $\{\tau(u)\}_{u=1}^{k}$. The matching of
$\E[Y^{(1)} \mid U = u]$ with $\E[Y^{(0)} \mid U = u]$ on the same
latent stratum happens automatically inside this shared similarity,
with no support-recovery problem to solve. The construction also
extends the scope of MTE-level SPO identification beyond the square
proxy regime. The recursive scalar moment algorithm of
\citet{mazaheri2024synthetic} inverts a $k \times k$ proxy moment
matrix at every step of its recursion and therefore requires
$d_x = d_z = k$, whereas the spectral framework projects onto the
$k$-dimensional signal subspace of a stacked observable matrix
before forming any inverse, and so accepts overcomplete proxy
systems $d_x, d_z \ge k$ directly.

The original scalar moments sequence \citet{mazaheri2024synthetic} follows
from this representation. For
observable vectors $\mathbf{a}$ and $\mathbf{c}$,
\[
m_{\ell} = \E[\tau(U)^{\ell}]
= \mathbf{a}^\top (\Delta \widetilde{\mathbf{Q}})^{\ell}\mathbf{c},
\qquad \ell \ge 0.
\]
Every moment used in the recursive construction is a bilinear
functional of a power of the same $k \times k$ operator. The compressed difference operator therefore identifies the latent treatment-effect support through its eigenvalues. Its lifted left eigenvectors, together with the anchor coordinate \(X_1=1\) and the marginal mean \(\E[X]\), identify the target-proxy feature matrix and the latent mixture proportions. Equivalently, the operator and the observable anchor pair induce the spectral measure whose atoms and masses recover the MTE distribution. The operator is the canonical object, and the moment sequence is its trace through one particular pairing.

\paragraph{Contributions.}
\textbf{First}, we give a spectral identification theorem for the latent MTE on overcomplete proxy systems. The latent treatment effects appear as the eigenvalues of a single \(k \times k\) compressed difference operator
\(\Delta\widetilde{\mathbf{Q}}\). Under spectral separation, the corresponding lifted left eigenvectors, after normalization by the anchor coordinate \(X_1=1\), recover the target-proxy feature matrix \(\mathbf{B}\); the latent mixture proportions then follow from the observable marginal equation \(\E[X]=\mathbf{B}\bm p\). Unlike the inverse-based SPO recursion, which is naturally formulated in the square proxy regime \(d_x=d_z=k\), this construction accepts overcomplete proxy systems \(d_x,d_z\ge k\) directly.

\textbf{Second}, we prove an operator-moment equivalence theorem. The scalar moment sequence underlying recursive synthetic constructions is generated by powers of the compressed difference operator, so the operator together with the observable anchor pair determines every polynomial functional of the latent treatment-effect law

\textbf{Third}, we give a population geometric characterization of the model. Exact rank and row-space identities turn latent dimensionality, strict positivity, and treatment-effect homogeneity into structural properties of observable matrices, rather than assumptions invoked from outside the data.

\textbf{Lastly}, we provide finite-sample guarantees and empirical confirmation. We prove high-probability \(n^{-1/2}\) perturbation bounds for the recovered eigenvalues, lifted eigenspaces, and also for the actual algorithmic outputs obtained after anchor normalization and simplex projection: the feature matrix \(\widehat{\mathbf B}\) and mixture weights \(\widehat{\bm p}\). Our experiments show that the compressed spectral estimator is substantially more stable than the recursive scalar implementation.

\section{Related Work}
\label{app:related_work}

We delineate our theoretical contributions relative to the following streams of literature.

\paragraph{Proximal Causal Inference and Negative Controls.}
The use of proxy variables to adjust for unmeasured confounding has accelerated following the formalization of proximal causal learning \citep{miao2018identifying, tchetgen2020introduction}. Central to this framework is the identification of a treatment-inducing proxy (negative control exposure) and an outcome-inducing proxy (negative control outcome) \citep{lipsitch2010negative, shi2020selective}. These approaches point-identify causal effects through Fredholm integral equations, kernel instrumental-variable estimators, and modern saddle-point, spectral, and density-ratio-free proxy formulations for bridge-function estimation \citet{singh2019kernel,mastouri2021proximal,miao2018confounding,bozkurt2025densityratio}. \citep{guo2026comparingproxymethodscausal} compare bridge-equation and array-decomposition proxy methods; ours is the latter. Closest in spectral spirit, \citet{pmlr-v258-sun25d} use low-rank conditional-expectation representations for IV and proxy causal learning. Whereas most proximal inference targets marginal effects, we point-identify the treatment-effect mixture across unobserved strata under a discrete-confounder assumption.

\paragraph{Spectral Latent Variable Models and Observable Operators.}
Our methodology is deeply connected to the literature on spectral learning for latent variable models and Observable Operator Models (OOMs) \citet{jaeger2000observable, hsu2013learning, anandkumar2014tensor}. In theoretical machine learning, simultaneous diagonalization and spectral tensor methods are classically deployed to recover the hidden transition dynamics of Hidden Markov Models (HMMs) or symmetric mixtures of products. However, our work adapts these algebraic tools to a strictly causal regime. Rather than recovering a sequence of temporal transitions or a symmetric joint probability tensor, we show that the asymmetric directed acyclic graph (DAG) of proximal causal inference permits the construction of a targeted low-rank quotient operator. After isolating the signal subspace, the spectrum of this operator recovers the counterfactual differences without requiring reconstruction of the full high-dimensional joint distribution. This distinguishes our causal operator from generic OOM applications.

\paragraph{Latent Mixture Identifiability and Tensor Decompositions.}
The identifiability of finite mixture models is a foundational problem in theoretical statistics. \citet{kruskal1977three} and \citet{allman2009identifiability} established the geometric conditions under which multi-way arrays uniquely decompose, identifying latent structural parameters. Within causal inference, \citet{wang2019blessings} proposed using latent factor models to deconfound multiple treatments (the ``blessings of multiple causes''), though this approach was critically examined regarding strict identifiability guarantees \citet{ogburn2019comment, ogburn2020counterexamples}. Modern algorithms have vastly improved sample complexity for specific mixture graphs and related finite-class confounding models \citet{gordon2023causal,gordon2023identification,gordon2021source,rabani2014learning,mazaheri2023latentclassconfounding}. While Canonical Polyadic (CP) tensor decompositions generally aim to recover the full joint probability distribution and often scale poorly with condition number dependencies, our work demonstrates that targeting causal MTEs structurally bypasses the need for full tensor reconstruction, enabling identification strictly through low-dimensional subspace projections.

\paragraph{Heterogeneous Treatment Effects (HTEs) and Latent Subgroups.}
A vast literature exists for estimating conditional average treatment effects (CATEs) when the adjustment set is fully observed \citep{imai2011estimation, xie2012estimating}. Modern approaches utilize flexible machine learning meta-learners \citep{kunzel2019causaltoolbox, nie2021quasi, wendling2018comparing} and causal forests \citep{wager2018estimation}. Recent work also studies treatment-effect uncertainty beyond mean CATEs under unobserved confounding \citep{xu2026aleatoric}. Our work addresses a strictly harder regime where the heterogeneity is governed by a variable $U$ that is entirely unobserved. In this regime, standard distance-based clustering or nearest-neighbor matching fails fundamentally because the observable proxy distributions overlap topologically. Identifying these effects requires structural algebraic methods to mathematically untangle the counterfactual states prior to analyzing heterogeneity \citep{pearl2022detecting, loh2022evaluating, lyu2023estimating, kim2015mixture, suk2021hybridizing}.

\paragraph{Synthetic Interventions and Polynomial Moment Methods.}
Our methodology shares conceptual lineage with synthetic controls
\citet{abadie2010synthetic, abadie2021using} and synthetic interventions
\citet{squires2022causal, agarwal2023causal}, which construct unobserved
counterfactuals via linear combinations of observed units. Our direct
structural predecessor, \citet{mazaheri2024synthetic}, established MTE
identifiability through Synthetic Potential Outcomes (SPOs). Their
constructive route expresses the latent treatment-effect distribution
through recursively synthesized scalar moments and recovers its support
by a Hankel pencil. We show that this scalar sequence is one projection
of a deeper object. Every latent treatment-effect moment is a bilinear
functional of a power of a single compressed observable operator
$\Delta\widetilde{\mathbf{Q}}$, so the operator alone determines the
entire SPO identification target. Working with this operator directly
reduces high-order scalar inversion to finite-dimensional spectral
analysis over second- and third-order observable matrices, and
naturally accommodates rectangular proxy systems through shared-subspace
compression.

\section{Preliminaries}

We consider a latent Mixture of Treatment Effects (MTE) setting. Let $T \in \{0, 1\}$ denote a binary treatment and $Y \in \R$ denote the observed continuous or discrete outcome. We assume the presence of an unobserved, discrete latent confounder $U \in \{1, \dots, k\}$. For the population identification results below, the latent dimension $k$ is treated as known. In applications where $k$ is unknown, one may select it using a separate rank-selection procedure applied to the stacked observable moment matrix.

We observe two vectors of proxy features: a reference proxy $Z \in \R^{d_z}$ and a target proxy $X \in \R^{d_x}$, written explicitly as
\[
Z=(Z_1,\ldots,Z_{d_z})^\top,
\qquad
X=(X_1,\ldots,X_{d_x})^\top.
\]
We anchor the probability scale by defining the first element of the target proxy as a constant,
\[
X_1=1.
\]
We explicitly allow for overcomplete feature spaces, such that $d_x, d_z \ge k$.

\subsection{Structural Causal Assumptions}

To isolate the causal parameters from spurious observational dependence, we assume the conditional independencies used by the proxy-factorization framework of \citet{mazaheri2024synthetic}. These relations are exactly the ones needed to derive the observable matrix factorizations below.

\begin{assumption}[Proxy Causal Structure] \label{ass:graph}
The observed variables satisfy
\begin{equation}
Z \indep (X,Y) \mid (T,U),
\qquad
X \indep (Y,T) \mid U.
\end{equation}
In addition, we assume causal consistency, $Y=Y^{(T)}$, and latent ignorability,
\[
Y^{(t)} \indep T \mid U,
\qquad t \in \{0,1\}.
\]
\end{assumption}

Equivalently, conditional on $(T,U)$, the reference proxy $Z$ is separated from both the outcome and the target proxy, while the target proxy $X$ depends on treatment and outcome only through the latent state $U$. These relations yield the factorizations in Equations \ref{eq:fac1} and \ref{eq:fac2} by repeated application of the law of total expectation.

\begin{figure}[t]
    \centering
    \includegraphics[width=\textwidth]{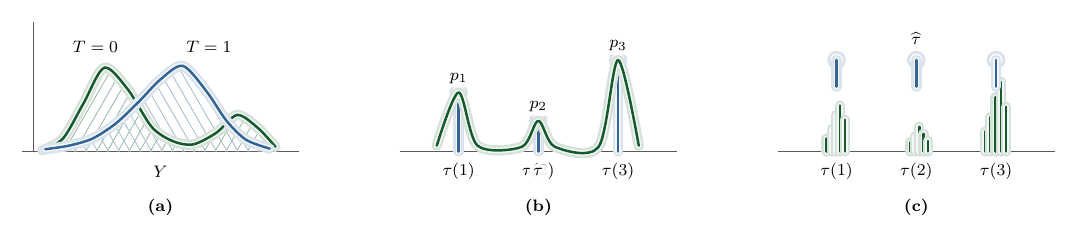}
    \caption{
    Distributional view of latent treatment effects.
    (a) The observed arm-specific outcome distributions need not reveal the latent causal support.
    (b) The latent treatment-effect law is discrete, supported on $\{\tau(u)\}_{u=1}^k$ with masses $\Pr(U=u)$.
    (c) Spectral recovery targets this latent support through the recovered values $\widehat{\tau}$.
    }
    \label{fig:latent_distribution}
\end{figure}

\subsection{Matrix Definitions and Factorization}

Our goal is to recover unobserved causal parameters from strictly observable data. For each treatment arm $t \in \{0, 1\}$, we define the empirically observable cross-moment matrix $\mathbf{M}_{ZX|t} \in \R^{d_z \times d_x}$ and the outcome-weighted cross-moment matrix $\mathbf{M}_{ZXY|t} \in \R^{d_z \times d_x}$ as:
\begin{align}
    \mathbf{M}_{ZX|t} &= \E[Z X^\top \mid T=t], \\
    \mathbf{M}_{ZXY|t} &= \E[Z X^\top Y \mid T=t].
\end{align}
To build our geometric framework, we structure the corresponding unobservable latent parameters as matrices. Define $\mathbf{A}_t \in \R^{d_z \times k}$ and $\mathbf{B} \in \R^{d_x \times k}$ such that their $u$-th columns represent the conditional expectations of the proxy features given the latent state $U=u$:

Writing $Z=(Z_1,\ldots,Z_{d_z})^\top$ and $X=(X_1,\ldots,X_{d_x})^\top$,
\[
\mathbf{A}_t
=
\begin{bmatrix}
\E[Z_1 \mid U=1,T=t] & \E[Z_1 \mid U=2,T=t] & \cdots & \E[Z_1 \mid U=k,T=t] \\
\E[Z_2 \mid U=1,T=t] & \E[Z_2 \mid U=2,T=t] & \cdots & \E[Z_2 \mid U=k,T=t] \\
\vdots & \vdots & \ddots & \vdots \\
\E[Z_{d_z} \mid U=1,T=t] & \E[Z_{d_z} \mid U=2,T=t] & \cdots & \E[Z_{d_z} \mid U=k,T=t]
\end{bmatrix},
\]
and
\[
\mathbf{B}
=
\begin{bmatrix}
\E[X_1 \mid U=1] & \E[X_1 \mid U=2] & \cdots & \E[X_1 \mid U=k] \\
\E[X_2 \mid U=1] & \E[X_2 \mid U=2] & \cdots & \E[X_2 \mid U=k] \\
\vdots & \vdots & \ddots & \vdots \\
\E[X_{d_x} \mid U=1] & \E[X_{d_x} \mid U=2] & \cdots & \E[X_{d_x} \mid U=k]
\end{bmatrix}.
\]
Because $X \indep T \mid U$, the target feature matrix $\mathbf{B}$ does not depend on the treatment arm. Let $\mathbf{D}_{U|t} \in \R^{k \times k}$ and $\mathbf{D}_{Y|t} \in \R^{k \times k}$ denote diagonal matrices containing the latent conditional probabilities and conditional potential outcomes, respectively:

That is,
\[
\mathbf{D}_{U|t}
=
\begin{bmatrix}
\pr(U=1 \mid T=t) & 0 & \cdots & 0 \\
0 & \pr(U=2 \mid T=t) & \cdots & 0 \\
\vdots & \vdots & \ddots & \vdots \\
0 & 0 & \cdots & \pr(U=k \mid T=t)
\end{bmatrix},
\]
and
\[
\mathbf{D}_{Y|t}
=
\begin{bmatrix}
\E[Y^{(t)} \mid U=1] & 0 & \cdots & 0 \\
0 & \E[Y^{(t)} \mid U=2] & \cdots & 0 \\
\vdots & \vdots & \ddots & \vdots \\
0 & 0 & \cdots & \E[Y^{(t)} \mid U=k]
\end{bmatrix}.
\]
By systematically applying the law of total expectation and leveraging the stated conditional independencies, the observable proxy matrices factorize algebraically into the unobservable parameters as follows:
\begin{align}
    \mathbf{M}_{ZX|t} &= \mathbf{A}_t \mathbf{D}_{U|t} \mathbf{B}^\top, \label{eq:fac1} \\
    \mathbf{M}_{ZXY|t} &= \mathbf{A}_t \mathbf{D}_{U|t} \mathbf{D}_{Y|t} \mathbf{B}^\top. \label{eq:fac2}
\end{align}

Figure~\ref{fig:proxy_factorization} depicts this factorization.

To ensure the causal retrieval problem is structurally well-posed and the latent classes are linearly distinguishable, we formalize the necessary rank condition underlying the proxy framework.

\begin{figure}[t]
    \centering
    \includegraphics[width=1.0\textwidth]{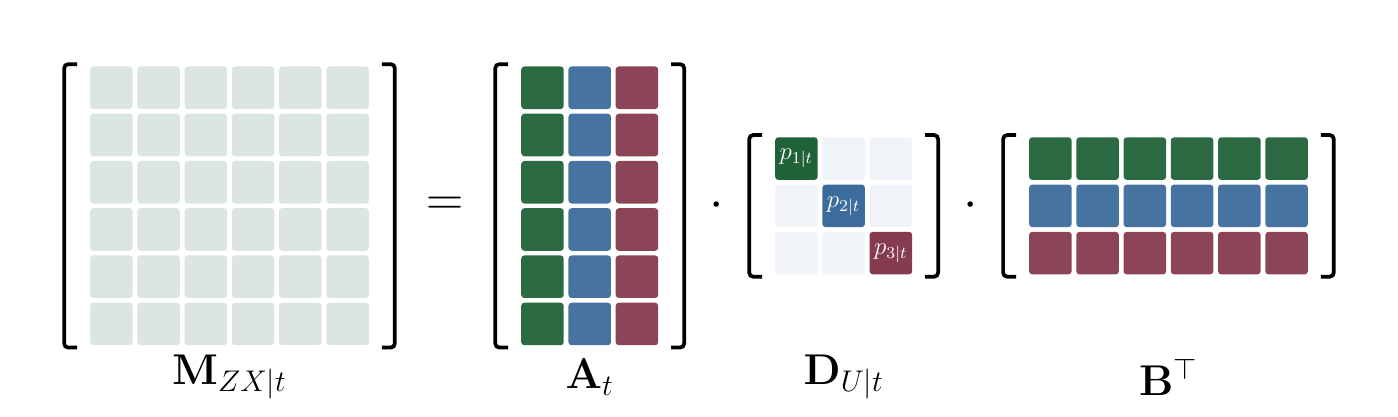}
    \caption{Visual representation of the proxy matrix factorization $\mathbf{M}_{ZX|t} = \mathbf{A}_t \mathbf{D}_{U|t} \mathbf{B}^\top$. The diagram illustrates how the observable cross-moment matrix is constrained by the $k$-dimensional latent space.}
    \label{fig:proxy_factorization}
\end{figure}

\begin{assumption}[Proxy Richness] \label{ass:richness}
        The proxy feature spaces are sufficiently informative such that $d_z \ge k$ and $d_x \ge k$. Furthermore, the unobservable conditional expectation matrices $\mathbf{A}_0$, $\mathbf{A}_1$, and $\mathbf{B}$ have full column rank $k$.
\end{assumption}

\section{The Geometry of Latent Confounding}
\label{sec:geometry}

Before formulating the estimation algorithm, it is instructive to consider the geometric implications of our causal assumptions on the observable moment spaces. Classical algebraic methods for estimating MTEs recursively substitute the factorizations in Equation \ref{eq:fac1} and Equation \ref{eq:fac2} into scalar moment sequences. In this section, we abandon scalar sequence generation entirely. Instead, we analyze the structural properties of these factorizations directly through the lens of matrix geometry. By understanding the geometric subspaces spanned by the proxy matrices, we establish powerful population-level diagnostics for causal complexity and distributional overlap.

\subsection{Intrinsic Causal Dimensionality}

In observational practice, the true number of latent confounding classes $k$ is rarely known a priori. Before executing any algorithm, a practitioner must have a principled way to ascertain this dimension. If Assumption \ref{ass:richness} holds, we demonstrate that the intrinsic dimension $k$ manifests exactly as the rank of the observable matrices, provided the latent classes are actively populated.

\begin{theorem}[Intrinsic Causal Dimensionality] \label{thm:dimensionality}
Let Assumption \ref{ass:richness} hold. Assume that every latent class exists with strictly positive probability under at least one treatment arm (i.e., $\max_{t \in \{0,1\}} \pr(U=u \mid T=t) > 0$ for all $u$). The intrinsic dimensionality $k$ of the unobserved confounding space is exactly identified by the algebraic rank of the concatenated cross-moment matrix:
\begin{equation}
    k = \rank \left( \begin{bmatrix} \mathbf{M}_{ZX|0} \\ \mathbf{M}_{ZX|1} \end{bmatrix} \right).
\end{equation}
\end{theorem}

Theorem \ref{thm:dimensionality} identifies the latent causal dimension exactly through the rank of the stacked observable cross-moment matrix. This gives a natural spectral target for empirical rank selection.

\subsection{Structural Signatures of Positivity}

Estimating valid causal effects requires the strict overlap condition $0 < \pr(T=1 \mid U=u) < 1$. In the present factorized proxy model, positivity induces an exact row-space statement at the population level. Let $\range(\mathbf{M})$ denote the column space (image) of a matrix.

\begin{theorem}[Geometric Signature of Positivity] \label{thm:positivity}
Under Assumption \ref{ass:richness} and assuming marginal support $\pr(U=u) > 0$ for all $u$:
\begin{enumerate}
    \item If strict latent positivity holds, i.e., $\pr(T=t \mid U=u) > 0$ for all $u \in \{1,\dots,k\},\; t \in \{0,1\}$, then
    \begin{equation}
        \range(\mathbf{M}_{ZX|1}^\top) = \range(\mathbf{M}_{ZX|0}^\top) = \range(\mathbf{B}).
    \end{equation}
    \item If positivity fails for some pair $(u^*,t^*)$, i.e., $\pr(T=t^* \mid U=u^*) = 0$, then
    \begin{equation}
        \rank(\mathbf{M}_{ZX|t^*}) < k,
        \qquad
        \range(\mathbf{M}_{ZX|t^*}^\top) \subsetneq \range(\mathbf{B}).
    \end{equation}
    In particular, the treated and control row spaces cannot both coincide with $\range(\mathbf{B})$.
\end{enumerate}
\end{theorem}

Theorem \ref{thm:positivity} identifies positivity through an exact population row-space identity induced by the proxy factorization. Under the model, overlap is therefore equivalent to full participation of every latent class in both treatment-specific proxy row spaces.

\section{Spectral Identification of Potential Outcomes}
\label{sec:spectral_outcomes}

We now turn to the central task: identifying the latent conditional potential outcomes $\E[Y^{(t)} \mid U]$ and the induced treatment-effect distribution. Rather than recursively generating scalar moments, we construct an observable operator whose similarity class coincides with the latent diagonal potential-outcome operator. This formulation recovers the spectral support directly and, as shown below, generates the full synthetic moment sequence through operator powers. We first define the mapping in the ambient feature space to expose the algebraic cancellation, and then compress it to the shared signal subspace to obtain the intrinsic $k \times k$ representation. Let $\dagger$ denote the Moore--Penrose pseudo-inverse
\citep{penrose1955generalized,benisrael2003generalized}. 

\subsection{The Ambient Quotient Operator}

We begin by exhibiting the observable operator that encodes the latent diagonal structure.

\begin{definition}[Ambient Spectral Operator] \label{def:quotient}
For $t \in \{0, 1\}$, define the ambient spectral operator $\mathbf{Q}_t \in \R^{d_x \times d_x}$ as:
\begin{equation}
    \mathbf{Q}_t \triangleq \mathbf{M}_{ZX|t}^\dagger \mathbf{M}_{ZXY|t}.
\end{equation}
\end{definition}

\begin{theorem}[Ambient Spectral Structure] \label{thm:spectral_identification}
Assume Assumption~\ref{ass:richness} and strict latent positivity hold. The ambient quotient operator satisfies
\begin{equation} \label{eq:clean_Q}
    \mathbf{Q}_t = (\mathbf{B}^\top)^\dagger \mathbf{D}_{Y|t} \mathbf{B}^\top.
\end{equation}
Moreover, $\mathbf{Q}_t (\mathbf{B}^\top)^\dagger = (\mathbf{B}^\top)^\dagger \mathbf{D}_{Y|t}$, so the $k$-dimensional subspace $\range(\mathbf{B}) = \range((\mathbf{B}^\top)^\dagger)$ is $\mathbf{Q}_t$-invariant. The restriction of $\mathbf{Q}_t$ to that subspace is similar to $\mathbf{D}_{Y|t}$, and its spectral values there are exactly $\left\{\E[Y^{(t)} \mid U=u]\right\}_{u=1}^k$ counted with multiplicity. Dually, the rows of $\mathbf{B}^\top$ form a left-eigenbasis for the signal component, with eigenvalues given by the diagonal entries of $\mathbf{D}_{Y|t}$.
\end{theorem}

\subsection{Shared Subspace Compression and Counterfactual Alignment}
\label{sec:shared_subspace}

While Theorem \ref{thm:spectral_identification} isolates the potential outcomes at the population level, the ambient representation is not intrinsic: it acts on $\R^{d_x}$ even though the causal dynamics have rank $k$. The shared-subspace compression below removes ambient null directions, places the treated and control operators in the same $k$-dimensional coordinate system, and yields the canonical representative on which counterfactual pairing is exact.

Let
\[
\mathbf{M}_{\mathrm{stack}}
\triangleq
\begin{bmatrix}
\mathbf{M}_{ZX|0}\\
\mathbf{M}_{ZX|1}
\end{bmatrix}
\in \R^{2d_z \times d_x}
\]
be the stacked proxy cross-moment matrix. By Theorems \ref{thm:dimensionality} and \ref{thm:positivity}, this matrix has rank $k$ and satisfies
\[
\range(\mathbf{M}_{\mathrm{stack}}^\top)=\range(\mathbf{B}).
\]
Hence its top-$k$ right singular subspace is exactly the row space of $\mathbf{B}^\top$. Let $\mathbf{V}_k \in \R^{d_x \times k}$ be any orthonormal basis of this subspace. Then there exists a nonsingular matrix $\mathbf{R} \in \R^{k \times k}$ such that:
\begin{equation} \label{eq:R_basis}
    \mathbf{B}^\top = \mathbf{R} \mathbf{V}_k^\top.
\end{equation}
By compressing our empirical data onto $\mathbf{V}_k$, we define the \textit{compressed causal operator} strictly within this stable, lower-dimensional $k \times k$ manifold.

\begin{definition}[Compressed Spectral Operator] \label{def:compressed_operator}
For $t \in \{0, 1\}$, define the compressed operator $\widetilde{\mathbf{Q}}_t \in \R^{k \times k}$ as the pseudo-inverse projection of the targeted moments onto the shared basis:
\begin{equation}
    \widetilde{\mathbf{Q}}_t \triangleq (\mathbf{M}_{ZX|t} \mathbf{V}_k)^\dagger (\mathbf{M}_{ZXY|t} \mathbf{V}_k).
\end{equation}
\end{definition}

\begin{assumption}[Spectral Separation] \label{ass:separation}
    The conditional treatment effects are mutually distinct across latent classes. That is, for all $u \neq v$, $\tau(u) \triangleq \E[Y^{(1)} \mid U=u] - \E[Y^{(0)} \mid U=u] \neq \tau(v)$.
\end{assumption}

\begin{theorem}[Compressed Counterfactual Pairing] \label{thm:pairing}
Assume Assumption~\ref{ass:richness} and strict latent positivity. Let $\mathbf{V}_k$ be any orthonormal basis of the shared row space of $\mathbf{M}_{\mathrm{stack}}$, and let $\mathbf{R}$ be defined by $\mathbf{B}^\top = \mathbf{R}\mathbf{V}_k^\top$.
Then, for each $t \in \{0,1\}$, $\widetilde{\mathbf{Q}}_t = \mathbf{R}^{-1}\mathbf{D}_{Y|t}\mathbf{R}$.
Consequently,
\begin{equation}
\Delta \widetilde{\mathbf{Q}}
\triangleq
\widetilde{\mathbf{Q}}_1-\widetilde{\mathbf{Q}}_0
=
\mathbf{R}^{-1}(\mathbf{D}_{Y|1}-\mathbf{D}_{Y|0})\mathbf{R}
=
\mathbf{R}^{-1}\mathbf{D}_{\tau}\mathbf{R},
\end{equation}
where $\mathbf{D}_{\tau}\triangleq \diag(\tau(1),\dots,\tau(k))$.
Hence the multiset of eigenvalues of $\Delta \widetilde{\mathbf{Q}}$ is exactly $\{\tau(u)\}_{u=1}^k$, counted with multiplicity. Under Assumption~\ref{ass:separation}, these eigenvalues are simple, and the corresponding left eigenvectors identify the rows of $\mathbf{R}$ up to row-wise scaling and common permutation.
\end{theorem}

\begin{remark}[Zero Treatment Effects on the Signal Subspace]
By confining the spectral dynamics entirely within the verified $k \times k$ signal subspace, the matrix $\Delta \widetilde{\mathbf{Q}}$ contains no extraneous ambient null directions. Consequently, a zero treatment effect appears as a genuine spectral component of the compressed causal operator rather than being conflated with ambient null-space directions. When several latent classes share the same treatment effect, including several zero-effect classes, the method recovers the associated invariant subspace on the signal space. This is a big practical win, because mixtures of treatment effects \textbf{\underline{often have zero treatment effect}}.
\end{remark}

To fully specify the underlying causal model and allow marginalization into the ATE, the unobserved mixture proportions $\pr(U)$ must be identified. This requires lifting the compressed eigenvectors back into the ambient feature space.

\begin{proposition}[Feature and Probability Recovery] \label{prop:latent_prob}
Assume spectral separation holds. Let $\mathbf{W} \in \R^{k \times k}$ be the matrix of left eigenvectors of $\Delta\widetilde{\mathbf{Q}}$, ordered consistently with the treatment-effect eigenvalues. Then there exists a nonsingular diagonal matrix $\mathbf{S}$ such that $\mathbf{W} = \mathbf{S}\mathbf{R}$. Lifting back to the ambient proxy space gives
\[
\mathbf{B}_{\mathrm{unscaled}}^\top \triangleq \mathbf{W}\mathbf{V}_k^\top = \mathbf{S}\mathbf{B}^\top.
\]
By normalizing each row of $\mathbf{B}_{\mathrm{unscaled}}^\top$ by its first coordinate, and using the anchor $X_1 = 1$, one exactly recovers $\mathbf{B}^\top$. The latent mixture proportions are then identified from the marginal mean by
\[
\bm{p} = \mathbf{B}^\dagger \E[X].
\]
\end{proposition}

\begin{remark}[Simplex Projection as Estimation Post-Processing]
At the population level, the identification formula is exactly $\bm{p} = \mathbf{B}^\dagger \E[X]$. In finite samples, however, the plug-in vector $\widehat{\bm{p}}_{\mathrm{raw}} \triangleq \widehat{\mathbf{B}}^\dagger \widehat{\bm{\mu}}_X$ need not lie exactly on the probability simplex because of perturbation error. Projecting $\widehat{\bm{p}}_{\mathrm{raw}}$ onto $\Delta^{k-1}$ is therefore a post-processing step that enforces nonnegativity and unit-sum constraints without changing the population identification formula. The finite-sample analysis in Theorem~\ref{thm:B_p_recovery} shows that this projection is nonexpansive in Euclidean norm, and hence preserves the $n^{-1/2}$ recovery rate for the mixture proportions.
\end{remark}

\begin{theorem}[Operator-Moment Equivalence and Spectral Measure Recovery] \label{thm:moment_equivalence}
Let $\mathbf{D}_{\tau} \triangleq \mathbf{D}_{Y|1}-\mathbf{D}_{Y|0} = \diag(\tau(1),\dots,\tau(k))$, so that $\Delta \widetilde{\mathbf{Q}} = \mathbf{R}^{-1}\mathbf{D}_{\tau}\mathbf{R}$. Define
\[
\mathbf{a}^\top \triangleq \E[X]^\top \mathbf{V}_k \in \R^{1 \times k},
\qquad
\mathbf{c} \triangleq \mathbf{V}_k^\top e_1 \in \R^k,
\]
where $e_1 \in \R^{d_x}$ is the first canonical basis vector corresponding to the anchor coordinate $X_1 = 1$. Then for every integer $\ell \ge 0$, $m_{\ell} \triangleq \E[\tau(U)^\ell] = \mathbf{a}^\top (\Delta \widetilde{\mathbf{Q}})^\ell \mathbf{c}$.
Consequently, the ordinary generating function of the latent treatment-effect distribution is
\[
\sum_{\ell=0}^{\infty} m_{\ell} z^\ell
=
\mathbf{a}^\top (\mathbf{I}_k - z \Delta \widetilde{\mathbf{Q}})^{-1}\mathbf{c},
\qquad
|z| < \|\Delta \widetilde{\mathbf{Q}}\|_2^{-1}.
\]
In particular, the compressed causal operator together with the observable anchor pair \((\mathbf a,\mathbf c)\) determines every polynomial functional of the latent treatment-effect distribution and contains the full scalar moment sequence used in recursive synthetic-moment constructions.
\end{theorem}

Theorem \ref{thm:moment_equivalence} shows that the spectral operator \textbf{\underline{is not}} an alternative estimator for the same target; it is in fact a finite-dimensional generator of the entire synthetic moment hierarchy.

\section{Invariance and Homogeneity}
\label{sec:transportability}

We now record two direct consequences of the similarity representation: invariance of spectral content across environments under stable latent mechanisms, and an exact operator characterization of latent causal homogeneity.

\begin{theorem}[Environment-Invariant Spectral Content] \label{thm:environment}
Let $e \in \mathcal{E}$ index environments. Suppose that in each environment and treatment arm,
\[
\mathbf{M}_{ZX|t}^{(e)} = \mathbf{A}_t^{(e)} \mathbf{D}_{U|t}^{(e)} \mathbf{B}^\top,
\qquad
\mathbf{M}_{ZXY|t}^{(e)} = \mathbf{A}_t^{(e)} \mathbf{D}_{U|t}^{(e)} \mathbf{D}_{Y|t} \mathbf{B}^\top,
\]
where $\mathbf{B}$ and $\mathbf{D}_{Y|t}$ are common across $e$, whereas $\mathbf{A}_t^{(e)}$ and $\mathbf{D}_{U|t}^{(e)}$ may vary with $e$ and remain full column rank. Let $\mathbf{V}_k^{(e)}$ span the shared row space in environment $e$, and write $\mathbf{B}^\top = \mathbf{R}^{(e)} (\mathbf{V}_k^{(e)})^\top$. Then $\widetilde{\mathbf{Q}}_t^{(e)} = (\mathbf{R}^{(e)})^{-1}\mathbf{D}_{Y|t}\mathbf{R}^{(e)}$ and $\Delta \widetilde{\mathbf{Q}}^{(e)} = (\mathbf{R}^{(e)})^{-1}\mathbf{D}_{\tau}\mathbf{R}^{(e)}$. Hence, for every environment $e$,
\[
\spec(\widetilde{\mathbf{Q}}_t^{(e)}) = \left\{\E[Y^{(t)} \mid U=u]\right\}_{u=1}^k,
\qquad
\spec(\Delta \widetilde{\mathbf{Q}}^{(e)}) = \{\tau(u)\}_{u=1}^k,
\]
counted with multiplicity.
\end{theorem}

\begin{proposition}[Operator Characterization of Homogeneity] \label{prop:homogeneity}
Under the standing assumptions, $\tau(1)=\cdots=\tau(k)=\tau_\star \quad\Longleftrightarrow\quad \Delta \widetilde{\mathbf{Q}} = \tau_\star \mathbf{I}_k.$ Equivalently, latent causal homogeneity holds if and only if the compressed difference operator is a scalar matrix. In that case, $\tau_\star = \mathbf{a}^\top \Delta \widetilde{\mathbf{Q}} \,\mathbf{c},$ with $\mathbf{a}$ and $\mathbf{c}$ as defined in Theorem \ref{thm:moment_equivalence}.
\end{proposition}

\section{Finite-Sample Estimation and Guarantees}
\label{sec:algorithm}

At the population level, the causal information is represented by the
$k$-dimensional compressed operators $\widetilde{\mathbf Q}_t$. In finite
samples, both the empirical moment matrices and the estimated signal subspace contribute perturbation error. We therefore estimate the shared row space by a
truncated singular value decomposition before constructing the compressed operators.

\begin{remark}[The Complex-Bifurcation Diagnostic]
At the population level, $\Delta\widetilde{\mathbf Q}=\mathbf R^{-1}\mathbf D_\tau\mathbf R$ is diagonalizable over $\mathbb R$. In finite samples, $\Delta\widehat{\widetilde{\mathbf Q}}$ is a real but generally non-symmetric perturbation of this operator. Writing the perturbation in diagonal coordinates as $\mathbf S\Delta\widehat{\widetilde{\mathbf Q}}\mathbf S^{-1}=\mathbf D_\tau+\mathbf F$, the localization argument in Lemma~\ref{lem:diag_localization} implies that, whenever $2\|\mathbf F\|_2<\delta_u$, the empirical eigenvalue associated with the simple population effect $\tau(u)$ is isolated in the $\delta_u/2$-neighborhood of $\tau(u)$. Since the matrix is real, an isolated single eigenvalue in a real-centered disc cannot form a non-real conjugate pair; hence it is real. Thus, under spectral separation, sufficiently small finite-sample perturbations preserve real empirical eigenvalues. Conversely, when the perturbation is large enough that these isolation discs overlap, real eigenvalues may collide and bifurcate into complex-conjugate pairs. Macroscopic imaginary components therefore indicate that the estimator is outside the local separated-eigenvalue regime, due either to limited sample size, ill-conditioning, rank misspecification, or model error. Residual imaginary parts at numerical precision, e.g.\ $10^{-15}$, can be safely discarded.
\end{remark}

\subsection{High-Probability Perturbation Bounds}

As discussed, our estimator operates directly on observable second-order cross-moments and outcome-weighted third-order moments after compression to the shared signal subspace. We quantify the resulting finite-sample behavior through a local perturbation analysis of the compressed operator and the associated lifted eigenvectors.

\begin{assumption}[Finite-Sample Regularity]
\label{ass:finite_sample}
Throughout this section $k,d_x,d_z$ are fixed. Assume that, for finite constants
$L_X,L_{ZX},L_{ZXY}>0$,
\[
\|X\|_2\le L_X,\qquad
\|ZX^\top\|_2\le L_{ZX},\qquad
\|YZX^\top\|_2\le L_{ZXY}
\qquad\text{almost surely}.
\]
Let $\pi\triangleq\min_{t\in\{0,1\}}\Pr(T=t)>0$.
Let $\mathbf V_k$ denote the population $k$-dimensional right singular subspace of
$\mathbf M_{\mathrm{stack}}$, and assume the signal singular-value margin
\[
\sigma\triangleq
\min\Bigl\{
\sigma_k(\mathbf M_{\mathrm{stack}}),\;
\min_{t\in\{0,1\}}\sigma_k(\mathbf M_{ZX|t}\mathbf V_k)
\Bigr\}>0.
\]
\end{assumption}

Define the fixed population quantities
\[
M_X\triangleq \max_t\|\mathbf M_{ZX|t}\|_2,\qquad
M_Y\triangleq \max_t\|\mathbf M_{ZXY|t}\|_2,\qquad
\kappa_B\triangleq \|\mathbf B^\top\|_2\|(\mathbf B^\top)^\dagger\|_2,
\]
where all maxima are over $t\in\{0,1\}$. If $C_{\mathrm W}$ denotes the
absolute constant in Wedin's theorem \citet{wedin1972perturbation}, set
\[
\Gamma_X\triangleq 1+C_{\mathrm W}M_X/\sigma,\qquad
\Gamma_Y\triangleq 1+C_{\mathrm W}M_Y/\sigma,\qquad
\kappa_{ZX}^{\star}\triangleq
\Gamma_Y+6\Gamma_X\bigl(M_Y/\sigma+\Gamma_Y\bigr).
\]
Thus $\kappa_{ZX}^{\star}$ is a fixed conditioning factor depending only on
the population moment matrices and the signal singular-value margin.

\begin{theorem}[High-Probability Perturbation Bound]
\label{thm:sample_complexity}
Assume Assumption~\ref{ass:richness}, strict latent positivity, and
Assumption~\ref{ass:finite_sample}. For $\eta\in(0,1)$, let $\Lambda_\eta\triangleq\log\frac{8(d_x+d_z)}{\eta}$, $L_{\max}\triangleq\max\{L_{ZX},L_{ZXY}\}$, and $\varepsilon_{n,\eta}\triangleq C_{\mathrm{mom}}L_{\max}\sqrt{\frac{\Lambda_\eta}{n\pi}}$, where $C_{\mathrm{mom}}$ is the universal constant in Lemma~\ref{lem:conditional_moment_concentration}. There exist
universal constants $c,C>0$ such that, if
$n\pi\ge c\Lambda_\eta$ and $\varepsilon_{n,\eta}\le c\sigma/\Gamma_X$, then,
with probability at least $1-\eta$,
\[
\forall\,\widehat\tau\in\spec(\Delta\widehat{\widetilde{\mathbf Q}}),
\qquad
\min_{u\in\{1,\ldots,k\}}|\widehat\tau-\tau(u)|
\le
C\frac{\kappa_B\kappa_{ZX}^{\star}L_{\max}}{\sigma}
\sqrt{\frac{\log(8(d_x+d_z)/\eta)}{n\pi}}.
\]
If $\delta_{\min}\triangleq\min_{u\neq v}|\tau(u)-\tau(v)|>0$ and the
right-hand side above is smaller than $\delta_{\min}/2$, then the empirical
eigenvalues admit a one-to-one matching with the population effects: there
exists a permutation $\rho\in S_k$ such that
\begin{equation}
\label{eq:bound_eigenvalue}
\max_{u\in\{1,\ldots,k\}}|\widehat\tau_{\rho(u)}-\tau(u)|
\le
C\frac{\kappa_B\kappa_{ZX}^{\star}L_{\max}}{\sigma}
\sqrt{\frac{\log(8(d_x+d_z)/\eta)}{n\pi}} .
\end{equation}
Furthermore, fix $u$ and let
$\delta_u\triangleq\min_{v\neq u}|\tau(u)-\tau(v)|>0$. If the preceding eigenvalue perturbation radius is smaller than $\delta_u/4$, then, after
choosing the corresponding lifted left eigenvector
$\widehat{\mathbf b}_u^\top$, aligning its sign/phase, and normalizing both
lifted rows to unit Euclidean norm,
\begin{equation}
\label{eq:bound_eigenvector}
\sin\theta(\widehat{\mathbf b}_u^\top,\mathbf b_u^\top)
\le
C L_{\max}
\left(
\frac{\kappa_B^2\kappa_{ZX}^{\star}}{\delta_u\sigma}
+
\frac{1}{\sigma}
\right)
\sqrt{\frac{\log(8(d_x+d_z)/\eta)}{n\pi}} .
\end{equation}
In particular, under fixed conditioning and fixed eigengap assumptions, both
eigenvalue recovery and lifted-row recovery occur at first-order $n^{-1/2}$
scale.
\end{theorem}

The two terms in \eqref{eq:bound_eigenvector} have distinct origins: the first
is diagonal-coordinate left-eigenvector perturbation transported through the
similarity map, while the second is the subspace-lifting error induced by
estimating $\mathbf V_k$.

\begin{theorem}[Finite-Sample Recovery of the Feature Matrix and Mixture Weights]
\label{thm:B_p_recovery}
Assume Assumption~\ref{ass:richness}, strict latent positivity,
Assumption~\ref{ass:finite_sample}, and spectral separation. Let
\[
\delta_{\min}\triangleq
\begin{cases}
\min_{u\neq v}|\tau(u)-\tau(v)|, & k\ge 2,\\
+\infty, & k=1.
\end{cases}
\]
Thus \(\delta_{\min}^{-1}=0\) when \(k=1\), and the spectral-separation
condition is vacuous in the one-class case. Let
\(\mathbf b_u^\top\) denote the \(u\)-th row of the true matrix
\(\mathbf B^\top\). Since \(X_1=1\), \(b_{u1}=1\) for every \(u\). Define the
anchor margin
\[
\alpha_{\mathrm{anc}}
\triangleq
\min_{1\le u\le k}\frac{|b_{u1}|}{\|\mathbf b_u\|_2}
=
\min_{1\le u\le k}\frac{1}{\|\mathbf b_u\|_2}
>0.
\]
Also define
$\sigma_B\triangleq \sigma_k(\mathbf B)$,
$\bm\mu_X\triangleq \E[X]$,
$M_\mu\triangleq \|\bm\mu_X\|_2$,
$L_Q\triangleq \max\{L_{ZX},L_{ZXY}\}$.
For $\eta\in(0,1)$, set $\Lambda_\eta^\star\triangleq \log\frac{16(d_x+d_z)}{\eta}$,
and define the eigenvalue and lifted-row radii
\[
r_{\tau,n,\eta}\triangleq C_\tau\frac{\kappa_B\kappa_{ZX}^{\star}L_Q}{\sigma}\sqrt{\tfrac{\Lambda_\eta^\star}{n\pi}},
\quad
r_{B,n,\eta}\triangleq C_B L_Q\!\left(\tfrac{\kappa_B^2\kappa_{ZX}^{\star}}{\delta_{\min}\sigma}+\tfrac{1}{\sigma}\right)\!\sqrt{\tfrac{\Lambda_\eta^\star}{n\pi}},
\]
and
$\varepsilon_{B,n,\eta}\triangleq C_{\mathrm{anc}}\alpha_{\mathrm{anc}}^{-2}r_{B,n,\eta}$,
$\varepsilon_{\mu,n,\eta}\triangleq C_\mu L_X\sqrt{\Lambda_\eta^\star/n}$,
where $C_\tau,C_B,C_{\mathrm{anc}},C_\mu>0$ are universal constants.
Assume the sample-size conditions of Theorem~\ref{thm:sample_complexity} hold with
$\eta/2$ in place of $\eta$, and assume additionally that
\[
r_{\tau,n,\eta}<\tfrac{\delta_{\min}}{2},\qquad
r_{B,n,\eta}\le c_0\alpha_{\mathrm{anc}},\qquad
\sqrt{k}\,\varepsilon_{B,n,\eta}\le \tfrac{\sigma_B}{2},
\]
for a sufficiently small universal constant $c_0>0$. Then, with probability at least $1-\eta$, there exists a permutation $\rho\in S_k$ such that the following hold simultaneously.

First, the treatment-effect estimates obey
\[
\max_{1\le u\le k}|\widehat\tau_{\rho(u)}-\tau(u)|\le r_{\tau,n,\eta}.
\]
Second, let $\widehat{\mathbf w}_{\rho(u)}^\top$ be the empirical left eigenvector associated with $\widehat\tau_{\rho(u)}$, and define the corresponding unit-norm lifted row
\[
\widehat{\mathbf q}_{\rho(u)}^\top\triangleq \frac{\widehat{\mathbf w}_{\rho(u)}^\top\widehat{\mathbf V}_k^\top}{\bigl\|\widehat{\mathbf w}_{\rho(u)}^\top\widehat{\mathbf V}_k^\top\bigr\|_2}.
\]
After choosing the real sign so that its first coordinate is aligned with the population anchor, define the anchor-normalized row estimate $\widehat{\mathbf b}_{\rho(u)}^\top\triangleq \widehat{\mathbf q}_{\rho(u)}^\top/\widehat q_{\rho(u),1}$. Then
\[
\max_{1\le u\le k}\bigl\|\widehat{\mathbf b}_{\rho(u)}^\top-\mathbf b_u^\top\bigr\|_2\le \varepsilon_{B,n,\eta}.
\]
Equivalently, if $\widehat{\mathbf B}_\rho^\top$ is the row-permuted estimator whose $u$-th row is $\widehat{\mathbf b}_{\rho(u)}^\top$, then
\[
\|\widehat{\mathbf B}_\rho-\mathbf B\|_2\le \|\widehat{\mathbf B}_\rho^\top-\mathbf B^\top\|_F\le \sqrt{k}\,\varepsilon_{B,n,\eta}.
\]
Third, let
$\widehat{\bm\mu}_X\triangleq \tfrac1n\sum_{i=1}^n X_i$,
$\widehat{\bm p}_{\rho,\mathrm{raw}}\triangleq \widehat{\mathbf B}_\rho^\dagger \widehat{\bm\mu}_X$,
$\widehat{\bm p}_{\rho}\triangleq \Pi_{\Delta^{k-1}}\bigl(\widehat{\bm p}_{\rho,\mathrm{raw}}\bigr)$,
where $\Pi_{\Delta^{k-1}}$ denotes Euclidean projection onto the probability simplex. Then
\[
\|\widehat{\bm p}_{\rho}-\bm p\|_2\le \frac{2}{\sigma_B}\varepsilon_{\mu,n,\eta}+\frac{6M_\mu\sqrt{k}}{\sigma_B^2}\varepsilon_{B,n,\eta}.
\]
Similar to the previous theorem, under fixed conditioning, fixed eigengap, and fixed anchor-margin assumptions, the treatment effects, the anchor-normalized feature matrix, and the simplex-projected mixture proportions are all recovered at first-order $n^{-1/2}$ scale.
\end{theorem}

\section{Empirical Evaluation}
\label{sec:experiments}

We evaluate the compressed spectral estimator on synthetic data generated from
the finite-mixture proxy model of Section~\ref{sec:geometry}. Across all
experiments the latent dimension $k$ is treated as known. Throughout, we work
in the overcomplete proxy regime $d_z = d_x = k + 3$. This is both the
empirically realistic setting --- observed feature spaces typically admit more
proxy coordinates than the underlying latent cardinality --- and the regime in
which the SVD compression of Definition \ref{def:compressed_operator} does meaningful
work, projecting onto the signal subspace and discarding ambient noise
directions. When $d_z = d_x = k$ the SVD step becomes trivial: $\widehat V_k$ is a rotation of the identity, and Theorem \ref{thm:moment_equivalence}
then implies that our spectral estimator and the base SPO baseline below produce the same multiset
of eigenvalues through algebraically equivalent procedures, \textit{generically when the induced spectral
weights are nonzero for all latent treatment effects}. We verified this empirically: The two
estimators agree to numerical precision across the full grid of $(k, N)$
settings reported in this section. We therefore report only the
overcomplete results, in which the comparison between adaptive subspace
projection and fixed truncation is informative.

For each treatment arm $t\in\{0,1\}$, we form
\[
\widehat{\mathbf{M}}_{ZX|t}
=
\frac{1}{n_t}\sum_{i:T_i=t} Z_iX_i^\top,
\qquad
\widehat{\mathbf{M}}_{ZXY|t}
=
\frac{1}{n_t}\sum_{i:T_i=t} Y_iZ_iX_i^\top,
\]
extract $\widehat{\mathbf{V}}_k$ from the truncated SVD of the stacked proxy
matrix
$\widehat{\mathbf M}_{\mathrm{stack}}
= [\widehat{\mathbf M}_{ZX|0}^\top \mid \widehat{\mathbf M}_{ZX|1}^\top]^\top$,
and construct
\[
\widehat{\widetilde{\mathbf{Q}}}_t
=
(\widehat{\mathbf{M}}_{ZX|t}\widehat{\mathbf{V}}_k)^\dagger
(\widehat{\mathbf{M}}_{ZXY|t}\widehat{\mathbf{V}}_k).
\]
The latent treatment effects are estimated by the real parts of the
eigenvalues of
\[
\Delta\widehat{\widetilde{\mathbf{Q}}}
=
\widehat{\widetilde{\mathbf{Q}}}_1-
\widehat{\widetilde{\mathbf{Q}}}_0,
\]
sorted increasingly before comparison with the population effects.

For comparison, we use the base SPO moment-chain of
\citep{mazaheri2024synthetic}. Because that construction is built around the
inverse of a square $k\times k$ proxy moment matrix, it does not natively
accept overcomplete inputs; we therefore restrict its input to the first $k$
proxy coordinates. We form the arm-specific quotients
$\widehat{\mathbf Q}_t = \widehat{\mathbf M}_{ZX|t}^{-1}\widehat{\mathbf M}_{ZXY|t}$,
and build the scalar moment sequence
\[
m_\ell
=
\widehat{\mathbf{a}}^\top
(\widehat{\mathbf{Q}}_1-\widehat{\mathbf{Q}}_0)^\ell
\widehat{\mathbf{c}},
\qquad
\ell=0,\ldots,2k-1,
\]
recovering the support through the generalized eigenvalues of the Hankel
pencil $(H_1,H_0)$. The empirical comparison in  Section \ref{sec:experiments_eigenvalues} is
therefore exactly the question of how to extract a $k$-dimensional working
space from the available $d_x = k+3$ proxy coordinates: adaptively, via the
top-$k$ right singular subspace of
$\widehat{\mathbf M}_{\mathrm{stack}}$, or by fixed truncation to the first
$k$ coordinates. Appendix \ref{app:no_truncation} additionally evaluates base SPO applied without
truncation to the full $(k+3)\times(k+3)$ moment matrices; that variant
performs \textbf{\underline{substantially worse}} than the truncated baseline reported in this
section, confirming that the truncation is itself a numerically favorable
choice within the original method.

\subsection{Eigenvalue Recovery Across Dimension and Sample Size}
\label{sec:experiments_eigenvalues}

We first measure recovery of the latent treatment-effect support over a grid of latent dimensions and sample sizes. Data are generated with
\[
U\sim \mathrm{Uniform}[k],
\qquad
\Pr(T=1\mid U=u)\in[0.3,0.7],
\qquad
\tau(u)=\mathrm{linspace}(-2,2,k),
\]
where treatment propensities are linearly spaced across latent classes. Proxy noise variance is fixed at $0.25$. Results are aggregated over 15 independent trials per cell, and we report median absolute eigenvalue error.

\begin{figure}[t]
    \centering
    \includegraphics[width=\textwidth]{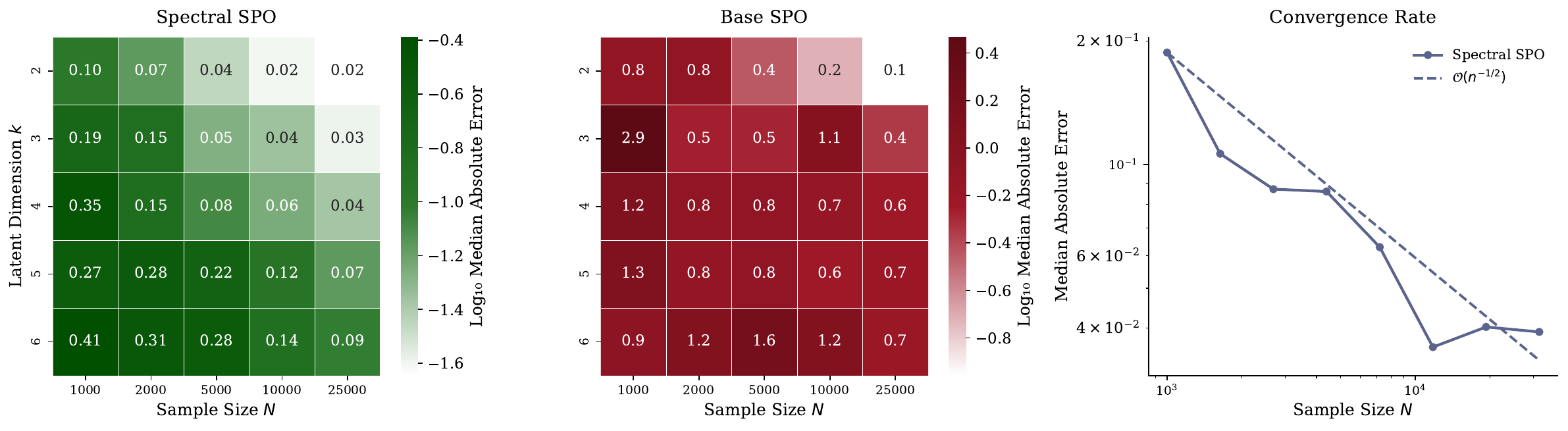}
    \caption{
    \textbf{Left:} Median absolute eigenvalue error of the compressed spectral estimator across latent dimensions $k\in\{2,\ldots,6\}$ and sample sizes $N\in\{10^3,\ldots,2.5\times 10^4\}$.
    \textbf{Center:} Corresponding errors for the base SPO moment-chain baseline. Errors are larger and less stable, especially as $k$ increases.
    \textbf{Right:} Empirical convergence rate of the spectral estimator at $k=3$ against the $\mathcal{O}(n^{-1/2})$ reference.
    }
    \label{fig:main_experiments}
\end{figure}

\begin{table}[t]
  \centering
  \caption{
  Median absolute eigenvalue error across latent dimensions. 
  Lower is better. Bold entries indicate the better method at each sample size.
  }
  \label{tab:mae_comparison}
  \setlength{\tabcolsep}{7pt}
  \renewcommand{\arraystretch}{1.12}
  \begin{tabular}{c cc cc cc c}
    \toprule
    & \multicolumn{2}{c}{$N=1{,}000$}
    & \multicolumn{2}{c}{$N=5{,}000$}
    & \multicolumn{2}{c}{$N=25{,}000$}
    & \cellcolor{gray!10}{} \\
    \cmidrule(lr){2-3}
    \cmidrule(lr){4-5}
    \cmidrule(lr){6-7}
    $k$
    & Spectral & Base
    & Spectral & Base
    & Spectral & Base
    & \cellcolor{gray!10}{Gain} \\
    \midrule
    \rowcolor{gray!3}
    $2$ & \textbf{0.102} & 0.810 & \textbf{0.035} & 0.357 & \textbf{0.023} & 0.110 & \cellcolor{green!8}{$4.9\times$} \\
    $3$ & \textbf{0.191} & 2.931 & \textbf{0.054} & 0.524 & \textbf{0.026} & 0.445 & \cellcolor{green!15}{$17.4\times$} \\
    \rowcolor{gray!3}
    $4$ & \textbf{0.348} & 1.248 & \textbf{0.083} & 0.815 & \textbf{0.042} & 0.584 & \cellcolor{green!13}{$13.8\times$} \\
    $5$ & \textbf{0.273} & 1.253 & \textbf{0.219} & 0.820 & \textbf{0.069} & 0.653 & \cellcolor{green!10}{$9.5\times$} \\
    \rowcolor{gray!3}
    $6$ & \textbf{0.408} & 0.907 & \textbf{0.283} & 1.631 & \textbf{0.094} & 0.654 & \cellcolor{green!9}{$7.0\times$} \\
    \bottomrule
  \end{tabular}
\end{table}

Figure~\ref{fig:main_experiments} and Table~\ref{tab:mae_comparison} show that the compressed spectral estimator is consistently more accurate across the full grid. At $N=25{,}000$, the spectral errors range from $0.023$ to $0.094$, while the base SPO errors range from $0.110$ to $0.654$. The improvement is largest at $k=3$, where the final-sample error falls from $0.445$ to $0.026$, a $17.4\times$ reduction. The base SPO estimator also exhibits non-monotone behavior in several rows, consistent with the instability of recovering polynomial roots from empirically perturbed moment sequences. The spectral estimator instead works directly with the compressed difference operator, and its convergence curve follows the predicted $n^{-1/2}$ scale.

\subsection{Finite-Sample Eigenspace Distributions}
\label{sec:experiments_eigenspaces}

We next visualize the finite-sample distribution of recovered treatment effects in a fixed three-class problem. We run 300 Monte Carlo trials with
\[
k=3,
\qquad
N=5000,
\qquad
\tau(u)\in\{-2,0,2\}.
\]
The naive ATE estimator collapses the heterogeneous treatment-effect distribution into one scalar. The base SPO estimator targets the correct latent support in principle but is dispersed because perturbations in the moment sequence are amplified by the Hankel pencil. The compressed spectral estimator diagonalizes $\Delta\widehat{\widetilde{\mathbf{Q}}}$ directly on the recovered signal subspace, producing localized clusters around the population effects.

\begin{figure}[!htbp]
    \centering
    \includegraphics[width=0.88\textwidth]{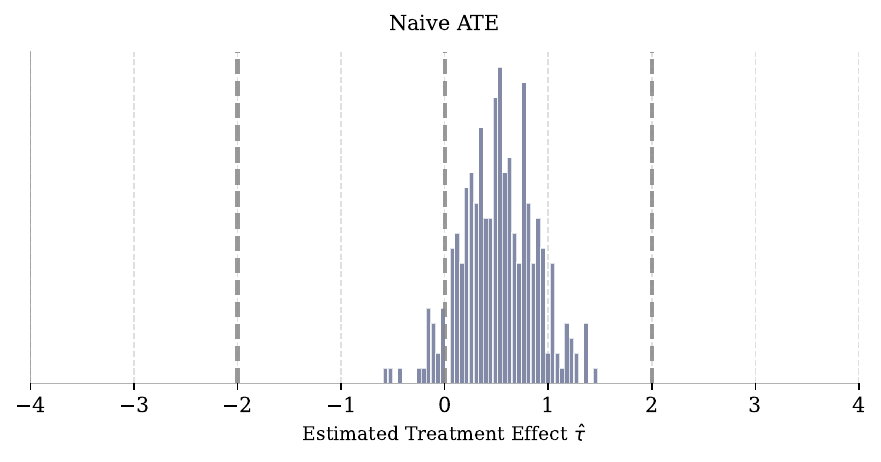}
    \caption{
    Naive ATE estimates across 300 trials for the three-class mixture with true effects $\{-2,0,2\}$. The estimator collapses the latent treatment-effect law into a single scalar and does not recover the mixture support.
    }
    \label{fig:hist_naive}
\end{figure}

\begin{figure}[!htbp]
    \centering
    \includegraphics[width=0.9\textwidth]{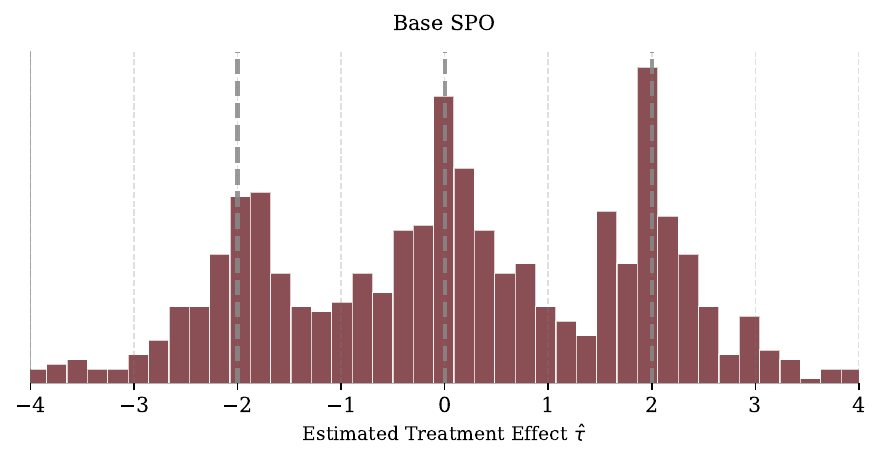}
    \caption{
    Base SPO eigenvalue estimates across 300 trials. Although the scalar moment construction is population-identifying, the finite-sample estimates are diffuse over $[-4,4]$, reflecting the instability of Hankel pencil recovery under empirical moment noise.
    }
    \label{fig:hist_base}
\end{figure}

\begin{figure}[!htbp]
    \centering
    \includegraphics[width=0.9\textwidth]{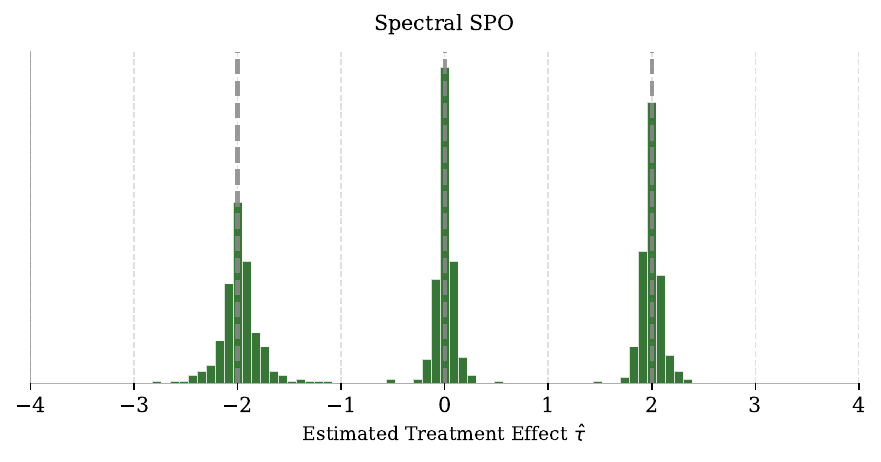}
    \caption{
    Compressed spectral estimator eigenvalue estimates across 300 trials. The empirical distribution forms sharp clusters around the true effects $\tau\in\{-2,0,2\}$, consistent with the eigenvalue perturbation guarantee of Theorem~\ref{thm:sample_complexity}.
    }
    \label{fig:hist_spectral}
\end{figure}

Figures~\ref{fig:hist_naive}--\ref{fig:hist_spectral} give a distributional view of the same pattern. The naive estimator misses latent heterogeneity, the base SPO moment-chain estimator is sensitive to finite-sample perturbations, and the compressed spectral estimator remains concentrated around the true latent treatment effects. The histograms also show that the improvement is not driven only by a few favorable trials: the spectral estimator has visibly smaller dispersion and fewer extreme failures across repeated simulations. 

\section{Discussion and Future Directions}
\label{sec:discussion}

Our results suggest that latent causal mixture identification is fundamentally an operator problem. The scalar moment construction of prior work is not separate from this perspective; it is obtained by evaluating bilinear functionals of powers of the same compressed quotient operator. The spectral method therefore exposes the finite-dimensional object that the moment sequence was implicitly probing, and identifies the latent treatment effects by diagonalizing this object directly on the shared proxy signal subspace. This reframing also clarifies the role of overcomplete proxy features. Rectangular proxy embeddings do not change the latent estimand, but they can provide a better-conditioned observable representation before the final rank-$k$ compression, stabilizing the pseudo-inverse and eigendecomposition steps. Several directions remain open. One is to extend the quotient-operator formulation beyond finite latent classes, for example through compact integral operators or RKHS analogues. It is also worth noting that the exact homogeneity characterization $\Delta \widetilde{\mathbf{Q}}=\tau_\star \mathbf{I}_k$ suggests a direct path to tests for treatment-effect heterogeneity, while the observed gains from overcomplete proxies motivate sharper population lower bounds for signal singular values.

\section{Acknowledgements}

Support for Bijan Mazaheri was provided by the Advanced Research Concepts (ARC) COMPASS program, sponsored by the Defense Advanced Research Projects Agency (DARPA) under agreement number
HR001-25-3-0212.

\newpage

\bibliographystyle{plainnat}
\bibliography{references}

@misc{guo2026comparingproxymethodscausal,
      title={Comparing Two Proxy Methods for Causal Identification}, 
      author={Helen Guo and Elizabeth L. Ogburn and Ilya Shpitser},
      year={2026},
      eprint={2512.00175},
      archivePrefix={arXiv},
      primaryClass={stat.ME},
      url={https://arxiv.org/abs/2512.00175}, 
}

@article{miao2018identifying,
  title     = {Identifying Causal Effects with Proxy Variables of an Unmeasured Confounder},
  author    = {Miao, Wang and Geng, Zhi and Tchetgen Tchetgen, Eric J.},
  journal   = {Biometrika},
  volume    = {105},
  number    = {4},
  pages     = {987--993},
  year      = {2018},
  publisher = {Oxford University Press},
  url       = {https://doi.org/10.1093/biomet/asy038}
}

@article{loh2022evaluating,
  title     = {Evaluating Sensitivity to Classification Uncertainty in Latent Subgroup Effect Analyses},
  author    = {Loh, Wen Wei and Kim, Jee-Seon},
  journal   = {BMC Medical Research Methodology},
  volume    = {22},
  number    = {1},
  pages     = {247},
  year      = {2022},
  publisher = {Springer},
  url       = {https://doi.org/10.1186/s12874-022-01720-8}
}

@misc{xu2026aleatoric,
  title         = {Estimating Aleatoric Uncertainty in the Causal Treatment Effect},
  author        = {Xu, Liyuan and Mazaheri, Bijan},
  year          = {2026},
  eprint        = {2602.08461},
  archivePrefix = {arXiv},
  primaryClass  = {cs.LG},
  doi           = {10.48550/arXiv.2602.08461},
  url           = {https://arxiv.org/abs/2602.08461}
}

@misc{mazaheri2023latentclassconfounding,
  title         = {Causal Discovery under Latent Class Confounding},
  author        = {Mazaheri, Bijan and Gordon, Spencer and Rabani, Yuval and Schulman, Leonard J.},
  year          = {2023},
  eprint        = {2311.07454},
  doi           = {10.48550/arXiv.2311.07454},
  url           = {https://arxiv.org/abs/2311.07454},
  note          = {Last revised October 16, 2024}
}

@inproceedings{bozkurt2025densityratio,
  title = {Density Ratio-based Proxy Causal Learning Without Density Ratios},
  author = {Bozkurt, Bariscan and Deaner, Ben and Meunier, Dimitri and Xu, Liyuan and Gretton, Arthur},
  booktitle = {Proceedings of The 28th International Conference on Artificial Intelligence and Statistics},
  series = {Proceedings of Machine Learning Research},
  volume = {258},
  pages = {5095--5103},
  year = {2025},
  publisher = {PMLR},
  url = {https://proceedings.mlr.press/v258/bozkurt25a.html}
}

@article{lyu2023estimating,
  title     = {Estimating Heterogeneous Treatment Effects Within Latent Class Multilevel Models: A Bayesian Approach},
  author    = {Lyu, Weicong and Kim, Jee-Seon and Suk, Youmi},
  journal   = {Journal of Educational and Behavioral Statistics},
  volume    = {48},
  number    = {1},
  pages     = {3--36},
  year      = {2023},
  publisher = {SAGE Publications},
  doi       = {10.3102/10769986221115446},
  url       = {https://doi.org/10.3102/10769986221115446}
}

@article{shi2020selective,
  title     = {A Selective Review of Negative Control Methods in Epidemiology},
  author    = {Shi, Xu and Miao, Wang and Tchetgen Tchetgen, Eric},
  journal   = {Current Epidemiology Reports},
  volume    = {7},
  number    = {4},
  pages     = {190--202},
  year      = {2020},
  publisher = {Springer},
  url       = {https://doi.org/10.1007/s40471-020-00243-4}
}

@article{allman2009identifiability,
  title     = {Identifiability of Parameters in Latent Structure Models with Many Observed Variables},
  author    = {Allman, Elizabeth S. and Matias, Catherine and Rhodes, John A.},
  journal   = {The Annals of Statistics},
  volume    = {37},
  number    = {6A},
  pages     = {3099--3132},
  year      = {2009},
  publisher = {Institute of Mathematical Statistics},
  url       = {https://doi.org/10.1214/09-AOS689}
}

@article{simpson1951interpretation,
  title={The interpretation of interaction in contingency tables},
  author={Simpson, Edward H},
  journal={Journal of the Royal Statistical Society: Series B (Methodological)},
  volume={13},
  number={2},
  pages={238--241},
  year={1951},
  publisher={Wiley Online Library}
}

@article{anandkumar2014tensor,
  title   = {Tensor Decompositions for Learning Latent Variable Models},
  author  = {Anandkumar, Animashree and Ge, Rong and Hsu, Daniel and Kakade, Sham M. and Telgarsky, Matus},
  journal = {Journal of Machine Learning Research},
  volume  = {15},
  number  = {80},
  pages   = {2773--2832},
  year    = {2014},
  url     = {https://jmlr.org/papers/v15/anandkumar14b.html}
}

@article{wager2018estimation,
  title     = {Estimation and Inference of Heterogeneous Treatment Effects Using Random Forests},
  author    = {Wager, Stefan and Athey, Susan},
  journal   = {Journal of the American Statistical Association},
  volume    = {113},
  number    = {523},
  pages     = {1228--1242},
  year      = {2018},
  publisher = {Taylor \& Francis},
  url       = {https://doi.org/10.1080/01621459.2017.1319839}
}

@incollection{kim2015mixture,
  title     = {Mixture Modeling Methods for Causal Inference with Multilevel Data},
  author    = {Kim, Jee-Seon and Steiner, Peter M. and Lim, Wen-Chiang},
  booktitle = {Advances in Multilevel Modeling for Educational Research: Addressing Practical Issues Found in Real-World Applications},
  pages     = {335--359},
  year      = {2015},
  publisher = {Information Age Publishing},
  url       = {https://www.amazon.com/Advances-Multilevel-Modeling-Educational-Research/dp/1681233282}
}

@article{suk2021hybridizing,
  title={Hybridizing machine learning methods and finite mixture models for estimating heterogeneous treatment effects in latent classes},
  author={Suk, Youmi and Kim, Jee-Seon and Kang, Hyunseung},
  journal={Journal of Educational and Behavioral Statistics},
  volume={46},
  number={3},
  pages={323--347},
  year={2021},
  publisher={Sage Publications Sage CA: Los Angeles, CA}
}

@incollection{pearl2022detecting,
  title     = {Detecting Latent Heterogeneity},
  author    = {Pearl, Judea},
  booktitle = {Probabilistic and Causal Inference: The Works of Judea Pearl},
  pages     = {483--506},
  year      = {2022},
  publisher = {Association for Computing Machinery},
  url       = {https://doi.org/10.1145/3501714.3501742}
}

@article{nie2021quasi,
  title     = {Quasi-Oracle Estimation of Heterogeneous Treatment Effects},
  author    = {Nie, Xinkun and Wager, Stefan},
  journal   = {Biometrika},
  volume    = {108},
  number    = {2},
  pages     = {299--319},
  year      = {2021},
  publisher = {Oxford University Press},
  doi       = {10.1093/biomet/asaa076},
  url       = {https://doi.org/10.1093/biomet/asaa076}
}

@article{wang2019blessings,
  title     = {The Blessings of Multiple Causes},
  author    = {Wang, Yixin and Blei, David M.},
  journal   = {Journal of the American Statistical Association},
  volume    = {114},
  number    = {528},
  pages     = {1574--1596},
  year      = {2019},
  publisher = {Taylor \& Francis},
  url       = {https://doi.org/10.1080/01621459.2019.1686987}
}

@article{ogburn2019comment,
  title     = {Comment on ``Blessings of Multiple Causes''},
  author    = {Ogburn, Elizabeth L. and Shpitser, Ilya and Tchetgen Tchetgen, Eric J.},
  journal   = {Journal of the American Statistical Association},
  volume    = {114},
  number    = {528},
  pages     = {1611--1615},
  year      = {2019},
  publisher = {Taylor \& Francis},
  url       = {https://doi.org/10.1080/01621459.2019.1689139}
}

@inproceedings{hsu2013learning,
  title={Learning mixtures of spherical gaussians: moment methods and spectral decompositions},
  author={Hsu, Daniel and Kakade, Sham M},
  booktitle={Proceedings of the 4th conference on Innovations in Theoretical Computer Science},
  pages={11--20},
  year={2013}
}

@inproceedings{gordon2023identification,
  title     = {Identification of Mixtures of Discrete Product Distributions in Near-Optimal Sample and Time Complexity},
  author    = {Gordon, Spencer L. and Jahn, Erik and Mazaheri, Bijan and Rabani, Yuval and Schulman, Leonard J.},
  booktitle = {Proceedings of Thirty Seventh Conference on Learning Theory},
  pages     = {2071--2091},
  year      = {2024},
  volume    = {247},
  series    = {Proceedings of Machine Learning Research},
  publisher = {PMLR},
  url       = {https://proceedings.mlr.press/v247/gordon24a.html}
}

@inproceedings{rabani2014learning,
  title     = {Learning Mixtures of Arbitrary Distributions over Large Discrete Domains},
  author    = {Rabani, Yuval and Schulman, Leonard J. and Swamy, Chaitanya},
  booktitle = {Proceedings of the 5th Conference on Innovations in Theoretical Computer Science},
  pages     = {207--224},
  year      = {2014},
  publisher = {Association for Computing Machinery},
  url       = {https://doi.org/10.1145/2554797.2554818}
}

@article{ogburn2020counterexamples,
  title   = {Counterexamples to ``The Blessings of Multiple Causes'' by Wang and Blei},
  author  = {Ogburn, Elizabeth L. and Shpitser, Ilya and Tchetgen Tchetgen, Eric J.},
  journal = {arXiv preprint arXiv:2001.06555},
  year    = {2020},
  url     = {https://arxiv.org/abs/2001.06555}
}

@book{pearl2009causality,
  title={Causality},
  author={Pearl, Judea},
  year={2009},
  publisher={Cambridge university press}
}

@inproceedings{gordon2021source,
  title     = {Source Identification for Mixtures of Product Distributions},
  author    = {Gordon, Spencer and Mazaheri, Bijan H. and Rabani, Yuval and Schulman, Leonard},
  booktitle = {Proceedings of Thirty Fourth Conference on Learning Theory},
  pages     = {2193--2216},
  year      = {2021},
  volume    = {134},
  series    = {Proceedings of Machine Learning Research},
  publisher = {PMLR},
  url       = {https://proceedings.mlr.press/v134/gordon21a.html}
}

@inproceedings{gordon2023causal,
  title     = {Causal Inference Despite Limited Global Confounding via Mixture Models},
  author    = {Gordon, Spencer L. and Mazaheri, Bijan and Rabani, Yuval and Schulman, Leonard},
  booktitle = {Proceedings of the Second Conference on Causal Learning and Reasoning},
  pages     = {574--601},
  year      = {2023},
  volume    = {213},
  series    = {Proceedings of Machine Learning Research},
  publisher = {PMLR},
  url       = {https://proceedings.mlr.press/v213/gordon23a.html}
}

@inproceedings{squires2022causal,
  title     = {Causal Imputation via Synthetic Interventions},
  author    = {Squires, Chandler and Shen, Dennis and Agarwal, Anish and Shah, Devavrat and Uhler, Caroline},
  booktitle = {Proceedings of the First Conference on Causal Learning and Reasoning},
  pages     = {688--711},
  year      = {2022},
  volume    = {177},
  series    = {Proceedings of Machine Learning Research},
  publisher = {PMLR},
  url       = {https://proceedings.mlr.press/v177/squires22b.html}
}

@inproceedings{agarwal2023causal,
  title     = {Causal Matrix Completion},
  author    = {Agarwal, Anish and Dahleh, Munther and Shah, Devavrat and Shen, Dennis},
  booktitle = {Proceedings of Thirty Sixth Conference on Learning Theory},
  pages     = {3821--3826},
  year      = {2023},
  volume    = {195},
  series    = {Proceedings of Machine Learning Research},
  publisher = {PMLR},
  url       = {https://proceedings.mlr.press/v195/agarwal23c.html}
}

@article{abadie2010synthetic,
  title={Synthetic control methods for comparative case studies: Estimating the effect of California’s tobacco control program},
  author={Abadie, Alberto and Diamond, Alexis and Hainmueller, Jens},
  journal={Journal of the American statistical Association},
  volume={105},
  number={490},
  pages={493--505},
  year={2010},
  publisher={Taylor \& Francis}
}

@article{abadie2021using,
  author    = {Abadie, Alberto},
  title     = {Using Synthetic Controls: Feasibility, Data Requirements, and Methodological Aspects},
  journal   = {Journal of Economic Literature},
  year      = {2021},
  volume    = {59},
  number    = {2},
  pages     = {391--425},
  publisher = {American Economic Association},
  url       = {https://doi.org/10.1257/jel.20191450}
}

@article{lipsitch2010negative,
  title={Negative controls: a tool for detecting confounding and bias in observational studies},
  author={Lipsitch, Marc and Tchetgen, Eric Tchetgen and Cohen, Ted},
  journal={Epidemiology},
  volume={21},
  number={3},
  pages={383--388},
  year={2010},
  publisher={LWW}
}

@article{miao2018confounding,
  title     = {A Confounding Bridge Approach for Double Negative Control Inference on Causal Effects},
  author    = {Miao, Wang and Shi, Xu and Li, Yilin and Tchetgen Tchetgen, Eric J.},
  journal   = {Statistical Theory and Related Fields},
  volume    = {8},
  number    = {4},
  pages     = {262--273},
  year      = {2024},
  publisher = {Taylor \& Francis},
  url       = {https://doi.org/10.1080/24754269.2024.2390748}
}

@article{tchetgen2020introduction,
  title   = {An Introduction to Proximal Causal Learning},
  author  = {Tchetgen Tchetgen, Eric J. and Ying, Andrew and Cui, Yifan and Shi, Xu and Miao, Wang},
  journal = {arXiv preprint arXiv:2009.10982},
  year    = {2020},
  doi     = {10.48550/arXiv.2009.10982},
  url     = {https://arxiv.org/abs/2009.10982}
}

@inproceedings{mastouri2021proximal,
  title     = {Proximal Causal Learning with Kernels: Two-Stage Estimation and Moment Restriction},
  author    = {Mastouri, Afsaneh and Zhu, Yuchen and Gultchin, Limor and Korba, Anna and Silva, Ricardo and Kusner, Matt and Gretton, Arthur and Muandet, Krikamol},
  booktitle = {Proceedings of the 38th International Conference on Machine Learning},
  pages     = {7512--7523},
  year      = {2021},
  volume    = {139},
  series    = {Proceedings of Machine Learning Research},
  publisher = {PMLR},
  url       = {https://proceedings.mlr.press/v139/mastouri21a.html}
}

@inproceedings{singh2019kernel,
  title     = {Kernel Instrumental Variable Regression},
  author    = {Singh, Rahul and Sahani, Maneesh and Gretton, Arthur},
  booktitle = {Advances in Neural Information Processing Systems},
  volume    = {32},
  year      = {2019},
  url       ={https://proceedings.neurips.cc/paper_files/paper/2019/hash/17b3c7061788dbe82de5abe9f6fe22b3-Abstract.html}
}

@article{imai2011estimation,
  title={Estimation of heterogeneous treatment effects from randomized experiments, with application to the optimal planning of the get-out-the-vote campaign},
  author={Imai, Kosuke and Strauss, Aaron},
  journal={Political Analysis},
  volume={19},
  number={1},
  pages={1--19},
  year={2011},
  publisher={Cambridge University Press}
}

@article{kunzel2019causaltoolbox,
  title     = {CausalToolbox---Estimator Stability for Heterogeneous Treatment Effects},
  author    = {K{\"u}nzel, S{\"o}ren R. and Walter, Simon J. S. and Sekhon, Jasjeet S.},
  journal   = {Observational Studies},
  volume    = {5},
  number    = {2},
  pages     = {105--117},
  year      = {2019},
  publisher = {University of Pennsylvania Press},
  url       = {https://doi.org/10.1353/obs.2019.0005}
}

@article{wendling2018comparing,
  title     = {Comparing Methods for Estimation of Heterogeneous Treatment Effects Using Observational Data from Health Care Databases},
  author    = {Wendling, Thierry and Jung, Kenneth and Callahan, Alison and Schuler, Alejandro and Shah, Nigam H. and Gallego, Blanca},
  journal   = {Statistics in Medicine},
  volume    = {37},
  number    = {23},
  pages     = {3309--3324},
  year      = {2018},
  publisher = {Wiley},
  url       = {https://doi.org/10.1002/sim.7820}
}

@article{xie2012estimating,
  title     = {Estimating Heterogeneous Treatment Effects with Observational Data},
  author    = {Xie, Yu and Brand, Jennie E. and Jann, Ben},
  journal   = {Sociological Methodology},
  volume    = {42},
  number    = {1},
  pages     = {314--347},
  year      = {2012},
  publisher = {SAGE Publications},
  url       = {https://doi.org/10.1177/0081175012452652}
}

@article{kruskal1977three,
  title={Three-way arrays: rank and uniqueness of trilinear decompositions, with application to arithmetic complexity and statistics},
  author={Kruskal, Joseph B},
  journal={Linear algebra and its applications},
  volume={18},
  number={2},
  pages={95--138},
  year={1977},
  publisher={Elsevier}
}

@inproceedings{mazaheri2024synthetic,
  title     = {Synthetic Potential Outcomes and Causal Mixture Identifiability},
  author    = {Mazaheri, Bijan and Squires, Chandler and Uhler, Caroline},
  booktitle = {Proceedings of The 28th International Conference on Artificial Intelligence and Statistics},
  pages     = {4276--4284},
  year      = {2025},
  volume    = {258},
  series    = {Proceedings of Machine Learning Research},
  publisher = {PMLR},
  url       = {https://proceedings.mlr.press/v258/mazaheri25a.html}
}

@book{snow1855mode,
  author    = {Snow, John},
  title     = {On the Mode of Communication of Cholera},
  edition   = {2nd},
  publisher = {John Churchill},
  address   = {London},
  year      = {1855},
  url       = {https://www.gutenberg.org/ebooks/72894}
}

@article{pearson1896regression,
  author    = {Pearson, Karl},
  title     = {Mathematical Contributions to the Theory of Evolution.
               {III}. Regression, Heredity, and Panmixia},
  journal   = {Philosophical Transactions of the Royal Society of London,
               Series A},
  volume    = {187},
  pages     = {253--318},
  year      = {1896}
}

@article{neyman1923application,
  author    = {Spława-Neyman, Jerzy},
  title     = {On the Application of Probability Theory to Agricultural
               Experiments: Essay on Principles, Section 9},
  journal   = {Statistical Science},
  volume    = {5},
  number    = {4},
  pages     = {465--472},
  year      = {1990},
  doi       = {10.1214/ss/1177012031},
  note      = {Translated from the Polish original (1923) by
               D.~M.~Dabrowska and T.~P.~Speed}
}

@article{rubin1974estimating,
  author    = {Rubin, Donald B.},
  title     = {Estimating Causal Effects of Treatments in Randomized
               and Nonrandomized Studies},
  journal   = {Journal of Educational Psychology},
  volume    = {66},
  number    = {5},
  pages     = {688--701},
  year      = {1974}
}

@article{haavelmo1944probability,
  author    = {Haavelmo, Trygve},
  title     = {The Probability Approach in Econometrics},
  journal   = {Econometrica},
  volume    = {12},
  pages     = {1--115},
  year      = {1944},
  note      = {Supplement}
}

@article{penrose1955generalized,
  title = {A Generalized Inverse for Matrices},
  author = {Penrose, Roger},
  journal = {Mathematical Proceedings of the Cambridge Philosophical Society},
  volume = {51},
  number = {3},
  pages = {406--413},
  year = {1955},
  publisher = {Cambridge University Press},
  doi = {10.1017/S0305004100030401},
  url = {https://doi.org/10.1017/S0305004100030401}
}

@book{benisrael2003generalized,
  title = {Generalized Inverses: Theory and Applications},
  author = {Ben-Israel, Adi and Greville, Thomas N. E.},
  edition = {2},
  series = {CMS Books in Mathematics},
  publisher = {Springer},
  address = {New York},
  year = {2003},
  doi = {10.1007/b97366},
  isbn = {978-0-387-00293-4},
  url = {https://doi.org/10.1007/b97366}
}

@book{horn2012matrix,
  title = {Matrix Analysis},
  author = {Horn, Roger A. and Johnson, Charles R.},
  edition = {2},
  publisher = {Cambridge University Press},
  address = {Cambridge},
  year = {2012},
  isbn = {978-0-521-83940-2},
  doi = {10.1017/CBO9781139020411},
  url = {https://doi.org/10.1017/CBO9781139020411}
}

@article{wedin1972perturbation,
  title = {Perturbation Bounds in Connection with Singular Value Decomposition},
  author = {Wedin, Per-{\AA}ke},
  journal = {BIT Numerical Mathematics},
  volume = {12},
  pages = {99--111},
  year = {1972},
  doi = {10.1007/BF01932678},
  url = {https://doi.org/10.1007/BF01932678}
}

@article{tropp2012user,
  title = {User-Friendly Tail Bounds for Sums of Random Matrices},
  author = {Tropp, Joel A.},
  journal = {Foundations of Computational Mathematics},
  volume = {12},
  pages = {389--434},
  year = {2012},
  doi = {10.1007/s10208-011-9099-z},
  url = {https://doi.org/10.1007/s10208-011-9099-z}
}

@article{bauer1960norms,
  title = {Norms and Exclusion Theorems},
  author = {Bauer, Friedrich L. and Fike, C. T.},
  journal = {Numerische Mathematik},
  volume = {2},
  pages = {137--141},
  year = {1960},
  doi = {10.1007/BF01386217},
  url = {https://doi.org/10.1007/BF01386217}
}

@book{kato1995perturbation,
  title = {Perturbation Theory for Linear Operators},
  author = {Kato, Tosio},
  series = {Classics in Mathematics},
  publisher = {Springer},
  address = {Berlin},
  year = {1995},
  note = {Reprint of the 1980 edition},
  isbn = {978-3-540-58661-6},
  doi = {10.1007/978-3-642-66282-9},
  url = {https://doi.org/10.1007/978-3-642-66282-9}
}

@book{stewart1990matrix,
  title = {Matrix Perturbation Theory},
  author = {Stewart, G. W. and Sun, Ji-guang},
  publisher = {Academic Press},
  address = {Boston},
  year = {1990},
  isbn = {978-0-12-670230-9}
}

@InProceedings{pmlr-v258-sun25d,
  title = 	 {Spectral Representation for Causal Estimation with Hidden Confounders},
  author =       {Sun, Haotian and Moulin, Antoine and Ren, Tongzheng and Gretton, Arthur and Dai, Bo},
  booktitle = 	 {Proceedings of The 28th International Conference on Artificial Intelligence and Statistics},
  pages = 	 {2719--2727},
  year = 	 {2025},
  volume = 	 {258},
  series = 	 {Proceedings of Machine Learning Research},
  month = 	 {03--05 May},
  publisher =    {PMLR},
  url = 	 {https://proceedings.mlr.press/v258/sun25d.html}
}

@article{jaeger2000observable,
  title   = {Observable Operator Models for Discrete Stochastic Time Series},
  author  = {Jaeger, Herbert},
  journal = {Neural Computation},
  volume  = {12},
  number  = {6},
  pages   = {1371--1398},
  year    = {2000},
  doi     = {10.1162/089976600300015411},
  url     = {https://doi.org/10.1162/089976600300015411}
}

\newpage
\appendix
\addcontentsline{toc}{section}{Appendices}

\section{Algorithm}
\label{app:algorithm}
\begin{algorithm}[H]
\caption{Subspace-Projected Spectral SPO}
\label{alg:spectral_spo}
\KwIn{Dataset $\mathcal{D} = \{(Z_i, X_i, T_i, Y_i)\}_{i=1}^n$, with $X_{i,1}=1$}
\KwOut{Causal MTEs $\widehat{\bm{\tau}}$, Latent Proportions $\widehat{\bm{p}}$, Feature matrix $\widehat{\mathbf{B}}$}

\For{$t \in \{0,1\}$}{
    $\widehat{\mathbf{M}}_{ZX|t} \gets \frac{1}{n_t}\sum_{i:T_i=t} Z_i X_i^\top$;\quad
    $\widehat{\mathbf{M}}_{ZXY|t} \gets \frac{1}{n_t}\sum_{i:T_i=t} Y_i Z_i X_i^\top$
}

Form stacked proxy matrix $\widehat{\mathbf{M}}_{\text{stack}} = \bigl[\widehat{\mathbf{M}}_{ZX|0}^\top \;\big|\; \widehat{\mathbf{M}}_{ZX|1}^\top\bigr]^\top$\;

Set or select $k$ via rank-selection on $\widehat{\mathbf{M}}_{\text{stack}}$; extract top-$k$ right singular vectors $\widehat{\mathbf{V}}_k \in \mathbb{R}^{d_x \times k}$\;

\For{$t \in \{0,1\}$}{
    $\widehat{\mathbf{P}}_t \gets \widehat{\mathbf{M}}_{ZX|t}\widehat{\mathbf{V}}_k$;\quad
    $\widehat{\mathbf{O}}_t \gets \widehat{\mathbf{M}}_{ZXY|t}\widehat{\mathbf{V}}_k$\;
    $\widehat{\widetilde{\mathbf{Q}}}_t \gets \widehat{\mathbf{P}}_t^\dagger\,\widehat{\mathbf{O}}_t \in \mathbb{R}^{k\times k}$
}

Form difference operator $\Delta\widehat{\widetilde{\mathbf{Q}}} \gets \widehat{\widetilde{\mathbf{Q}}}_1 - \widehat{\widetilde{\mathbf{Q}}}_0$\;

Eigendecompose $(\Delta\widehat{\widetilde{\mathbf{Q}}})^\top$; let $\widehat{\bm{\tau}}$ be eigenvalues (ordered by increasing real part) and rows of $\widehat{\mathbf{W}}$ be associated left eigenvectors; discard negligible imaginary parts\;

$\widehat{\mathbf{B}}_{\text{unscaled}}^\top \gets \widehat{\mathbf{W}}\widehat{\mathbf{V}}_k^\top$\;

Normalize each row of $\widehat{\mathbf{B}}_{\text{unscaled}}^\top$ by its first entry (anchor $= 1$) to obtain $\widehat{\mathbf{B}}^\top$\;

$\hat{\bm{\mu}}_X \gets \frac{1}{n}\sum X_i$;\quad
$\widehat{\bm{p}}_{\text{raw}} \gets \widehat{\mathbf{B}}^\dagger\hat{\bm{\mu}}_X$\;

Project $\widehat{\bm{p}}_{\text{raw}}$ onto the probability simplex $\Delta^{k-1}$ to obtain $\widehat{\bm{p}}$\;
\end{algorithm}

\section{Proofs and Auxiliary Lemmas}
\label{app:proofs}

Throughout, for each $t \in \{0,1\}$,
\begin{align}
    \mathbf{M}_{ZX|t} &= \mathbf{A}_t \mathbf{D}_{U|t} \mathbf{B}^\top, \label{eq:app_fac1_new}\\
    \mathbf{M}_{ZXY|t} &= \mathbf{A}_t \mathbf{D}_{U|t} \mathbf{D}_{Y|t} \mathbf{B}^\top. \label{eq:app_fac2_new}
\end{align}
We write
\[
\mathbf{C}_t \triangleq \mathbf{A}_t \mathbf{D}_{U|t} \in \mathbb{R}^{d_z \times k}.
\]
Under Assumption~\ref{ass:richness} and strict latent positivity, each $\mathbf{C}_t$ has full column rank $k$, and $\mathbf{B}^\top$ has full row rank $k$.

\let\oldsubsection\subsection
\renewcommand{\subsection}[1]{\oldsubsection*{#1}\addcontentsline{toc}{subsection}{#1}}

\subsection{Auxiliary Lemmas}

\begin{lemma}[Pseudo-inverse of a full-rank product]
\label{lem:pinv_product}
Let $\mathbf{C} \in \mathbb{R}^{m \times k}$ have full column rank and let $\mathbf{F} \in \mathbb{R}^{k \times n}$ have full row rank. Then $(\mathbf{C}\mathbf{F})^\dagger = \mathbf{F}^\dagger \mathbf{C}^\dagger$.
\end{lemma}

\begin{proof}
Since $\mathbf{C}$ has full column rank, $\mathbf{C}^\dagger = (\mathbf{C}^\top \mathbf{C})^{-1}\mathbf{C}^\top$ and $\mathbf{C}^\dagger \mathbf{C} = \mathbf{I}_k$. Since $\mathbf{F}$ has full row rank, $\mathbf{F}^\dagger = \mathbf{F}^\top (\mathbf{F}\mathbf{F}^\top)^{-1}$ and $\mathbf{F}\mathbf{F}^\dagger = \mathbf{I}_k$. Set $\mathbf{X} \triangleq \mathbf{F}^\dagger \mathbf{C}^\dagger$ and verify the Moore--Penrose equations for $\mathbf{X}$ relative to $\mathbf{C}\mathbf{F}$. First,
\[
(\mathbf{C}\mathbf{F})\mathbf{X}(\mathbf{C}\mathbf{F})
=
\mathbf{C}\mathbf{F}\mathbf{F}^\dagger \mathbf{C}^\dagger \mathbf{C}\mathbf{F}
=
\mathbf{C}\mathbf{I}_k \mathbf{I}_k \mathbf{F}
=
\mathbf{C}\mathbf{F}.
\]
Second, $\mathbf{X}(\mathbf{C}\mathbf{F})\mathbf{X} = \mathbf{F}^\dagger \mathbf{C}^\dagger \mathbf{C}\mathbf{F}\mathbf{F}^\dagger \mathbf{C}^\dagger = \mathbf{F}^\dagger \mathbf{I}_k \mathbf{I}_k \mathbf{C}^\dagger = \mathbf{F}^\dagger \mathbf{C}^\dagger = \mathbf{X}$. Third,
\[
\big((\mathbf{C}\mathbf{F})\mathbf{X}\big)^\top
=
\big(\mathbf{C}\mathbf{F}\mathbf{F}^\dagger \mathbf{C}^\dagger\big)^\top
=
\big(\mathbf{C}\mathbf{C}^\dagger\big)^\top
=
\mathbf{C}\mathbf{C}^\dagger
=
(\mathbf{C}\mathbf{F})\mathbf{X},
\]
because $\mathbf{C}\mathbf{C}^\dagger$ is the orthogonal projector onto $\range(\mathbf{C})$. Fourth, $\big(\mathbf{X}(\mathbf{C}\mathbf{F})\big)^\top = \big(\mathbf{F}^\dagger \mathbf{C}^\dagger \mathbf{C}\mathbf{F}\big)^\top = \big(\mathbf{F}^\dagger \mathbf{F}\big)^\top = \mathbf{F}^\dagger \mathbf{F} = \mathbf{X}(\mathbf{C}\mathbf{F})$, because $\mathbf{F}^\dagger \mathbf{F}$ is the orthogonal projector onto $\range(\mathbf{F}^\top)$ and is therefore symmetric. Thus $\mathbf{X}$ satisfies all four Moore--Penrose equations, so $(\mathbf{C}\mathbf{F})^\dagger = \mathbf{F}^\dagger \mathbf{C}^\dagger$.
\end{proof}

\begin{lemma}[Projection onto the exact row space]
\label{lem:rowspace_projection}
Let $\mathbf{M}\in\mathbb{R}^{m\times d}$ have rank $k$, and let
$\mathbf{V}\in\mathbb{R}^{d\times k}$ have orthonormal columns spanning
$\range(\mathbf M^\top)$. Define $\mathbf C\triangleq \mathbf M\mathbf V$.
Then:
\begin{enumerate}
    \item $\mathbf M=\mathbf C\mathbf V^\top$ and $\rank(\mathbf C)=k$;
    \item the nonzero singular values of $\mathbf M$ and $\mathbf M\mathbf V$
    coincide;
    \item in particular,
    \[
    \sigma_j(\mathbf M\mathbf V)=\sigma_j(\mathbf M),
    \qquad j=1,\dots,k.
    \]
\end{enumerate}
If additionally $m=k$, then $\mathbf C=\mathbf M\mathbf V$ is invertible.
\end{lemma}

The lemma is an exact row-space compression identity that follows from standard singular-value and orthogonal-projection facts; see, e.g., \citet{horn2012matrix}.

\begin{figure}[!htbp]
    \centering
    \includegraphics[width=0.92\textwidth]{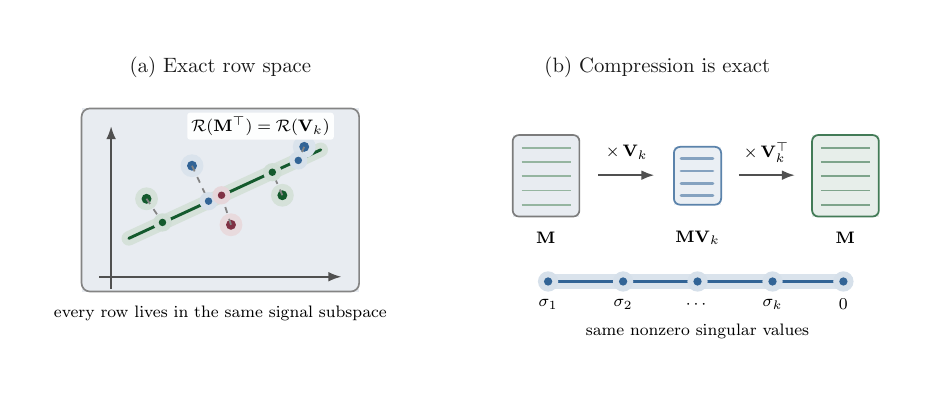}
    \caption{
    Row-space compression underlying Lemma~\ref{lem:rowspace_projection}.
    \textbf{Left:} The rows of $\mathbf M$ lie in the signal subspace
    $\range(\mathbf M^\top)=\range(\mathbf V_k)$, so projecting onto
    $\mathbf V_k$ preserves the population row space.
    \textbf{Right:} Compression is exact because
    $\mathbf M=\mathbf M\mathbf V_k\mathbf V_k^\top$; equivalently,
    $\mathbf M$ and $\mathbf M\mathbf V_k$ have the same nonzero singular values.
    }
    \label{fig:rowspace_projection}
\end{figure}

\begin{proof}
Because the columns of $\mathbf V$ form an orthonormal basis for
$\range(\mathbf M^\top)$, the matrix $\mathbf P\triangleq \mathbf V\mathbf V^\top$
is the orthogonal projector onto $\range(\mathbf M^\top)$. Hence
$\mathbf P\mathbf M^\top=\mathbf M^\top$, and therefore
\[
\mathbf M=\mathbf M\mathbf P=\mathbf M\mathbf V\mathbf V^\top.
\]
With $\mathbf C\triangleq \mathbf M\mathbf V$, this gives
\[
\mathbf M=\mathbf C\mathbf V^\top.
\]
Moreover,
\[
\mathbf M\mathbf M^\top
=
\mathbf C\mathbf V^\top\mathbf V\mathbf C^\top
=
\mathbf C\mathbf C^\top.
\]
Thus $\mathbf M\mathbf M^\top$ and $\mathbf C\mathbf C^\top$ have exactly the
same nonzero eigenvalues. Equivalently, $\mathbf M$ and
$\mathbf C=\mathbf M\mathbf V$ have exactly the same nonzero singular values.
Since $\rank(\mathbf M)=k$, there are precisely $k$ positive singular values,
so
\[
\sigma_j(\mathbf M\mathbf V)=\sigma_j(\mathbf M),
\qquad j=1,\dots,k.
\]
Finally, if $m=k$, then $\mathbf C\in\mathbb R^{k\times k}$ has rank $k$, and is
therefore invertible.
\end{proof}

\begin{lemma}[Concentration of treatment-conditional moment matrices]
\label{lem:conditional_moment_concentration}
Let $(Y_i,T_i,X_i,Z_i)_{i=1}^n$ be i.i.d., with
$X_i\in\mathbb R^{d_x}$ and $Z_i\in\mathbb R^{d_z}$. Assume
\[
\|ZX^\top\|_2\le L_{ZX},\qquad \|YZX^\top\|_2\le L_{ZXY}
\qquad\text{a.s.},
\]
and put $\pi\triangleq \min_{t\in\{0,1\}}\Pr(T=t)>0$. For $t\in\{0,1\}$, define
$n_t=\sum_{i=1}^n\mathbbm 1\{T_i=t\}$ and, when $n_t>0$,
\[
\widehat{\mathbf M}_{ZX|t}=\frac1{n_t}\sum_{i:T_i=t}Z_iX_i^\top,\qquad
\widehat{\mathbf M}_{ZXY|t}=\frac1{n_t}\sum_{i:T_i=t}Y_iZ_iX_i^\top .
\]
There exist universal constants $c,C>0$ such that, if
\[
n\pi\ge c\log\frac{8(d_z+d_x)}{\eta},
\]
then with probability at least $1-\eta$, simultaneously for $t=0,1$,
\[
n_t\ge \frac{n\pi}{2},
\]
and
\[
\begin{aligned}
\|\widehat{\mathbf M}_{ZX|t}-\mathbf M_{ZX|t}\|_2
&\le
CL_{ZX}\sqrt{\frac{\log(8(d_z+d_x)/\eta)}{n\pi}}, \\
\|\widehat{\mathbf M}_{ZXY|t}-\mathbf M_{ZXY|t}\|_2
&\le
CL_{ZXY}\sqrt{\frac{\log(8(d_z+d_x)/\eta)}{n\pi}} .
\end{aligned}
\]
\end{lemma}

\begin{proof}
Write $p_t=\Pr(T=t)$, so $p_t\ge\pi$. Since
$n_t=\sum_{i=1}^n\mathbbm 1\{T_i=t\}\sim{\rm Binomial}(n,p_t)$, the
multiplicative Chernoff bound gives
\[
\Pr\!\left(n_t<\frac{np_t}{2}\right)\le \exp\!\left(-\frac{np_t}{8}\right)
\le \exp\!\left(-\frac{n\pi}{8}\right).
\]
Hence, for
\[
\mathcal E\triangleq
\left\{n_t\ge \frac{np_t}{2}\ge\frac{n\pi}{2}\;\text{ for }t=0,1\right\},
\]
a union bound yields
\[
\Pr(\mathcal E^c)\le 2\exp\!\left(-\frac{n\pi}{8}\right)\le \frac{\eta}{4},
\]
after increasing the universal constant $c$ in the assumed lower bound on
$n\pi$.

Fix $t$ and condition on $S_t\triangleq\{i:T_i=t\}$ and $n_t=m$. Conditional on
$(S_t,n_t=m)$, the variables $\{(Y_i,X_i,Z_i):i\in S_t\}$ are i.i.d. from the
law of $(Y,X,Z)\mid T=t$. For the $ZX^\top$ moment, set
\[
\mathbf W_i^{(t)}\triangleq Z_iX_i^\top-\mathbf M_{ZX|t},\qquad i\in S_t.
\]
Then
\[
\E[\mathbf W_i^{(t)}\mid T_i=t]=0,\qquad
\widehat{\mathbf M}_{ZX|t}-\mathbf M_{ZX|t}
=\frac1m\sum_{i\in S_t}\mathbf W_i^{(t)}.
\]
Moreover,
\[
\|\mathbf W_i^{(t)}\|_2
\le \|Z_iX_i^\top\|_2+\|\mathbf M_{ZX|t}\|_2
\le L_{ZX}+\E[\|ZX^\top\|_2\mid T=t]\le 2L_{ZX},
\]
and the rectangular Bernstein variance proxies satisfy
\[
\left\|\sum_{i\in S_t}\E[\mathbf W_i^{(t)}(\mathbf W_i^{(t)})^\top\mid T_i=t]\right\|_2
\le \sum_{i\in S_t}\E[\|\mathbf W_i^{(t)}\|_2^2\mid T_i=t]\le 4mL_{ZX}^2,
\]
\[
\left\|\sum_{i\in S_t}\E[(\mathbf W_i^{(t)})^\top\mathbf W_i^{(t)}\mid T_i=t]\right\|_2
\le \sum_{i\in S_t}\E[\|\mathbf W_i^{(t)}\|_2^2\mid T_i=t]\le 4mL_{ZX}^2.
\]
Therefore rectangular matrix Bernstein, obtained by the standard self-adjoint dilation argument
\citep{tropp2012user}, implies that, for every $\delta\in(0,1)$,
with conditional probability at least $1-\delta$,
\[
\|\widehat{\mathbf M}_{ZX|t}-\mathbf M_{ZX|t}\|_2
\le
C_0L_{ZX}
\left\{
\sqrt{\frac{\log((d_z+d_x)/\delta)}{m}}
+
\frac{\log((d_z+d_x)/\delta)}{m}
\right\}.
\]
The same argument applied to
\[
\mathbf G_i^{(t)}\triangleq Y_iZ_iX_i^\top-\mathbf M_{ZXY|t},\qquad i\in S_t,
\]
uses
\[
\E[\mathbf G_i^{(t)}\mid T_i=t]=0,\qquad
\|\mathbf G_i^{(t)}\|_2
\le \|Y_iZ_iX_i^\top\|_2+\|\mathbf M_{ZXY|t}\|_2\le 2L_{ZXY},
\]
and gives, with conditional probability at least $1-\delta$,
\[
\|\widehat{\mathbf M}_{ZXY|t}-\mathbf M_{ZXY|t}\|_2
\le
C_0L_{ZXY}
\left\{
\sqrt{\frac{\log((d_z+d_x)/\delta)}{m}}
+
\frac{\log((d_z+d_x)/\delta)}{m}
\right\}.
\]
Because these conditional bounds hold for every realization of $(S_t,n_t)$,
they also hold unconditionally. Taking $\delta=\eta/8$ and union-bounding over
the two matrices and two treatment arms, with probability at least $1-\eta/2$,
simultaneously for $t=0,1$,
\[
\|\widehat{\mathbf M}_{ZX|t}-\mathbf M_{ZX|t}\|_2
\le
C_0L_{ZX}\left\{
\sqrt{\frac{\Lambda_\eta}{n_t}}+\frac{\Lambda_\eta}{n_t}
\right\},
\qquad
\|\widehat{\mathbf M}_{ZXY|t}-\mathbf M_{ZXY|t}\|_2
\le
C_0L_{ZXY}\left\{
\sqrt{\frac{\Lambda_\eta}{n_t}}+\frac{\Lambda_\eta}{n_t}
\right\},
\]
where
\[
\Lambda_\eta\triangleq \log\frac{8(d_z+d_x)}{\eta}.
\]
On the intersection with $\mathcal E$, $n_t\ge n\pi/2$, so
\[
\sqrt{\frac{\Lambda_\eta}{n_t}}+\frac{\Lambda_\eta}{n_t}
\le
\sqrt{\frac{2\Lambda_\eta}{n\pi}}+\frac{2\Lambda_\eta}{n\pi}.
\]
Since $n\pi\ge c\Lambda_\eta$, the linear Bernstein term is dominated by the
square-root term:
\[
\frac{2\Lambda_\eta}{n\pi}
\le C_1\sqrt{\frac{\Lambda_\eta}{n\pi}},
\]
after enlarging $c$ and changing only universal constants. Thus
\[
\sqrt{\frac{\Lambda_\eta}{n_t}}+\frac{\Lambda_\eta}{n_t}
\le C_2\sqrt{\frac{\Lambda_\eta}{n\pi}}.
\]
Substituting this into the preceding simultaneous Bernstein bounds and
combining with $\Pr(\mathcal E^c)\le\eta/4$ yields the claimed inequalities
with total failure probability at most $\eta$.
\end{proof}

\begin{lemma}[Concentration of the target-proxy mean]
\label{lem:mean_concentration}
Let \(X_1,\ldots,X_n\in\R^{d_x}\) be i.i.d. with
\(\|X_i\|_2\le L_X\) almost surely. Let
\[
\widehat{\bm\mu}_X\triangleq \frac1n\sum_{i=1}^n X_i,
\qquad
\bm\mu_X\triangleq \E[X].
\]
There exist universal constants \(c,C_\mu>0\) such that, if
\[
n\ge c\log\frac{4d_x}{\eta},
\]
then, with probability at least \(1-\eta\),
\[
\|\widehat{\bm\mu}_X-\bm\mu_X\|_2
\le
C_\mu L_X
\sqrt{\frac{\log(4d_x/\eta)}{n}}.
\]
\end{lemma}

\begin{proof}
Set \(\mathbf v_i\triangleq X_i-\bm\mu_X\). Then
\(\E[\mathbf v_i]=0\) and
\[
\|\mathbf v_i\|_2
\le
\|X_i\|_2+\|\bm\mu_X\|_2
\le
2L_X.
\]
View \(\mathbf v_i\) as a \(d_x\times 1\) random matrix. The rectangular matrix Bernstein inequality applied to
\(\sum_{i=1}^n \mathbf v_i\) gives, for a universal constant \(C>0\),
\[
\left\|
\sum_{i=1}^n \mathbf v_i
\right\|_2
\le
C L_X
\left\{
\sqrt{n\log(4d_x/\eta)}
+
\log(4d_x/\eta)
\right\}
\]
with probability at least \(1-\eta\). Dividing by \(n\) yields
\[
\|\widehat{\bm\mu}_X-\bm\mu_X\|_2
\le
C L_X
\left\{
\sqrt{\frac{\log(4d_x/\eta)}{n}}
+
\frac{\log(4d_x/\eta)}{n}
\right\}.
\]
When \(n\ge c\log(4d_x/\eta)\), the linear term is dominated by the square-root term after increasing the universal constant. This proves the claim.
\end{proof}

\begin{lemma}[Local Lipschitzness of anchor normalization]
\label{lem:anchor_normalization}
Let \(\mathbf y\in\R^d\) satisfy \(\|\mathbf y\|_2=1\) and
\[
a\triangleq |y_1|>0.
\]
Let \(\mathbf x\in\R^d\) satisfy \(\|\mathbf x\|_2=1\) and
\[
\|\mathbf x-\mathbf y\|_2\le \frac{a}{2}.
\]
Then \(x_1\neq0\), and
\[
\left\|
\frac{\mathbf x}{x_1}
-
\frac{\mathbf y}{y_1}
\right\|_2
\le
4a^{-2}\|\mathbf x-\mathbf y\|_2.
\]
\end{lemma}

\begin{proof}
Let \(\varepsilon\triangleq\|\mathbf x-\mathbf y\|_2\). Since
\[
|x_1-y_1|\le \varepsilon\le a/2,
\]
we have
\[
|x_1|\ge |y_1|-|x_1-y_1|\ge a/2,
\]
so \(x_1\neq0\). Now
\[
\frac{\mathbf x}{x_1}-\frac{\mathbf y}{y_1}
=
\frac{\mathbf x-\mathbf y}{x_1}
+
\mathbf y\left(\frac1{x_1}-\frac1{y_1}\right).
\]
Therefore, using \(\|\mathbf y\|_2=1\),
\[
\begin{aligned}
\left\|
\frac{\mathbf x}{x_1}-\frac{\mathbf y}{y_1}
\right\|_2
&\le
\frac{\|\mathbf x-\mathbf y\|_2}{|x_1|}
+
\frac{|x_1-y_1|}{|x_1||y_1|}  \\
&\le
\frac{2\varepsilon}{a}
+
\frac{2\varepsilon}{a^2}.
\end{aligned}
\]
Since \(a\le \|\mathbf y\|_2=1\), we have \(a^{-1}\le a^{-2}\), and hence
\[
\frac{2\varepsilon}{a}
+
\frac{2\varepsilon}{a^2}
\le
4a^{-2}\varepsilon.
\]
This proves the claim.
\end{proof}

\begin{lemma}[Nonexpansiveness of simplex projection]
\label{lem:simplex_projection}
Let \(\Delta^{k-1}\subset\R^k\) be the probability simplex, and let
\(\Pi_{\Delta^{k-1}}\) denote Euclidean projection onto it. Then, for all
\(\mathbf x,\mathbf y\in\R^k\),
\[
\left\|
\Pi_{\Delta^{k-1}}(\mathbf x)
-
\Pi_{\Delta^{k-1}}(\mathbf y)
\right\|_2
\le
\|\mathbf x-\mathbf y\|_2.
\]
In particular, if \(\bm p\in\Delta^{k-1}\), then
\[
\left\|
\Pi_{\Delta^{k-1}}(\mathbf x)-\bm p
\right\|_2
\le
\|\mathbf x-\bm p\|_2.
\]
\end{lemma}

\begin{proof}
Euclidean projection onto a closed convex set is firmly nonexpansive. For completeness, let
\(\mathbf u=\Pi_{\Delta^{k-1}}(\mathbf x)\) and
\(\mathbf v=\Pi_{\Delta^{k-1}}(\mathbf y)\). The variational inequality for Euclidean projection gives
\[
(\mathbf x-\mathbf u)^\top(\mathbf v-\mathbf u)\le0,
\qquad
(\mathbf y-\mathbf v)^\top(\mathbf u-\mathbf v)\le0.
\]
Adding the two inequalities yields
\[
(\mathbf x-\mathbf y)^\top(\mathbf u-\mathbf v)
\ge
\|\mathbf u-\mathbf v\|_2^2.
\]
By Cauchy's inequality,
\[
\|\mathbf u-\mathbf v\|_2^2
\le
\|\mathbf x-\mathbf y\|_2\|\mathbf u-\mathbf v\|_2,
\]
and the claim follows. Taking \(\mathbf y=\bm p\in\Delta^{k-1}\) gives
\(\Pi_{\Delta^{k-1}}(\bm p)=\bm p\), proving the final statement.
\end{proof}

\begin{lemma}[Pseudo-inverse perturbation for full-column-rank matrices]
\label{lem:pinv_perturb}
Let $\mathbf P\in\mathbb R^{m\times k}$ have full column rank and write
\[
\widehat{\mathbf P}=\mathbf P+\Delta\mathbf P,\qquad
\sigma_{\min}(\mathbf P)=\sigma>0.
\]
If $\|\Delta\mathbf P\|_2\le \sigma/2$, then $\widehat{\mathbf P}$ has full
column rank and
\[
\|\widehat{\mathbf P}^{\dagger}-\mathbf P^\dagger\|_2
\le
\frac{6}{\sigma^2}\|\Delta\mathbf P\|_2 .
\]
Consequently,
\[
\|\widehat{\mathbf P}^{\dagger}-\mathbf P^\dagger\|_2
=
\mathcal O\!\left(\frac{\|\Delta\mathbf P\|_2}{\sigma^2}\right).
\]
\end{lemma}

\begin{proof}
By Weyl's inequality,
\[
\sigma_{\min}(\widehat{\mathbf P})
\ge
\sigma_{\min}(\mathbf P)-\|\Delta\mathbf P\|_2
\ge
\sigma/2,
\]
so $\widehat{\mathbf P}$ has full column rank and
\[
\|\mathbf P^\dagger\|_2=\sigma^{-1},\qquad
\|\widehat{\mathbf P}^\dagger\|_2
=\sigma_{\min}(\widehat{\mathbf P})^{-1}
\le 2\sigma^{-1}.
\]
For two matrices of the same full column rank, with
$\widehat{\mathbf P}=\mathbf P+\Delta\mathbf P$, the standard perturbation
identity for the Moore--Penrose inverse \citep{benisrael2003generalized,stewart1990matrix} gives
\[
\widehat{\mathbf P}^{\dagger}-\mathbf P^\dagger
=
-\widehat{\mathbf P}^{\dagger}\Delta\mathbf P\,\mathbf P^\dagger
+
\widehat{\mathbf P}^{\dagger}(\widehat{\mathbf P}^{\dagger})^\top
(\Delta\mathbf P)^\top
(\mathbf I-\mathbf P\mathbf P^\dagger).
\]
For completeness, this identity follows by expanding the right-hand side:
\[
\begin{aligned}
&-\widehat{\mathbf P}^{\dagger}(\widehat{\mathbf P}-\mathbf P)\mathbf P^\dagger
+\widehat{\mathbf P}^{\dagger}(\widehat{\mathbf P}^{\dagger})^\top
(\widehat{\mathbf P}^\top-\mathbf P^\top)(\mathbf I-\mathbf P\mathbf P^\dagger)  \\
&\quad
=
-\widehat{\mathbf P}^{\dagger}\widehat{\mathbf P}\mathbf P^\dagger
+\widehat{\mathbf P}^{\dagger}\mathbf P\mathbf P^\dagger
+\widehat{\mathbf P}^{\dagger}(\widehat{\mathbf P}^{\dagger})^\top
\widehat{\mathbf P}^{\top}(\mathbf I-\mathbf P\mathbf P^\dagger)
-\widehat{\mathbf P}^{\dagger}(\widehat{\mathbf P}^{\dagger})^\top
\mathbf P^\top(\mathbf I-\mathbf P\mathbf P^\dagger)  \\
&\quad
=
-\mathbf P^\dagger
+\widehat{\mathbf P}^{\dagger}\mathbf P\mathbf P^\dagger
+\widehat{\mathbf P}^{\dagger}(\mathbf I-\mathbf P\mathbf P^\dagger)
-0
=
\widehat{\mathbf P}^{\dagger}-\mathbf P^\dagger,
\end{aligned}
\]
where we used $\widehat{\mathbf P}^{\dagger}\widehat{\mathbf P}=\mathbf I_k$,
$\widehat{\mathbf P}^{\dagger}(\widehat{\mathbf P}^{\dagger})^\top
\widehat{\mathbf P}^\top=\widehat{\mathbf P}^{\dagger}$, and
$\mathbf P^\top(\mathbf I-\mathbf P\mathbf P^\dagger)=0$. Hence,
since $\mathbf P\mathbf P^\dagger$ is an orthogonal projector,
\[
\begin{aligned}
\|\widehat{\mathbf P}^{\dagger}-\mathbf P^\dagger\|_2
&\le
\|\widehat{\mathbf P}^{\dagger}\|_2
\|\Delta\mathbf P\|_2
\|\mathbf P^\dagger\|_2
+
\|\widehat{\mathbf P}^{\dagger}\|_2^2
\|\Delta\mathbf P\|_2
\|\mathbf I-\mathbf P\mathbf P^\dagger\|_2  \\
&\le
\frac{2}{\sigma^2}\|\Delta\mathbf P\|_2
+
\frac{4}{\sigma^2}\|\Delta\mathbf P\|_2
=
\frac{6}{\sigma^2}\|\Delta\mathbf P\|_2 .
\end{aligned}
\]
\end{proof}

\begin{lemma}[Eigenvalue localization for a diagonal matrix plus perturbation]
\label{lem:diag_localization}
Let $\mathbf\Lambda=\diag(\lambda_1,\ldots,\lambda_k)$ and let
$\mathbf F\in\mathbb R^{k\times k}$. Then every
$\mu\in\spec(\mathbf\Lambda+\mathbf F)$ satisfies
\[
\min_{1\le j\le k}|\mu-\lambda_j|\le \|\mathbf F\|_2 .
\]
If $\lambda_1,\ldots,\lambda_k$ are pairwise distinct and, for a fixed $u$,
\[
2\|\mathbf F\|_2<\delta_u,\qquad
\delta_u\triangleq\min_{v\neq u}|\lambda_u-\lambda_v|,
\]
then $\mathbf\Lambda+\mathbf F$ has exactly one eigenvalue
$\widehat\lambda_u$ in
\[
D_u\triangleq\{z\in\mathbb C:|z-\lambda_u|<\delta_u/2\},
\]
and
\[
|\widehat\lambda_u-\lambda_u|\le \|\mathbf F\|_2 .
\]
\end{lemma}

The first part is the diagonal special case of the Bauer--Fike eigenvalue perturbation theorem \citep{bauer1960norms}; the local counting argument is inspired by the Riesz-projection continuity argument from analytic perturbation theory \citep{kato1995perturbation}.

\begin{proof}
Let $\mu\in\spec(\mathbf\Lambda+\mathbf F)$ and suppose, toward contradiction,
that $\min_j|\mu-\lambda_j|>\|\mathbf F\|_2$. Then
$\mathbf\Lambda-\mu\mathbf I$ is invertible and
\[
\|(\mathbf\Lambda-\mu\mathbf I)^{-1}\mathbf F\|_2
\le
\frac{\|\mathbf F\|_2}{\min_j|\mu-\lambda_j|}<1.
\]
Hence $\mathbf I+(\mathbf\Lambda-\mu\mathbf I)^{-1}\mathbf F$ is invertible,
and therefore
\[
\mathbf\Lambda+\mathbf F-\mu\mathbf I
=
(\mathbf\Lambda-\mu\mathbf I)
\bigl[\mathbf I+(\mathbf\Lambda-\mu\mathbf I)^{-1}\mathbf F\bigr]
\]
is invertible, contradicting $\mu\in\spec(\mathbf\Lambda+\mathbf F)$. Thus
\[
\min_j|\mu-\lambda_j|\le \|\mathbf F\|_2 . \tag{1}
\]

For the local assertion, fix $u$ and define
\[
\mathbf A(\alpha)\triangleq\mathbf\Lambda+\alpha\mathbf F,\qquad
\Gamma_u\triangleq\{z\in\mathbb C:|z-\lambda_u|=\delta_u/2\},
\qquad
\alpha\in[0,1].
\]
For $z\in\Gamma_u$,
\[
\min_j|z-\lambda_j|
\ge
\min\left\{|z-\lambda_u|,\min_{v\neq u}\bigl(|\lambda_v-\lambda_u|-|z-\lambda_u|\bigr)\right\}
\ge
\delta_u/2.
\]
Since $\|\alpha\mathbf F\|_2\le\|\mathbf F\|_2<\delta_u/2$, applying the
resolvent argument above to $\mathbf\Lambda+\alpha\mathbf F$ gives
\[
\Gamma_u\cap\spec(\mathbf A(\alpha))=\varnothing
\qquad(\alpha\in[0,1]).
\]
Consequently the Riesz spectral count
\[
N(\alpha)
\triangleq
\frac{1}{2\pi i}\int_{\Gamma_u}
\operatorname{tr}\!\left((z\mathbf I-\mathbf A(\alpha))^{-1}\right)\,dz
\]
is well-defined, continuous in $\alpha$, and integer-valued; hence it is
constant. Since $N(0)=1$, because only $\lambda_u$ lies inside $\Gamma_u$, we
obtain $N(1)=1$. Thus $\mathbf\Lambda+\mathbf F$ has exactly one eigenvalue in
$D_u$; denote it by $\widehat\lambda_u$.

It remains only to sharpen its location. By (1),
$\min_j|\widehat\lambda_u-\lambda_j|\le\|\mathbf F\|_2$. If a minimizer were
$v\neq u$, then, using $\widehat\lambda_u\in D_u$,
\[
\delta_u
\le
|\lambda_v-\lambda_u|
\le
|\lambda_v-\widehat\lambda_u|+|\widehat\lambda_u-\lambda_u|
<
\|\mathbf F\|_2+\delta_u/2
<
\delta_u,
\]
a contradiction. Hence the minimizing index is $u$, and
$|\widehat\lambda_u-\lambda_u|\le\|\mathbf F\|_2$.
\end{proof}

\begin{lemma}[Left eigenvector perturbation for a simple diagonal eigenvalue]
\label{lem:left_eig_diag}
Let $\mathbf\Lambda=\diag(\lambda_1,\ldots,\lambda_k)$, let
$\mathbf F\in\mathbb R^{k\times k}$, and fix $u\in\{1,\ldots,k\}$. Suppose
\[
4\|\mathbf F\|_2<\delta_u,\qquad
\delta_u\triangleq\min_{v\neq u}|\lambda_u-\lambda_v|.
\]
Let $\widehat\lambda_u$ be the unique eigenvalue of $\mathbf\Lambda+\mathbf F$
in $|\mu-\lambda_u|<\delta_u/2$. Then there exists a left eigenvector
$\widehat{\mathbf y}_u^\top\in\mathbb C^{1\times k}$ associated with
$\widehat\lambda_u$ such that
\[
\widehat{\mathbf y}_u^\top=\mathbf e_u^\top+\mathbf r_u^\top,\qquad
\mathbf e_u^\top\mathbf r_u=0,\qquad
\|\mathbf r_u\|_2\le
\frac{\|\mathbf F\|_2}{\delta_u-2\|\mathbf F\|_2}
\le
\frac{2\|\mathbf F\|_2}{\delta_u}.
\]
Consequently,
\[
\sin\theta(\widehat{\mathbf y}_u^\top,\mathbf e_u^\top)
\le
\|\mathbf r_u\|_2
\le
\frac{2\|\mathbf F\|_2}{\delta_u}.
\]
\end{lemma}

This is an block-resolvent specialization of the standard perturbation theory for a simple isolated eigenvalue; see \citet{kato1995perturbation,stewart1990matrix}. We give the direct proof because the diagonal population form yields sharper constants and simpler notation. 

\begin{proof}
After a permutation of coordinates, take $u=1$ and write
\[
\mathbf\Lambda+\mathbf F=
\begin{bmatrix}
\lambda_1+f_{11} & \mathbf f_{12}^\top\\
\mathbf f_{21} & \mathbf\Lambda_{-1}+\mathbf F_{22}
\end{bmatrix},
\qquad
\mathbf\Lambda_{-1}=\diag(\lambda_2,\ldots,\lambda_k).
\]
By Lemma~\ref{lem:diag_localization},
$|\widehat\lambda_1-\lambda_1|\le \|\mathbf F\|_2$, hence, for every
$v\neq1$,
\[
|\lambda_v-\widehat\lambda_1|
\ge |\lambda_v-\lambda_1|-|\widehat\lambda_1-\lambda_1|
\ge \delta_1-\|\mathbf F\|_2,
\qquad
\|(\mathbf\Lambda_{-1}-\widehat\lambda_1\mathbf I)^{-1}\|_2
\le \frac1{\delta_1-\|\mathbf F\|_2}.
\]
Moreover,
\[
\|(\mathbf\Lambda_{-1}-\widehat\lambda_1\mathbf I)^{-1}\mathbf F_{22}\|_2
\le
\frac{\|\mathbf F\|_2}{\delta_1-\|\mathbf F\|_2}<1,
\]
so
\[
\mathbf B\triangleq \mathbf\Lambda_{-1}+\mathbf F_{22}
-\widehat\lambda_1\mathbf I
=
(\mathbf\Lambda_{-1}-\widehat\lambda_1\mathbf I)
\Bigl[\mathbf I+
(\mathbf\Lambda_{-1}-\widehat\lambda_1\mathbf I)^{-1}\mathbf F_{22}\Bigr]
\]
is invertible and
\[
\|\mathbf B^{-1}\|_2
\le
\frac{1}{\delta_1-\|\mathbf F\|_2}
\left(1-\frac{\|\mathbf F\|_2}{\delta_1-\|\mathbf F\|_2}\right)^{-1}
=
\frac1{\delta_1-2\|\mathbf F\|_2}.
\]

Let $\widehat{\mathbf y}_1^\top=(a,\mathbf z^\top)$ be a nonzero left
eigenvector for $\widehat\lambda_1$. If $a=0$, then
$\mathbf z^\top\mathbf B=0$, contradicting the invertibility of $\mathbf B$;
hence $a\neq0$, and rescaling gives
$\widehat{\mathbf y}_1^\top=(1,\mathbf x^\top)$. The last $k-1$ coordinates of
$\widehat{\mathbf y}_1^\top(\mathbf\Lambda+\mathbf F)
=\widehat\lambda_1\widehat{\mathbf y}_1^\top$ yield
\[
\mathbf f_{12}^\top+\mathbf x^\top(\mathbf\Lambda_{-1}+\mathbf F_{22})
=
\widehat\lambda_1\mathbf x^\top,
\qquad
\mathbf x^\top=-\mathbf f_{12}^\top\mathbf B^{-1}.
\]
Therefore
\[
\|\mathbf x\|_2
\le
\|\mathbf f_{12}\|_2\|\mathbf B^{-1}\|_2
\le
\frac{\|\mathbf F\|_2}{\delta_1-2\|\mathbf F\|_2}
\le
\frac{2\|\mathbf F\|_2}{\delta_1},
\]
where the last inequality follows from $4\|\mathbf F\|_2<\delta_1$. Setting
$\mathbf r_1^\top=(0,\mathbf x^\top)$ gives
$\widehat{\mathbf y}_1^\top=\mathbf e_1^\top+\mathbf r_1^\top$ and
$\mathbf e_1^\top\mathbf r_1=0$. Finally,
\[
\sin\theta(\mathbf e_1^\top+\mathbf r_1^\top,\mathbf e_1^\top)
=
\frac{\|\mathbf r_1\|_2}{\sqrt{1+\|\mathbf r_1\|_2^2}}
\le
\|\mathbf r_1\|_2.
\]
Undoing the coordinate permutation proves the claim for general $u$.
\end{proof}

\begin{lemma}[Angle distortion under an invertible linear map]
\label{lem:angle_linear_map}
Let $\mathbb F\in\{\mathbb R,\mathbb C\}$, let
$\mathbf S\in\mathbb F^{k\times k}$ be invertible, and let
$\mathbf x,\mathbf y\in\mathbb F^k\setminus\{0\}$. Then $\sin\theta(\mathbf S\mathbf x,\mathbf S\mathbf y) \le \kappa(\mathbf S)\sin\theta(\mathbf x,\mathbf y)$, where $\kappa(\mathbf S)\triangleq \|\mathbf S\|_2\|\mathbf S^{-1}\|_2$. Equivalently, for nonzero row vectors $\mathbf x^\top,\mathbf y^\top$, $\sin\theta(\mathbf x^\top\mathbf S,\mathbf y^\top\mathbf S) \le \kappa(\mathbf S)\sin\theta(\mathbf x^\top,\mathbf y^\top)$.
\end{lemma}

\begin{proof}
The inequality is homogeneous in $\mathbf{x}$ and $\mathbf{y}$, so we may normalize them without loss of generality, and thus assume $\|\mathbf{x}\|_2 = \|\mathbf{y}\|_2 = 1$. We use the standard characterization of the sine of the principal angle between one-dimensional subspaces: for any nonzero vectors $\mathbf{u},\mathbf{v}$,
\[
\sin \theta(\mathbf{u},\mathbf{v})
=
\min_{\alpha \in \mathbb F}
\frac{\|\mathbf{u}-\alpha \mathbf{v}\|_2}{\|\mathbf{u}\|_2}.
\]
Applying this identity with $\mathbf{u} = \mathbf{S}\mathbf{x}$ and $\mathbf{v} = \mathbf{S}\mathbf{y}$, we obtain
\[
\sin \theta(\mathbf{S}\mathbf{x},\mathbf{S}\mathbf{y})
=
\min_{\alpha \in \mathbb F}
\frac{\|\mathbf{S}\mathbf{x}-\alpha \mathbf{S}\mathbf{y}\|_2}{\|\mathbf{S}\mathbf{x}\|_2}.
\]
Next, use linearity of $\mathbf{S}$ to factor the numerator: $\mathbf{S}\mathbf{x}-\alpha \mathbf{S}\mathbf{y} = \mathbf{S}(\mathbf{x}-\alpha \mathbf{y})$, so by the definition of the operator norm,
\[
\|\mathbf{S}\mathbf{x}-\alpha \mathbf{S}\mathbf{y}\|_2
=
\|\mathbf{S}(\mathbf{x}-\alpha \mathbf{y})\|_2
\le
\|\mathbf{S}\|_2 \,\|\mathbf{x}-\alpha \mathbf{y}\|_2.
\]
Therefore, for every $\alpha \in \mathbb F$,
\[
\frac{\|\mathbf{S}\mathbf{x}-\alpha \mathbf{S}\mathbf{y}\|_2}{\|\mathbf{S}\mathbf{x}\|_2}
\le
\frac{\|\mathbf{S}\|_2\,\|\mathbf{x}-\alpha \mathbf{y}\|_2}{\|\mathbf{S}\mathbf{x}\|_2}.
\]
We now lower-bound the denominator. Since $\mathbf{S}$ is invertible, $\|\mathbf{S}\mathbf{x}\|_2 \ge \sigma_{\min}(\mathbf{S}) \|\mathbf{x}\|_2$, and because we normalized $\|\mathbf{x}\|_2=1$, this becomes $\|\mathbf{S}\mathbf{x}\|_2 \ge \sigma_{\min}(\mathbf{S})$. Hence, still for every $\alpha \in \mathbb F$,
\[
\frac{\|\mathbf{S}\mathbf{x}-\alpha \mathbf{S}\mathbf{y}\|_2}{\|\mathbf{S}\mathbf{x}\|_2}
\le
\frac{\|\mathbf{S}\|_2}{\sigma_{\min}(\mathbf{S})}
\|\mathbf{x}-\alpha \mathbf{y}\|_2.
\]
Taking the minimum over $\alpha$ on both sides yields
\[
\sin \theta(\mathbf{S}\mathbf{x},\mathbf{S}\mathbf{y})
\le
\frac{\|\mathbf{S}\|_2}{\sigma_{\min}(\mathbf{S})}
\min_{\alpha \in \mathbb F} \|\mathbf{x}-\alpha \mathbf{y}\|_2.
\]
Since $\|\mathbf{x}\|_2=1$, the same one-dimensional formula gives
$\min_{\alpha \in \mathbb F} \|\mathbf{x}-\alpha \mathbf{y}\|_2
= \sin \theta(\mathbf{x},\mathbf{y})$. Substituting this in, and using $\|\mathbf{S}^{-1}\|_2 = 1/\sigma_{\min}(\mathbf{S})$, we obtain
\[
\sin \theta(\mathbf{S}\mathbf{x}, \mathbf{S}\mathbf{y})
\le
\frac{\|\mathbf{S}\|_2}{\sigma_{\min}(\mathbf{S})}
\sin \theta(\mathbf{x},\mathbf{y})
=
\|\mathbf{S}\|_2 \|\mathbf{S}^{-1}\|_2
\sin \theta(\mathbf{x},\mathbf{y})
=
\kappa(\mathbf{S}) \sin \theta(\mathbf{x}, \mathbf{y}),
\]
as claimed.
\end{proof}

\subsection{A. Proof of Theorem \ref{thm:dimensionality}}
\label{app:proof_thm_dimensionality_revised}

\begin{proof}
Define the stacked observable cross-moment matrix by
\[
\mathbf{M}_{\mathrm{stack}}
\triangleq
\begin{bmatrix}
\mathbf{M}_{ZX|0}\\[0.2em]
\mathbf{M}_{ZX|1}
\end{bmatrix}
\in \mathbb{R}^{2d_z \times d_x}.
\]
Using the treatment-specific factorization \eqref{eq:app_fac1_new}, we may write $\mathbf{M}_{ZX|0} = \mathbf{A}_0 \mathbf{D}_{U|0}\mathbf{B}^\top$ and $\mathbf{M}_{ZX|1} = \mathbf{A}_1 \mathbf{D}_{U|1}\mathbf{B}^\top$. Stacking these two block rows gives
\[
\mathbf{M}_{\mathrm{stack}}
=
\begin{bmatrix}
\mathbf{A}_0 \mathbf{D}_{U|0}\\[0.2em]
\mathbf{A}_1 \mathbf{D}_{U|1}
\end{bmatrix}
\mathbf{B}^\top
=:
\mathbf{C}_{\mathrm{stack}} \mathbf{B}^\top,
\qquad
\mathbf{C}_{\mathrm{stack}} \in \mathbb{R}^{2d_z \times k}.
\]
Our goal is to prove that $\rank(\mathbf{M}_{\mathrm{stack}})=k$. Since $\mathbf{B}$ has full column rank $k$ by Assumption~\ref{ass:richness}, $\mathbf{B}^\top$ has full row rank $k$. Thus the only remaining issue is whether $\mathbf{C}_{\mathrm{stack}}$ has full column rank. We now show that it does.

Let $\mathbf{w} \in \mathbb{R}^k$ satisfy $\mathbf{C}_{\mathrm{stack}} \mathbf{w} = \mathbf{0}$. Writing this blockwise, both block equations must hold separately:
\[
\mathbf{A}_0 \mathbf{D}_{U|0}\mathbf{w} = \mathbf{0},
\qquad
\mathbf{A}_1 \mathbf{D}_{U|1}\mathbf{w} = \mathbf{0}.
\]
Since each of $\mathbf{A}_0$, $\mathbf{A}_1$ has rank $k$, their Moore--Penrose pseudo-inverses satisfy $\mathbf{A}_0^\dagger \mathbf{A}_0 = \mathbf{I}_k$ and $\mathbf{A}_1^\dagger \mathbf{A}_1 = \mathbf{I}_k$. Applying $\mathbf{A}_t^\dagger$ to the respective equations gives
\[
\mathbf{D}_{U|0}\mathbf{w} = \mathbf{0},
\qquad
\mathbf{D}_{U|1}\mathbf{w} = \mathbf{0}.
\]

We now expand these two diagonal equations coordinatewise. Since $\mathbf{D}_{U|t} = \diag(\pr(U=1\mid T=t),\dots,\pr(U=k\mid T=t))$, the equations above mean that for each $u=1,\dots,k$,
\[
\pr(U=u \mid T=0)\, w_u = 0
\qquad\text{and}\qquad
\pr(U=u \mid T=1)\, w_u = 0.
\]
By the hypothesis of the theorem, for each $u \in \{1,\dots,k\}$ there exists at least one $t \in \{0,1\}$ such that $\pr(U=u \mid T=t) > 0$. For that value of $t$, the corresponding equation above forces $w_u = 0$. Because this argument holds for every $u=1,\dots,k$, we conclude that $\mathbf{w}=\mathbf{0}$.

We have therefore shown that $\ker(\mathbf{C}_{\mathrm{stack}})=\{0\}$, so $\rank(\mathbf{C}_{\mathrm{stack}})=k$. Since $\mathbf{C}_{\mathrm{stack}}$ has full column rank $k$ and $\mathbf{B}^\top$ has full row rank $k$, it follows that
\[
\rank(\mathbf{M}_{\mathrm{stack}})
=
\rank(\mathbf{C}_{\mathrm{stack}}\mathbf{B}^\top)
=
k.
\]
Substituting back the definition of $\mathbf{M}_{\mathrm{stack}}$, we obtain
\[
k
=
\rank\!\left(
\begin{bmatrix}
\mathbf{M}_{ZX|0}\\
\mathbf{M}_{ZX|1}
\end{bmatrix}
\right),
\]
which is exactly the claimed identification formula for the latent dimension.
\end{proof}

\subsection{B. Proof of Theorem \ref{thm:positivity}}
\begin{proof}
Starting from \eqref{eq:app_fac1_new}, for each treatment arm $t \in \{0,1\}$ we have $\mathbf{M}_{ZX|t} = \mathbf{A}_t \mathbf{D}_{U|t} \mathbf{B}^\top$. Taking transposes yields $\mathbf{M}_{ZX|t}^\top = \mathbf{B}\mathbf{D}_{U|t}\mathbf{A}_t^\top$. Because $\mathbf{A}_t$ has full column rank by Assumption~\ref{ass:richness}, the transpose $\mathbf{A}_t^\top$ has full row rank $k$. Hence right multiplication by $\mathbf{A}_t^\top$ does not alter the column space of the matrix to its left, and therefore
\[
\range(\mathbf{M}_{ZX|t}^\top)
=
\range(\mathbf{B}\mathbf{D}_{U|t}\mathbf{A}_t^\top)
=
\range(\mathbf{B}\mathbf{D}_{U|t}).
\]

Now assume that strict latent positivity holds, so that $\pr(T=t \mid U=u) > 0$ for every $u \in \{1,\dots,k\}$ and every $t \in \{0,1\}$. Since the theorem also assumes marginal support $\pr(U=u)>0$ for all $u$, Bayes' rule gives
\[
\pr(U=u \mid T=t)
=
\frac{\pr(T=t \mid U=u)\pr(U=u)}{\pr(T=t)}
>
0.
\]
Thus every diagonal entry of $\mathbf{D}_{U|t} = \diag(\pr(U=1\mid T=t),\dots,\pr(U=k\mid T=t))$ is strictly positive, so $\mathbf{D}_{U|t}$ is invertible. Since multiplication on the right by an invertible matrix does not change column space, $\range(\mathbf{B}\mathbf{D}_{U|t})=\range(\mathbf{B})$. Combining this with the identity above yields
\[
\range(\mathbf{M}_{ZX|0}^\top)
=
\range(\mathbf{M}_{ZX|1}^\top)
=
\range(\mathbf{B}).
\]

Conversely, suppose there exist $u^\ast \in \{1,\dots,k\}$ and $t^\ast \in \{0,1\}$ such that $\pr(T=t^\ast \mid U=u^\ast)=0$. Applying Bayes' rule gives $\pr(U=u^\ast \mid T=t^\ast) = 0$, so the $u^\ast$-th diagonal entry of $\mathbf{D}_{U|t^\ast}$ is zero. Letting $\mathbf{b}_u$ denote the $u$-th column of $\mathbf{B}$, we have
\[
\mathbf{B}\mathbf{D}_{U|t^\ast}
=
\sum_{u=1}^k \pr(U=u \mid T=t^\ast)\,\mathbf{b}_u\mathbf{e}_u^\top,
\]
where $\mathbf{e}_u$ is the $u$-th standard basis vector in $\mathbb{R}^k$. Since $\pr(U=u^\ast \mid T=t^\ast)=0$, the coefficient multiplying $\mathbf{b}_{u^\ast}$ vanishes, so that column direction is removed entirely. Because $\mathbf{B}$ has full column rank $k$, its columns are linearly independent, and once one of those independent directions is annihilated, the resulting matrix cannot still have rank $k$. Therefore $\rank(\mathbf{B}\mathbf{D}_{U|t^\ast})<k$.

Using the identity established at the beginning of the proof, $\range(\mathbf{M}_{ZX|t^\ast}^\top)=\range(\mathbf{B}\mathbf{D}_{U|t^\ast})$, we obtain $\dim\range(\mathbf{M}_{ZX|t^\ast}^\top) = \rank(\mathbf{B}\mathbf{D}_{U|t^\ast}) < k$. On the other hand, since $\mathbf{B}$ has full column rank, $\dim\range(\mathbf{B})=k$. Therefore
\[
\range(\mathbf{M}_{ZX|t^\ast}^\top)\subsetneq \range(\mathbf{B}),
\]
and in particular the treated and control row spaces cannot both coincide with $\range(\mathbf{B})$.
\end{proof}

\subsection{C. Proof of Theorem \ref{thm:spectral_identification}}
\label{app:proof_thm_spectral_identification_revised}

\begin{figure}[!htbp]
    \centering
    \includegraphics[width=0.92\textwidth]{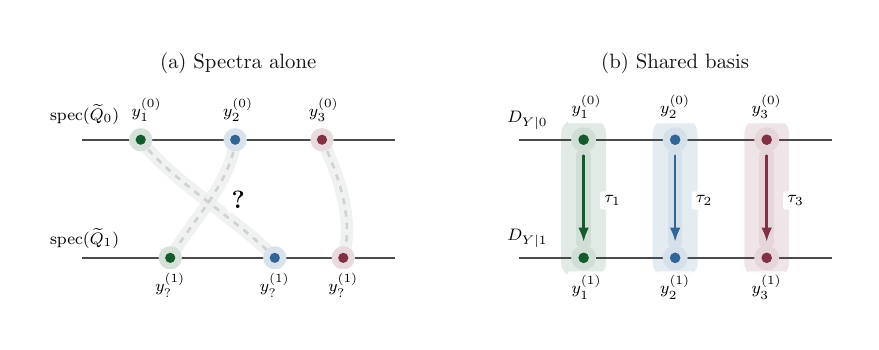}
    \caption{
    Counterfactual pairing through the common similarity basis.
    \textbf{Left:} The spectra of the treated and control operators separately reveal unordered potential-outcome values, so the classwise matching is not determined from the two spectra alone.
    \textbf{Right:} The shared similarity basis aligns the latent coordinates before differencing; hence
    $\Delta\widetilde{\mathbf{Q}}=\mathbf{R}^{-1}\mathbf{D}_{\tau}\mathbf{R}$
    and the eigenvalues recover the paired latent treatment effects.
    }
    \label{fig:counterfactual_pairing_geometry}
\end{figure}

\begin{proof}
Recall that $\mathbf{M}_{ZX|t} = \mathbf{C}_t \mathbf{B}^\top$ and $\mathbf{M}_{ZXY|t} = \mathbf{C}_t \mathbf{D}_{Y|t}\mathbf{B}^\top$. By Assumption~\ref{ass:richness} and strict positivity, $\mathbf{C}_t$ has full column rank $k$, while $\mathbf{B}^\top$ has full row rank $k$. Therefore Lemma~\ref{lem:pinv_product} applies to the product $\mathbf{C}_t \mathbf{B}^\top$ and yields $\mathbf{M}_{ZX|t}^\dagger = (\mathbf{B}^\top)^\dagger \mathbf{C}_t^\dagger$. Substituting this into the definition of the ambient quotient operator gives
\begin{align}
\mathbf{Q}_t
&=
\mathbf{M}_{ZX|t}^\dagger \mathbf{M}_{ZXY|t} \nonumber\\
&=
(\mathbf{B}^\top)^\dagger \mathbf{C}_t^\dagger
\big(\mathbf{C}_t \mathbf{D}_{Y|t}\mathbf{B}^\top\big) \nonumber\\
&=
(\mathbf{B}^\top)^\dagger
(\mathbf{C}_t^\dagger \mathbf{C}_t)
\mathbf{D}_{Y|t}\mathbf{B}^\top. \label{eq:ambient_operator_middle_revised}
\end{align}
Since $\mathbf{C}_t$ has full column rank, $\mathbf{C}_t^\dagger \mathbf{C}_t = \mathbf{I}_k$, and \eqref{eq:ambient_operator_middle_revised} simplifies to
\begin{equation}
\mathbf{Q}_t
=
(\mathbf{B}^\top)^\dagger \mathbf{D}_{Y|t}\mathbf{B}^\top.
\label{eq:ambient_operator_clean_revised}
\end{equation}
This is the claimed ambient representation.

We now examine the action of $\mathbf{Q}_t$ on the signal subspace determined by $\mathbf{B}$. Right-multiplying \eqref{eq:ambient_operator_clean_revised} by $(\mathbf{B}^\top)^\dagger$ and using the full-row-rank identity $\mathbf{B}^\top (\mathbf{B}^\top)^\dagger = \mathbf{I}_k$ yields
\begin{equation}
\mathbf{Q}_t (\mathbf{B}^\top)^\dagger
=
(\mathbf{B}^\top)^\dagger \mathbf{D}_{Y|t}.
\label{eq:ambient_intertwining_revised}
\end{equation}
Equation \eqref{eq:ambient_intertwining_revised} is an intertwining relation showing that $\range((\mathbf{B}^\top)^\dagger)$ is $\mathbf{Q}_t$-invariant: if $v = (\mathbf{B}^\top)^\dagger a$ for some $a \in \mathbb{R}^k$, then $\mathbf{Q}_t v = (\mathbf{B}^\top)^\dagger \mathbf{D}_{Y|t} a \in \range((\mathbf{B}^\top)^\dagger)$. Since $\mathbf{B}$ has full column rank, $\range((\mathbf{B}^\top)^\dagger)=\range(\mathbf{B})$, so $\range(\mathbf{B})$ is $\mathbf{Q}_t$-invariant. Moreover, \eqref{eq:ambient_intertwining_revised} shows that the restriction of $\mathbf{Q}_t$ to this $k$-dimensional signal subspace is similar to $\mathbf{D}_{Y|t}$.

We next identify the corresponding left eigenstructure directly. Left-multiplying \eqref{eq:ambient_operator_clean_revised} by $\mathbf{B}^\top$ and using $\mathbf{B}^\top (\mathbf{B}^\top)^\dagger = \mathbf{I}_k$, we obtain
\begin{equation}
\mathbf{B}^\top \mathbf{Q}_t
=
\mathbf{D}_{Y|t}\mathbf{B}^\top.
\label{eq:left_eigen_relation_revised}
\end{equation}
Writing \eqref{eq:left_eigen_relation_revised} row by row, the $u$-th row of $\mathbf{B}^\top$ satisfies
\[
\mathbf{b}_u^\top \mathbf{Q}_t
=
\E[Y^{(t)} \mid U=u]\,\mathbf{b}_u^\top,
\]
so for each $u=1,\dots,k$, the $u$-th row of $\mathbf{B}^\top$ is a left eigenvector of $\mathbf{Q}_t$ with eigenvalue $\E[Y^{(t)} \mid U=u]$. Because $\mathbf{B}^\top$ has rank $k$, its $k$ row directions span the full left signal space. Therefore the spectral values of $\mathbf{Q}_t$ on the invariant signal subspace $\range(\mathbf{B})$ are exactly the diagonal entries of $\mathbf{D}_{Y|t}$, counted with multiplicity, proving that the restriction of $\mathbf{Q}_t$ to the signal component has spectrum
\[
\left\{\E[Y^{(t)} \mid U=u]\right\}_{u=1}^k.
\]
\end{proof}

\subsection{D. Proof of Theorem \ref{thm:pairing} and Proposition \ref{prop:latent_prob}}

\begin{proof}
Let
\[
\mathbf{M}_{\mathrm{stack}}
=
\begin{bmatrix}
\mathbf{M}_{ZX|0}\\[-0.1em]
\mathbf{M}_{ZX|1}
\end{bmatrix}.
\]
By Theorem~\ref{thm:positivity},
\[
\begin{gathered}
\range(\mathbf{M}_{ZX|0}^\top) = \range(\mathbf{M}_{ZX|1}^\top) = \range(\mathbf{B}),\\[-0.2em]
\range(\mathbf{M}_{\mathrm{stack}}^\top) = \range(\mathbf{B}).
\end{gathered}
\]
Let $\mathbf{V}_k \in \mathbb{R}^{d_x \times k}$ have orthonormal columns spanning this space. Since $\mathbf{B}$ has full column rank and $\range(\mathbf{B})=\range(\mathbf{V}_k)$, we have
\[
\mathbf{B}^\top = \mathbf{B}^\top \mathbf{V}_k \mathbf{V}_k^\top,
\]
so taking $\mathbf{R} = \mathbf{B}^\top \mathbf{V}_k$ gives an invertible matrix with
\[
\mathbf{B}^\top = \mathbf{R}\mathbf{V}_k^\top.
\]

Now
\[
\begin{gathered}
\mathbf{M}_{ZX|t}\mathbf{V}_k = \mathbf{C}_t \mathbf{B}^\top \mathbf{V}_k = \mathbf{C}_t \mathbf{R},
\qquad
\mathbf{M}_{ZXY|t}\mathbf{V}_k = \mathbf{C}_t \mathbf{D}_{Y|t}\mathbf{R},\\[-0.2em]
(\mathbf{M}_{ZX|t}\mathbf{V}_k)^\dagger = \mathbf{R}^{-1}\mathbf{C}_t^\dagger.
\end{gathered}
\]
Since $\mathbf{C}_t$ has full column rank and $\mathbf{R}$ is invertible, Lemma~\ref{lem:pinv_product} yields the second identity. Hence
\begin{equation}
\begin{aligned}
\widetilde{\mathbf{Q}}_t
&=
(\mathbf{M}_{ZX|t}\mathbf{V}_k)^\dagger (\mathbf{M}_{ZXY|t}\mathbf{V}_k) \\
&=
\mathbf{R}^{-1}\mathbf{C}_t^\dagger
\big(\mathbf{C}_t \mathbf{D}_{Y|t}\mathbf{R}\big) \\
&=
\mathbf{R}^{-1}(\mathbf{C}_t^\dagger \mathbf{C}_t)\mathbf{D}_{Y|t}\mathbf{R}
=
\mathbf{R}^{-1}\mathbf{D}_{Y|t}\mathbf{R}.
\end{aligned}
\label{eq:compressed_similarity_revised}
\end{equation}

Subtracting the $t=0$ relation from the $t=1$ relation gives
\[
\begin{gathered}
\Delta \widetilde{\mathbf{Q}}
=
\widetilde{\mathbf{Q}}_1 - \widetilde{\mathbf{Q}}_0
=
\mathbf{R}^{-1}(\mathbf{D}_{Y|1}-\mathbf{D}_{Y|0})\mathbf{R},\\[-0.2em]
\mathbf{D}_{Y|1}-\mathbf{D}_{Y|0} = \diag(\tau(1),\dots,\tau(k)).
\end{gathered}
\]
Thus $\Delta \widetilde{\mathbf{Q}}$ is similar to the diagonal matrix above, and consequently the eigenvalues of $\Delta \widetilde{\mathbf{Q}}$ are exactly
\[
\tau(u) = \mathbb{E}[Y^{(1)} \mid U=u] - \mathbb{E}[Y^{(0)} \mid U=u],
\qquad
u=1,\dots,k.
\]

Now assume that the $\tau(u)$ are pairwise distinct, so each left eigenspace of $\Delta\widetilde{\mathbf{Q}}$ is one-dimensional. Since
\[
\mathbf{R}\Delta \widetilde{\mathbf{Q}} = (\mathbf{D}_{Y|1}-\mathbf{D}_{Y|0})\mathbf{R},
\]
the $u$-th row of $\mathbf{R}$ is a left eigenvector of $\Delta \widetilde{\mathbf{Q}}$ for eigenvalue $\tau(u)$. Therefore any matrix $\mathbf{W}$ whose rows are left eigenvectors of $\Delta \widetilde{\mathbf{Q}}$ must satisfy
\[
\mathbf{W} = \mathbf{S}\mathbf{R}
\]
for some nonsingular diagonal matrix $\mathbf{S}$. Multiplying on the right by $\mathbf{V}_k^\top$ gives
\[
\begin{gathered}
\mathbf{W}\mathbf{V}_k^\top
=
\mathbf{S}\mathbf{R}\mathbf{V}_k^\top
=
\mathbf{S}\mathbf{B}^\top,\\[-0.2em]
\mathbf{B}_{\mathrm{unscaled}}^\top \triangleq \mathbf{W}\mathbf{V}_k^\top,
\qquad
(\mathbf{B}_{\mathrm{unscaled}}^\top)_{u,:} = S_{uu}\,\mathbf{b}_u^\top.
\end{gathered}
\]

By assumption, $X_1 = 1$ almost surely, so
\[
b_{u,1} = \mathbb{E}[X_1 \mid U=u] = 1
\]
for every $u$. Hence the first coordinate of the $u$-th row of $\mathbf{B}_{\mathrm{unscaled}}^\top$ is
\[
(\mathbf{B}_{\mathrm{unscaled}}^\top)_{u1} = S_{uu}.
\]
Dividing the $u$-th row by its first entry removes the unknown scale factor $S_{uu}$ and recovers $\mathbf{b}_u^\top$. Thus $\mathbf{B}^\top$ is identified exactly.

Finally, let
\[
\mathbf{p} = (\pr(U=1),\dots,\pr(U=k))^\top.
\]
By the law of total expectation,
\[
\mathbb{E}[X]
=
\sum_{u=1}^k \mathbb{E}[X \mid U=u]\pr(U=u)
=
\mathbf{B}\mathbf{p}.
\]
Since $\mathbf{B}$ has full column rank, $\mathbf{B}^\dagger \mathbf{B} = \mathbf{I}_k$, and therefore
\[
\mathbf{p} = \mathbf{B}^\dagger \mathbb{E}[X].
\]
\end{proof}

\subsection{E. Proof of Theorem \ref{thm:moment_equivalence}}
\label{app:proof_thm_moment_equivalence}

\begin{proof}
Let $\mathbf{p} \triangleq (\pr(U=1),\dots,\pr(U=k))^\top$. From $\mathbf{B}^\top = \mathbf{R}\mathbf{V}_k^\top$ we obtain, by transposition, $\mathbf{B} = \mathbf{V}_k \mathbf{R}^\top$. Hence
\[
\begin{aligned}
\E[X] &= \mathbf{B}\mathbf{p} = \mathbf{V}_k \mathbf{R}^\top \mathbf{p},
\qquad
\mathbf{a}^\top
=
\E[X]^\top \mathbf{V}_k
=
\mathbf{p}^\top \mathbf{R}\mathbf{V}_k^\top \mathbf{V}_k
=
\mathbf{p}^\top \mathbf{R}.
\end{aligned}
\]

Next, because the first coordinate of $X$ is the constant $1$, each latent conditional mean vector satisfies $e_1^\top \E[X \mid U=u] = 1$. Equivalently, $\mathbf{B}^\top e_1 = \mathbf{1}_k$, where $\mathbf{1}_k \in \R^k$ is the all-ones vector. Using $\mathbf{B}^\top = \mathbf{R}\mathbf{V}_k^\top$, this gives
\[
\begin{aligned}
\mathbf{R}\mathbf{c}
&=
\mathbf{R}\mathbf{V}_k^\top e_1
=
\mathbf{B}^\top e_1
=
\mathbf{1}_k,
\qquad
\mathbf{c} = \mathbf{R}^{-1}\mathbf{1}_k.
\end{aligned}
\]

For any integer $\ell \ge 0$,
\[
\begin{aligned}
(\Delta \widetilde{\mathbf{Q}})^\ell
&=
(\mathbf{R}^{-1}\mathbf{D}_{\tau}\mathbf{R})^\ell
=
\mathbf{R}^{-1}\mathbf{D}_{\tau}^{\ell}\mathbf{R}.
\end{aligned}
\]
Consequently,
\[
\begin{aligned}
\mathbf{a}^\top (\Delta \widetilde{\mathbf{Q}})^\ell \mathbf{c}
&=
\mathbf{p}^\top \mathbf{R}
\left(
\mathbf{R}^{-1}\mathbf{D}_{\tau}^{\ell}\mathbf{R}
\right)
\mathbf{R}^{-1}\mathbf{1}_k \\
&=
\mathbf{p}^\top \mathbf{D}_{\tau}^{\ell}\mathbf{1}_k
=
\sum_{u=1}^k \pr(U=u)\tau(u)^\ell
=
\E[\tau(U)^\ell].
\end{aligned}
\]
This proves the moment identity.

For the generating function, let $|z| < \|\Delta \widetilde{\mathbf{Q}}\|_2^{-1}$. Then the Neumann series converges absolutely and
\[
\begin{aligned}
(\mathbf{I}_k - z \Delta \widetilde{\mathbf{Q}})^{-1}
&=
\sum_{\ell=0}^{\infty} z^\ell (\Delta \widetilde{\mathbf{Q}})^\ell.
\end{aligned}
\]
Left- and right-multiplying by $\mathbf{a}^\top$ and $\mathbf{c}$, and invoking the moment identity termwise, yields
\[
\begin{aligned}
\mathbf{a}^\top (\mathbf{I}_k - z \Delta \widetilde{\mathbf{Q}})^{-1}\mathbf{c}
&=
\sum_{\ell=0}^{\infty} z^\ell \E[\tau(U)^\ell]
=
\sum_{\ell=0}^{\infty} m_\ell z^\ell.
\end{aligned}
\]
The final claim follows because every polynomial functional of $\tau(U)$ is a finite linear combination of the moments $\{m_\ell\}_{\ell \ge 0}$.
\end{proof}

\subsection{F. Proof of Theorem \ref{thm:environment} and Proposition \ref{prop:homogeneity}}
\label{app:proof_thm_environment_homogeneity}

\begin{proof}[Proof of Theorem \ref{thm:environment}]
Fix an arbitrary environment $e \in \mathcal{E}$. By the assumed environment-specific factorization, for each treatment arm $t \in \{0,1\}$ we have
\[
\mathbf{M}_{ZX|t}^{(e)}
=
\mathbf{A}_t^{(e)} \mathbf{D}_{U|t}^{(e)} \mathbf{B}^\top,
\qquad
\mathbf{M}_{ZXY|t}^{(e)}
=
\mathbf{A}_t^{(e)} \mathbf{D}_{U|t}^{(e)} \mathbf{D}_{Y|t} \mathbf{B}^\top.
\]
By hypothesis, the matrix $\mathbf{V}_k^{(e)}$ spans the shared row space in environment $e$, and we may write
\[
\mathbf{B}^\top
=
\mathbf{R}^{(e)} (\mathbf{V}_k^{(e)})^\top
\]
for some nonsingular matrix $\mathbf{R}^{(e)} \in \mathbb{R}^{k \times k}$.

Now define, exactly as in the main text but with environment index $e$,
\[
\widetilde{\mathbf{Q}}_t^{(e)}
\triangleq
\big(\mathbf{M}_{ZX|t}^{(e)} \mathbf{V}_k^{(e)}\big)^\dagger
\big(\mathbf{M}_{ZXY|t}^{(e)} \mathbf{V}_k^{(e)}\big).
\]
Substituting the factorizations above into these projected matrices gives
\[
\mathbf{M}_{ZX|t}^{(e)} \mathbf{V}_k^{(e)}
=
\mathbf{A}_t^{(e)} \mathbf{D}_{U|t}^{(e)} \mathbf{B}^\top \mathbf{V}_k^{(e)},
\qquad
\mathbf{M}_{ZXY|t}^{(e)} \mathbf{V}_k^{(e)}
=
\mathbf{A}_t^{(e)} \mathbf{D}_{U|t}^{(e)} \mathbf{D}_{Y|t} \mathbf{B}^\top \mathbf{V}_k^{(e)}.
\]
Using the representation $\mathbf{B}^\top = \mathbf{R}^{(e)} (\mathbf{V}_k^{(e)})^\top$ and the orthonormality of the columns of $\mathbf{V}_k^{(e)}$, we obtain $\mathbf{B}^\top \mathbf{V}_k^{(e)} = \mathbf{R}^{(e)}$. Therefore
\[
\mathbf{M}_{ZX|t}^{(e)} \mathbf{V}_k^{(e)}
=
\mathbf{A}_t^{(e)} \mathbf{D}_{U|t}^{(e)} \mathbf{R}^{(e)},
\qquad
\mathbf{M}_{ZXY|t}^{(e)} \mathbf{V}_k^{(e)}
=
\mathbf{A}_t^{(e)} \mathbf{D}_{U|t}^{(e)} \mathbf{D}_{Y|t} \mathbf{R}^{(e)}.
\]

At this point the same algebra as in the proof of Theorem~\ref{thm:pairing} applies. Since $\mathbf{A}_t^{(e)} \mathbf{D}_{U|t}^{(e)}$ has full column rank and $\mathbf{R}^{(e)}$ is invertible, the pseudo-inverse-of-a-product identity yields
\[
\big(\mathbf{M}_{ZX|t}^{(e)} \mathbf{V}_k^{(e)}\big)^\dagger
=
\big(\mathbf{R}^{(e)}\big)^{-1}
\big(\mathbf{A}_t^{(e)} \mathbf{D}_{U|t}^{(e)}\big)^\dagger.
\]
Hence
\begin{align*}
\widetilde{\mathbf{Q}}_t^{(e)}
&=
\big(\mathbf{M}_{ZX|t}^{(e)} \mathbf{V}_k^{(e)}\big)^\dagger
\big(\mathbf{M}_{ZXY|t}^{(e)} \mathbf{V}_k^{(e)}\big) \\
&=
\big(\mathbf{R}^{(e)}\big)^{-1}
\big(\mathbf{A}_t^{(e)} \mathbf{D}_{U|t}^{(e)}\big)^\dagger
\big(\mathbf{A}_t^{(e)} \mathbf{D}_{U|t}^{(e)} \mathbf{D}_{Y|t} \mathbf{R}^{(e)}\big) \\
&=
\big(\mathbf{R}^{(e)}\big)^{-1}\mathbf{D}_{Y|t}\mathbf{R}^{(e)},
\end{align*}
because
\[
\big(\mathbf{A}_t^{(e)} \mathbf{D}_{U|t}^{(e)}\big)^\dagger
\big(\mathbf{A}_t^{(e)} \mathbf{D}_{U|t}^{(e)}\big)
=
\mathbf{I}_k.
\]

Subtracting the two treatment arms gives
\begin{align*}
\Delta \widetilde{\mathbf{Q}}^{(e)}
&\triangleq
\widetilde{\mathbf{Q}}_1^{(e)}-\widetilde{\mathbf{Q}}_0^{(e)} \\
&=
\big(\mathbf{R}^{(e)}\big)^{-1}\mathbf{D}_{Y|1}\mathbf{R}^{(e)}
-
\big(\mathbf{R}^{(e)}\big)^{-1}\mathbf{D}_{Y|0}\mathbf{R}^{(e)} \\
&=
\big(\mathbf{R}^{(e)}\big)^{-1}
\big(\mathbf{D}_{Y|1}-\mathbf{D}_{Y|0}\big)
\mathbf{R}^{(e)}
=
\big(\mathbf{R}^{(e)}\big)^{-1}
\mathbf{D}_{\tau}
\mathbf{R}^{(e)}.
\end{align*}

The stated spectral identities now follow immediately from similarity invariance of the spectrum. Since
\[
\widetilde{\mathbf{Q}}_t^{(e)}
=
\big(\mathbf{R}^{(e)}\big)^{-1}\mathbf{D}_{Y|t}\mathbf{R}^{(e)},
\]
the matrix \(\widetilde{\mathbf{Q}}_t^{(e)}\) is similar to the diagonal matrix \(\mathbf{D}_{Y|t}\), and therefore
\[
\spec(\widetilde{\mathbf{Q}}_t^{(e)})
=
\spec(\mathbf{D}_{Y|t})
=
\left\{\E[Y^{(t)} \mid U=u]\right\}_{u=1}^k,
\]
counted with multiplicity. Likewise, since
\[
\Delta \widetilde{\mathbf{Q}}^{(e)}
=
\big(\mathbf{R}^{(e)}\big)^{-1}\mathbf{D}_{\tau}\mathbf{R}^{(e)},
\]
we have
\[
\spec(\Delta \widetilde{\mathbf{Q}}^{(e)})
=
\spec(\mathbf{D}_{\tau})
=
\{\tau(u)\}_{u=1}^k,
\]
again counted with multiplicity.
\end{proof}

\begin{proof}[Proof of Proposition \ref{prop:homogeneity}]
Suppose first that
\[
\tau(1)=\cdots=\tau(k)=\tau_\star.
\]
By definition of \(\mathbf{D}_{\tau}\), this means that every diagonal entry of \(\mathbf{D}_{\tau}\) is equal to \(\tau_\star\), and hence \(\mathbf{D}_{\tau}=\tau_\star \mathbf{I}_k\). Now Theorem~\ref{thm:pairing} gives
\[
\Delta \widetilde{\mathbf{Q}}
=
\mathbf{R}^{-1}\mathbf{D}_{\tau}\mathbf{R}.
\]
Substituting \(\mathbf{D}_{\tau}=\tau_\star \mathbf{I}_k\) yields
\[
\begin{aligned}
\Delta \widetilde{\mathbf{Q}}
&=
\mathbf{R}^{-1}(\tau_\star \mathbf{I}_k)\mathbf{R}
=
\tau_\star \mathbf{R}^{-1}\mathbf{I}_k\mathbf{R}
=
\tau_\star \mathbf{I}_k,
\end{aligned}
\]
because scalar matrices commute with every matrix. Therefore \(\Delta \widetilde{\mathbf{Q}}=\tau_\star \mathbf{I}_k\).

Conversely, suppose that \(\Delta \widetilde{\mathbf{Q}}=\tau_\star \mathbf{I}_k\). Using again the similarity representation from Theorem~\ref{thm:pairing}, \(\Delta \widetilde{\mathbf{Q}}=\mathbf{R}^{-1}\mathbf{D}_{\tau}\mathbf{R}\). Conjugating both sides by \(\mathbf{R}\) gives
\[
\begin{aligned}
\mathbf{R}\,\Delta \widetilde{\mathbf{Q}}\,\mathbf{R}^{-1}
&=
\mathbf{R}\mathbf{R}^{-1}\mathbf{D}_{\tau}\mathbf{R}\mathbf{R}^{-1}
=
\mathbf{D}_{\tau}.
\end{aligned}
\]
Substituting \(\Delta \widetilde{\mathbf{Q}}=\tau_\star \mathbf{I}_k\), we obtain
\[
\begin{aligned}
\mathbf{D}_{\tau}
&=
\mathbf{R}\,(\tau_\star \mathbf{I}_k)\,\mathbf{R}^{-1}
=
\tau_\star \mathbf{R}\mathbf{I}_k\mathbf{R}^{-1}
=
\tau_\star \mathbf{I}_k,
\end{aligned}
\]
again using the fact that \(\tau_\star \mathbf{I}_k\) is a scalar matrix. Hence \(\mathbf{D}_{\tau}=\tau_\star \mathbf{I}_k\).

Since \(\mathbf{D}_{\tau}=\diag(\tau(1),\dots,\tau(k))\) is diagonal, the identity \(\mathbf{D}_{\tau}=\tau_\star \mathbf{I}_k\) is equivalent to the coordinatewise statement
\[
\tau(u)=\tau_\star
\qquad\text{for every }u \in \{1,\dots,k\}.
\]
This proves the equivalence
\[
\tau(1)=\cdots=\tau(k)=\tau_\star
\quad\Longleftrightarrow\quad
\Delta \widetilde{\mathbf{Q}}=\tau_\star \mathbf{I}_k.
\]

For the final identity, Theorem~\ref{thm:moment_equivalence} states that for every integer \(\ell \ge 0\),
\[
\begin{aligned}
m_\ell \triangleq \E[\tau(U)^\ell]
&=
\mathbf{a}^\top (\Delta \widetilde{\mathbf{Q}})^\ell \mathbf{c}.
\end{aligned}
\]
Specializing to \(\ell=1\) gives \(\E[\tau(U)] = \mathbf{a}^\top \Delta \widetilde{\mathbf{Q}}\,\mathbf{c}\). Under latent causal homogeneity, \(\tau(U)=\tau_\star\) almost surely, so \(\E[\tau(U)]=\tau_\star\). Combining the two displays yields
\[
\tau_\star
=
\mathbf{a}^\top \Delta \widetilde{\mathbf{Q}}\,\mathbf{c},
\]
which is exactly the claimed identity.
\end{proof}

\subsection{G. Proof of Theorem \ref{thm:sample_complexity}}
\label{app:proof_thm_sample_complexity_revised}

\begin{proof}
Let
\[
\Lambda_\eta\triangleq \log\frac{8(d_z+d_x)}{\eta},
\qquad
\varepsilon_{n,\eta}\triangleq
C_{\mathrm{mom}}L_{\max}
\sqrt{\frac{\Lambda_\eta}{n\pi}},
\qquad
L_{\max}\triangleq \max\{L_{ZX},L_{ZXY}\},
\]
where \(C_{\mathrm{mom}}\) is the constant from
Lemma~\ref{lem:conditional_moment_concentration}. Write
\[
\mathbf M_{\mathrm{stack}}
\triangleq
\begin{bmatrix}
\mathbf M_{ZX|0}\\
\mathbf M_{ZX|1}
\end{bmatrix},
\qquad
\widehat{\mathbf M}_{\mathrm{stack}}
\triangleq
\begin{bmatrix}
\widehat{\mathbf M}_{ZX|0}\\
\widehat{\mathbf M}_{ZX|1}
\end{bmatrix}.
\]
Let
\[
\sigma
\triangleq
\min\left\{
\sigma_k(\mathbf M_{\mathrm{stack}}),
\min_{t\in\{0,1\}}\sigma_k(\mathbf M_{ZX|t}\mathbf V_k)
\right\}>0,
\]
and define the fixed population quantities
\[
M_X\triangleq \max_{t\in\{0,1\}}\|\mathbf M_{ZX|t}\|_2,\qquad
M_Y\triangleq \max_{t\in\{0,1\}}\|\mathbf M_{ZXY|t}\|_2.
\]
For notational compactness put
\[
\Gamma_X\triangleq 1+C_{\mathrm W}\frac{M_X}{\sigma},
\qquad
\Gamma_Y\triangleq 1+C_{\mathrm W}\frac{M_Y}{\sigma},
\]
where \(C_{\mathrm W}\) is the universal constant in the Wedin subspace
bound used below, and set
\[
\kappa_{ZX}^{\star}
\triangleq
\Gamma_Y+6\Gamma_X\left(\frac{M_Y}{\sigma}+\Gamma_Y\right).
\]
This is a fixed conditioning factor depending only on the population moment
matrices and the singular-value margin \(\sigma\). Increasing the universal
sample-size constant in the theorem statement if necessary, assume throughout
that
\[
n\pi\ge c\Lambda_\eta
\qquad\text{and}\qquad
\varepsilon_{n,\eta}\le
\min\left\{\frac{\sigma}{4\sqrt 2},\,\frac{\sigma}{2\Gamma_X},\,\sigma\right\}.
\]

By Lemma~\ref{lem:conditional_moment_concentration}, with probability at least
\(1-\eta\), simultaneously for \(t=0,1\),
\[
\|\widehat{\mathbf M}_{ZX|t}-\mathbf M_{ZX|t}\|_2\le \varepsilon_{n,\eta},
\qquad
\|\widehat{\mathbf M}_{ZXY|t}-\mathbf M_{ZXY|t}\|_2\le \varepsilon_{n,\eta}.
\]
Work on this event. Then
\[
\widehat{\mathbf M}_{\mathrm{stack}}-\mathbf M_{\mathrm{stack}}
=
\begin{bmatrix}
\widehat{\mathbf M}_{ZX|0}-\mathbf M_{ZX|0}\\
\widehat{\mathbf M}_{ZX|1}-\mathbf M_{ZX|1}
\end{bmatrix},
\]
and hence
\[
\begin{aligned}
\|\widehat{\mathbf M}_{\mathrm{stack}}-\mathbf M_{\mathrm{stack}}\|_2^2
&=
\lambda_{\max}\!\left(
(\widehat{\mathbf M}_{ZX|0}-\mathbf M_{ZX|0})^\top
(\widehat{\mathbf M}_{ZX|0}-\mathbf M_{ZX|0})
\right. \\
&\qquad\left.
+
(\widehat{\mathbf M}_{ZX|1}-\mathbf M_{ZX|1})^\top
(\widehat{\mathbf M}_{ZX|1}-\mathbf M_{ZX|1})
\right) \\
&\le
\|\widehat{\mathbf M}_{ZX|0}-\mathbf M_{ZX|0}\|_2^2+
\|\widehat{\mathbf M}_{ZX|1}-\mathbf M_{ZX|1}\|_2^2
\le
2\varepsilon_{n,\eta}^2.
\end{aligned}
\]
Thus
\[
\|\widehat{\mathbf M}_{\mathrm{stack}}-\mathbf M_{\mathrm{stack}}\|_2
\le
\sqrt2\,\varepsilon_{n,\eta}.
\]

Let \(\mathbf V_k,\widehat{\mathbf V}_k\in\mathbb R^{d_x\times k}\) denote the
top \(k\) right singular-vector matrices of
\(\mathbf M_{\mathrm{stack}}\) and \(\widehat{\mathbf M}_{\mathrm{stack}}\).
By Theorem~\ref{thm:dimensionality}, \(\rank(\mathbf M_{\mathrm{stack}})=k\),
so the population right singular subspace is separated from the null space by
\(\sigma_k(\mathbf M_{\mathrm{stack}})\ge\sigma\). Since
\(\sqrt2\,\varepsilon_{n,\eta}\le\sigma/4\), Wedin's theorem \citep{wedin1972perturbation} gives an
orthogonal matrix \(\mathbf H\in\mathcal O(k)\) such that
\[
\|\widehat{\mathbf V}_k-\mathbf V_k\mathbf H\|_2
\le
C_{\mathrm W}
\frac{\|\widehat{\mathbf M}_{\mathrm{stack}}-\mathbf M_{\mathrm{stack}}\|_2}
{\sigma_k(\mathbf M_{\mathrm{stack}})}
\le
C_{\mathrm W}\frac{\varepsilon_{n,\eta}}{\sigma},
\]
after absorbing the factor \(\sqrt2\) into \(C_{\mathrm W}\).

For each \(t\in\{0,1\}\), define the aligned population matrices
\[
\mathbf P_t\triangleq \mathbf M_{ZX|t}\mathbf V_k\mathbf H,\qquad
\mathbf O_t\triangleq \mathbf M_{ZXY|t}\mathbf V_k\mathbf H,
\]
and the empirical matrices
\[
\widehat{\mathbf P}_t\triangleq \widehat{\mathbf M}_{ZX|t}\widehat{\mathbf V}_k,
\qquad
\widehat{\mathbf O}_t\triangleq \widehat{\mathbf M}_{ZXY|t}\widehat{\mathbf V}_k.
\]
Then, using \(\|\widehat{\mathbf V}_k\|_2=1\),
\[
\begin{aligned}
\widehat{\mathbf P}_t-\mathbf P_t
&=
(\widehat{\mathbf M}_{ZX|t}-\mathbf M_{ZX|t})\widehat{\mathbf V}_k
+
\mathbf M_{ZX|t}(\widehat{\mathbf V}_k-\mathbf V_k\mathbf H),\\
\|\widehat{\mathbf P}_t-\mathbf P_t\|_2
&\le
\varepsilon_{n,\eta}
+
M_X C_{\mathrm W}\frac{\varepsilon_{n,\eta}}{\sigma}
=
\Gamma_X\varepsilon_{n,\eta}.
\end{aligned}
\]
Similarly,
\[
\begin{aligned}
\widehat{\mathbf O}_t-\mathbf O_t
&=
(\widehat{\mathbf M}_{ZXY|t}-\mathbf M_{ZXY|t})\widehat{\mathbf V}_k
+
\mathbf M_{ZXY|t}(\widehat{\mathbf V}_k-\mathbf V_k\mathbf H),\\
\|\widehat{\mathbf O}_t-\mathbf O_t\|_2
&\le
\varepsilon_{n,\eta}
+
M_Y C_{\mathrm W}\frac{\varepsilon_{n,\eta}}{\sigma}
=
\Gamma_Y\varepsilon_{n,\eta}.
\end{aligned}
\]
Because \(\mathbf H\) is orthogonal,
\[
\sigma_{\min}(\mathbf P_t)
=
\sigma_k(\mathbf M_{ZX|t}\mathbf V_k\mathbf H)
=
\sigma_k(\mathbf M_{ZX|t}\mathbf V_k)
\ge
\sigma.
\]
Since \(\Gamma_X\varepsilon_{n,\eta}\le\sigma/2\), Lemma~\ref{lem:pinv_perturb}
applies and gives
\[
\|\widehat{\mathbf P}_t^\dagger-\mathbf P_t^\dagger\|_2
\le
\frac{6}{\sigma^2}\|\widehat{\mathbf P}_t-\mathbf P_t\|_2
\le
\frac{6\Gamma_X}{\sigma^2}\varepsilon_{n,\eta}.
\]
Also,
\[
\|\mathbf P_t^\dagger\|_2
=
\frac1{\sigma_{\min}(\mathbf P_t)}
\le
\frac1{\sigma},
\qquad
\|\widehat{\mathbf O}_t\|_2
\le
\|\mathbf O_t\|_2+\|\widehat{\mathbf O}_t-\mathbf O_t\|_2
\le
M_Y+\Gamma_Y\varepsilon_{n,\eta}
\le
M_Y+\Gamma_Y\sigma.
\]

Let
\[
\widetilde{\mathbf Q}_t
\triangleq
(\mathbf M_{ZX|t}\mathbf V_k)^\dagger
(\mathbf M_{ZXY|t}\mathbf V_k),
\qquad
\widehat{\widetilde{\mathbf Q}}_t
\triangleq
\widehat{\mathbf P}_t^\dagger\widehat{\mathbf O}_t.
\]
Since \(\mathbf P_t=(\mathbf M_{ZX|t}\mathbf V_k)\mathbf H\) and
\(\mathbf O_t=(\mathbf M_{ZXY|t}\mathbf V_k)\mathbf H\), the product
pseudo-inverse identity gives
\[
\mathbf P_t^\dagger
=
\mathbf H^\top(\mathbf M_{ZX|t}\mathbf V_k)^\dagger,
\qquad
\mathbf P_t^\dagger\mathbf O_t
=
\mathbf H^\top\widetilde{\mathbf Q}_t\mathbf H
\triangleq
\widetilde{\mathbf Q}_t^{\,\mathbf H}.
\]
Therefore
\[
\widehat{\widetilde{\mathbf Q}}_t-\widetilde{\mathbf Q}_t^{\,\mathbf H}
=
\widehat{\mathbf P}_t^\dagger\widehat{\mathbf O}_t
-
\mathbf P_t^\dagger\mathbf O_t
=
(\widehat{\mathbf P}_t^\dagger-\mathbf P_t^\dagger)\widehat{\mathbf O}_t
+
\mathbf P_t^\dagger(\widehat{\mathbf O}_t-\mathbf O_t),
\]
and consequently
\[
\begin{aligned}
\|\widehat{\widetilde{\mathbf Q}}_t-\widetilde{\mathbf Q}_t^{\,\mathbf H}\|_2
&\le
\|\widehat{\mathbf P}_t^\dagger-\mathbf P_t^\dagger\|_2
\|\widehat{\mathbf O}_t\|_2
+
\|\mathbf P_t^\dagger\|_2
\|\widehat{\mathbf O}_t-\mathbf O_t\|_2\\
&\le
\frac{6\Gamma_X}{\sigma^2}\varepsilon_{n,\eta}(M_Y+\Gamma_Y\sigma)
+
\frac{\Gamma_Y}{\sigma}\varepsilon_{n,\eta}\\
&=
\left[
6\Gamma_X\left(\frac{M_Y}{\sigma}+\Gamma_Y\right)+\Gamma_Y
\right]\frac{\varepsilon_{n,\eta}}{\sigma}
=
\kappa_{ZX}^{\star}\frac{\varepsilon_{n,\eta}}{\sigma}.
\end{aligned}
\]
Subtracting treatment arms yields, with
\[
\Delta\widehat{\widetilde{\mathbf Q}}
\triangleq
\widehat{\widetilde{\mathbf Q}}_1-\widehat{\widetilde{\mathbf Q}}_0,
\qquad
\Delta\widetilde{\mathbf Q}^{\,\mathbf H}
\triangleq
\widetilde{\mathbf Q}_1^{\,\mathbf H}-\widetilde{\mathbf Q}_0^{\,\mathbf H},
\]
the deterministic bound
\[
\|\Delta\widehat{\widetilde{\mathbf Q}}
-
\Delta\widetilde{\mathbf Q}^{\,\mathbf H}\|_2
\le
2\kappa_{ZX}^{\star}\frac{\varepsilon_{n,\eta}}{\sigma}.
\]
Absorbing the factor \(2\) into the universal constant, write
\[
\|\Delta\widehat{\widetilde{\mathbf Q}}
-
\Delta\widetilde{\mathbf Q}^{\,\mathbf H}\|_2
\le
C_\Delta\kappa_{ZX}^{\star}\frac{\varepsilon_{n,\eta}}{\sigma}.
\]

At population level, Theorem~\ref{thm:pairing} gives
\[
\Delta\widetilde{\mathbf Q}
=
\mathbf R^{-1}\mathbf D_\tau\mathbf R,
\qquad
\mathbf D_\tau=\diag(\tau(1),\ldots,\tau(k)).
\]
Since
\[
\Delta\widetilde{\mathbf Q}^{\,\mathbf H}
=
\mathbf H^\top\Delta\widetilde{\mathbf Q}\mathbf H
=
(\mathbf R\mathbf H)^{-1}\mathbf D_\tau(\mathbf R\mathbf H),
\]
define
\[
\mathbf S\triangleq \mathbf R\mathbf H,\qquad
\mathbf \Lambda\triangleq \mathbf D_\tau,\qquad
\mathbf E_Q\triangleq
\Delta\widehat{\widetilde{\mathbf Q}}
-
\Delta\widetilde{\mathbf Q}^{\,\mathbf H},
\qquad
\mathbf F\triangleq \mathbf S\mathbf E_Q\mathbf S^{-1}.
\]
Then
\[
\mathbf S\Delta\widehat{\widetilde{\mathbf Q}}\mathbf S^{-1}
=
\mathbf S(\Delta\widetilde{\mathbf Q}^{\,\mathbf H}+\mathbf E_Q)\mathbf S^{-1}
=
\mathbf \Lambda+\mathbf F.
\]
Because \(\mathbf H\) is orthogonal and
\(\mathbf B^\top=\mathbf R\mathbf V_k^\top\) with \(\mathbf V_k\) orthonormal,
the nonzero singular values of \(\mathbf R\) and \(\mathbf B^\top\) coincide;
hence
\[
\kappa(\mathbf S)=\kappa(\mathbf R\mathbf H)=\kappa(\mathbf R)
=\kappa(\mathbf B^\top)\triangleq \kappa_B.
\]
Therefore
\[
\|\mathbf F\|_2
\le
\|\mathbf S\|_2\|\mathbf S^{-1}\|_2\|\mathbf E_Q\|_2
=
\kappa_B\|\mathbf E_Q\|_2
\le
C_\Delta
\frac{\kappa_B\kappa_{ZX}^{\star}}{\sigma}\varepsilon_{n,\eta}.
\]
Set
\[
\rho_{n,\eta}
\triangleq
C_\Delta
\frac{\kappa_B\kappa_{ZX}^{\star}}{\sigma}\varepsilon_{n,\eta}.
\]

Since \(\Delta\widehat{\widetilde{\mathbf Q}}\) is similar to
\(\mathbf\Lambda+\mathbf F\), Lemma~\ref{lem:diag_localization} implies the
spectral inclusion
\[
\forall\,\widehat\tau\in\spec(\Delta\widehat{\widetilde{\mathbf Q}}),
\qquad
\min_{1\le u\le k}|\widehat\tau-\tau(u)|
\le
\|\mathbf F\|_2
\le
\rho_{n,\eta}.
\]
Moreover, if the latent treatment effects are pairwise separated and
\(2\rho_{n,\eta}<\delta_{\min}\), where
\[
\delta_{\min}\triangleq \min_{u\neq v}|\tau(u)-\tau(v)|,
\]
then the local counting part of Lemma~\ref{lem:diag_localization} gives exactly
one empirical eigenvalue in each disc
\[
\{z\in\mathbb C:|z-\tau(u)|<\delta_u/2\},
\qquad
\delta_u\triangleq \min_{v\neq u}|\tau(u)-\tau(v)|.
\]
Thus there exists a permutation \(\rho\in S_k\) such that
\[
\max_{u\in\{1,\ldots,k\}}
|\widehat\tau_{\rho(u)}-\tau(u)|
\le
\rho_{n,\eta}
=
C_\Delta
\frac{\kappa_B\kappa_{ZX}^{\star}}{\sigma}\varepsilon_{n,\eta}.
\]
Substituting the definition of \(\varepsilon_{n,\eta}\) yields
\[
\max_{u\in\{1,\ldots,k\}}
|\widehat\tau_{\rho(u)}-\tau(u)|
\le
C
\frac{\kappa_B\kappa_{ZX}^{\star}L_{\max}}{\sigma}
\sqrt{\frac{\log(8(d_z+d_x)/\eta)}{n\pi}}.
\]

It remains to prove the lifted left-eigenvector perturbation bound. Fix
\(u\) with \(\delta_u>0\), and assume
\[
4\rho_{n,\eta}<\delta_u,
\]
which is again enforced by increasing the sample-size constant in the theorem
statement. Lemma~\ref{lem:left_eig_diag}, applied to
\(\mathbf\Lambda+\mathbf F\), gives a left eigenvector
\(\widehat{\mathbf y}_u^\top\) associated with the matched empirical
eigenvalue and satisfying
\[
\sin\theta(\widehat{\mathbf y}_u^\top,\mathbf e_u^\top)
\le
C\frac{\|\mathbf F\|_2}{\delta_u}
\le
C
\frac{\kappa_B\kappa_{ZX}^{\star}}{\delta_u\sigma}
\varepsilon_{n,\eta}.
\]
Define the corresponding compressed left eigenvectors by
\[
\widehat{\mathbf w}_u^\top\triangleq \widehat{\mathbf y}_u^\top\mathbf S,
\qquad
\mathbf w_u^\top\mathbf H\triangleq \mathbf e_u^\top\mathbf S.
\]
Then \(\widehat{\mathbf w}_u^\top\) is a left eigenvector of
\(\Delta\widehat{\widetilde{\mathbf Q}}\), while
\(\mathbf w_u^\top\mathbf H\) is the \(\mathbf H\)-aligned population left
eigenvector. Lemma~\ref{lem:angle_linear_map} gives
\[
\begin{aligned}
\sin\theta(\widehat{\mathbf w}_u^\top,\mathbf w_u^\top\mathbf H)
&=
\sin\theta(\widehat{\mathbf y}_u^\top\mathbf S,\mathbf e_u^\top\mathbf S)\\
&\le
\kappa(\mathbf S)
\sin\theta(\widehat{\mathbf y}_u^\top,\mathbf e_u^\top)\\
&\le
C
\frac{\kappa_B^2\kappa_{ZX}^{\star}}{\delta_u\sigma}
\varepsilon_{n,\eta}.
\end{aligned}
\]

Now lift from compressed coordinates back to the proxy row space. Choose unit
representatives of the one-dimensional eigenspaces so that
\[
\|\widehat{\mathbf w}_u\|_2=\|\mathbf w_u\mathbf H\|_2=1
\]
and their phase/sign is aligned. Then
\[
\|\widehat{\mathbf w}_u^\top-\mathbf w_u^\top\mathbf H\|_2
\le
\sqrt2\,
\sin\theta(\widehat{\mathbf w}_u^\top,\mathbf w_u^\top\mathbf H).
\]
Define the lifted rows
\[
\widehat{\mathbf b}_u^\top
\triangleq
\widehat{\mathbf w}_u^\top\widehat{\mathbf V}_k^\top,
\qquad
\mathbf b_u^\top
\triangleq
\mathbf w_u^\top\mathbf V_k^\top.
\]
Because \(\widehat{\mathbf V}_k\) and \(\mathbf V_k\mathbf H\) have orthonormal
columns, right multiplication by their transposes is an isometry on
\(\mathbb C^k\). Hence the lifted representatives also have unit norm. Insert
the aligned term:
\[
(\mathbf w_u^\top\mathbf H)(\mathbf V_k\mathbf H)^\top
=
\mathbf w_u^\top\mathbf H\mathbf H^\top\mathbf V_k^\top
=
\mathbf w_u^\top\mathbf V_k^\top
=
\mathbf b_u^\top.
\]
Therefore
\[
\begin{aligned}
\|\widehat{\mathbf b}_u^\top-\mathbf b_u^\top\|_2
&=
\left\|
\widehat{\mathbf w}_u^\top\widehat{\mathbf V}_k^\top
-
(\mathbf w_u^\top\mathbf H)(\mathbf V_k\mathbf H)^\top
\right\|_2\\
&\le
\|(\widehat{\mathbf w}_u^\top-\mathbf w_u^\top\mathbf H)
\widehat{\mathbf V}_k^\top\|_2
+
\|(\mathbf w_u^\top\mathbf H)
(\widehat{\mathbf V}_k^\top-(\mathbf V_k\mathbf H)^\top)\|_2\\
&\le
\|\widehat{\mathbf w}_u^\top-\mathbf w_u^\top\mathbf H\|_2
+
\|\mathbf w_u^\top\mathbf H\|_2
\|\widehat{\mathbf V}_k-\mathbf V_k\mathbf H\|_2\\
&\le
\sqrt2\,
\sin\theta(\widehat{\mathbf w}_u^\top,\mathbf w_u^\top\mathbf H)
+
C_{\mathrm W}\frac{\varepsilon_{n,\eta}}{\sigma}.
\end{aligned}
\]
Since both lifted representatives have unit norm,
\[
\sin\theta(\widehat{\mathbf b}_u^\top,\mathbf b_u^\top)
\le
\|\widehat{\mathbf b}_u^\top-\mathbf b_u^\top\|_2.
\]
Combining the last two displays gives
\[
\sin\theta(\widehat{\mathbf b}_u^\top,\mathbf b_u^\top)
\le
C
\left(
\frac{\kappa_B^2\kappa_{ZX}^{\star}}{\delta_u\sigma}
+
\frac1{\sigma}
\right)
\varepsilon_{n,\eta}.
\]
Substituting \(\varepsilon_{n,\eta}\) yields
\[
\sin\theta(\widehat{\mathbf b}_u^\top,\mathbf b_u^\top)
\le
C L_{\max}
\left(
\frac{\kappa_B^2\kappa_{ZX}^{\star}}{\delta_u\sigma}
+
\frac1{\sigma}
\right)
\sqrt{\frac{\log(8(d_z+d_x)/\eta)}{n\pi}},
\]
which proves the claimed lifted-row perturbation bound.
\end{proof}

\subsection{H. Proof of Theorem \ref{thm:B_p_recovery}}

\begin{proof}
Apply Theorem~\ref{thm:sample_complexity} with failure probability \(\eta/2\), and apply Lemma~\ref{lem:mean_concentration} with failure probability \(\eta/2\). After increasing universal constants, both events hold simultaneously with probability at least \(1-\eta\), and all logarithmic factors are bounded by \(\Lambda_\eta^\star\).

On the event from Theorem~\ref{thm:sample_complexity}, the condition
\(r_{\tau,n,\eta}<\delta_{\min}/2\) gives a permutation \(\rho\in S_k\) such that
\[
\max_{1\le u\le k}
|\widehat\tau_{\rho(u)}-\tau(u)|
\le
r_{\tau,n,\eta}.
\]
Moreover, for each \(u\), the matched empirical eigenvalue lies in an isolation disc around the real population eigenvalue \(\tau(u)\). Since
\(\Delta\widehat{\widetilde{\mathbf Q}}\) is real, non-real eigenvalues occur in conjugate pairs. The isolation disc contains exactly one eigenvalue, so the matched empirical eigenvalue must be real. Hence the corresponding empirical left eigenvector can be chosen real.

Let
\[
\mathbf q_u^\top
\triangleq
\frac{\mathbf b_u^\top}{\|\mathbf b_u\|_2},
\]
where \(\mathbf b_u^\top\) is the \(u\)-th row of the true matrix
\(\mathbf B^\top\). Since \(b_{u1}=1\),
\[
|q_{u1}|
=
\frac{1}{\|\mathbf b_u\|_2}
\ge
\alpha_{\mathrm{anc}}.
\]
The lifted-row part of Theorem~\ref{thm:sample_complexity}, with
\(\delta_u\ge \delta_{\min}\), gives
\[
\sin\theta\left(
\widehat{\mathbf q}_{\rho(u)}^\top,
\mathbf q_u^\top
\right)
\le
r_{B,n,\eta}
\]
after choosing the empirical lifted row as a unit-norm representative. Align the real sign of
\(\widehat{\mathbf q}_{\rho(u)}^\top\) so that its inner product with
\(\mathbf q_u^\top\) is nonnegative. For unit vectors with aligned sign,
\[
\left\|
\widehat{\mathbf q}_{\rho(u)}^\top-\mathbf q_u^\top
\right\|_2
\le
\sqrt2\,
\sin\theta\left(
\widehat{\mathbf q}_{\rho(u)}^\top,
\mathbf q_u^\top
\right)
\le
\sqrt2\,r_{B,n,\eta}.
\]
The condition \(r_{B,n,\eta}\le c_0\alpha_{\mathrm{anc}}\), with \(c_0\) sufficiently small, ensures
\[
\left\|
\widehat{\mathbf q}_{\rho(u)}^\top-\mathbf q_u^\top
\right\|_2
\le
\frac{|q_{u1}|}{2}.
\]
Lemma~\ref{lem:anchor_normalization} then gives
\[
\begin{aligned}
\left\|
\widehat{\mathbf b}_{\rho(u)}^\top-\mathbf b_u^\top
\right\|_2
&=
\left\|
\frac{\widehat{\mathbf q}_{\rho(u)}^\top}
{\widehat q_{\rho(u),1}}
-
\frac{\mathbf q_u^\top}{q_{u1}}
\right\|_2  \\
&\le
4|q_{u1}|^{-2}
\left\|
\widehat{\mathbf q}_{\rho(u)}^\top-\mathbf q_u^\top
\right\|_2  \\
&\le
4\sqrt2\,\alpha_{\mathrm{anc}}^{-2}r_{B,n,\eta}.
\end{aligned}
\]
Choosing \(C_{\mathrm{anc}}\ge 4\sqrt2\) proves
\[
\max_{1\le u\le k}
\left\|
\widehat{\mathbf b}_{\rho(u)}^\top-\mathbf b_u^\top
\right\|_2
\le
\varepsilon_{B,n,\eta}.
\]

Let \(\widehat{\mathbf B}_\rho^\top\) denote the row-permuted estimator whose \(u\)-th row is
\(\widehat{\mathbf b}_{\rho(u)}^\top\). Then
\[
\|\widehat{\mathbf B}_\rho-\mathbf B\|_2
=
\|\widehat{\mathbf B}_\rho^\top-\mathbf B^\top\|_2
\le
\|\widehat{\mathbf B}_\rho^\top-\mathbf B^\top\|_F
\le
\sqrt{k}\,\varepsilon_{B,n,\eta}.
\]

It remains to control the mixture weights. By assumption,
\[
\sqrt{k}\,\varepsilon_{B,n,\eta}\le \sigma_B/2,
\]
so Weyl's inequality gives
\[
\sigma_k(\widehat{\mathbf B}_\rho)
\ge
\sigma_k(\mathbf B)-\|\widehat{\mathbf B}_\rho-\mathbf B\|_2
\ge
\sigma_B/2.
\]
Therefore
\[
\|\widehat{\mathbf B}_\rho^\dagger\|_2
\le
\frac{2}{\sigma_B}.
\]
Applying Lemma~\ref{lem:pinv_perturb} with
\[
\mathbf P=\mathbf B,
\qquad
\widehat{\mathbf P}=\widehat{\mathbf B}_\rho,
\qquad
\|\widehat{\mathbf B}_\rho-\mathbf B\|_2
\le
\sqrt{k}\,\varepsilon_{B,n,\eta},
\]
yields
\[
\|\widehat{\mathbf B}_\rho^\dagger-\mathbf B^\dagger\|_2
\le
\frac{6\sqrt{k}}{\sigma_B^2}
\varepsilon_{B,n,\eta}.
\]

Since \(\bm\mu_X=\mathbf B\bm p\) and \(\mathbf B\) has full column rank,
\[
\bm p=\mathbf B^\dagger\bm\mu_X.
\]
Hence
\[
\begin{aligned}
\widehat{\bm p}_{\rho,\mathrm{raw}}-\bm p
&=
\widehat{\mathbf B}_\rho^\dagger\widehat{\bm\mu}_X
-
\mathbf B^\dagger\bm\mu_X  \\
&=
\widehat{\mathbf B}_\rho^\dagger
(\widehat{\bm\mu}_X-\bm\mu_X)
+
(\widehat{\mathbf B}_\rho^\dagger-\mathbf B^\dagger)\bm\mu_X.
\end{aligned}
\]
Taking norms and using Lemma~\ref{lem:mean_concentration},
\[
\begin{aligned}
\|\widehat{\bm p}_{\rho,\mathrm{raw}}-\bm p\|_2
&\le
\|\widehat{\mathbf B}_\rho^\dagger\|_2
\|\widehat{\bm\mu}_X-\bm\mu_X\|_2
+
\|\widehat{\mathbf B}_\rho^\dagger-\mathbf B^\dagger\|_2
\|\bm\mu_X\|_2 \\
&\le
\frac{2}{\sigma_B}
\varepsilon_{\mu,n,\eta}
+
\frac{6M_\mu\sqrt{k}}{\sigma_B^2}
\varepsilon_{B,n,\eta}.
\end{aligned}
\]
Finally, because \(\bm p\in\Delta^{k-1}\), Lemma~\ref{lem:simplex_projection} gives
\[
\|\widehat{\bm p}_{\rho}-\bm p\|_2
=
\left\|
\Pi_{\Delta^{k-1}}(\widehat{\bm p}_{\rho,\mathrm{raw}})
-
\Pi_{\Delta^{k-1}}(\bm p)
\right\|_2
\le
\|\widehat{\bm p}_{\rho,\mathrm{raw}}-\bm p\|_2.
\]
Combining the last two displays proves the claimed mixture-weight bound.
\end{proof}

\section{Additional Experiments}
\label{app:no_truncation}

The base SPO moment chain of \citet{mazaheri2024synthetic} is constructed
around a square $k\times k$ proxy moment system. In our overcomplete regime
$d_z = d_x = k+3$, two natural inputs are available: restrict to the first $k$ proxy coordinates
and form $\widehat{\mathbf M}_{ZX|t}, \widehat{\mathbf M}_{ZXY|t}\in\mathbb{R}^{k\times k}$
(the baseline reported in Section \ref{sec:experiments}), or feed the full
$(k+3)\times(k+3)$ moment matrices to the algorithm without truncation. This appendix reports a
three-way comparison that includes this no-truncation variant.

\paragraph{Population equivalence, finite-sample divergence.}
At the population level the two inputs yield the same scalar moment sequence.
Writing $\Delta\mathbf{Q}_{\mathrm{amb}} = (\mathbf{B}^\top)^\dagger\mathbf{D}_\tau\mathbf{B}^\top$, the identity $\mathbf{B}^\top(\mathbf{B}^\top)^\dagger = \mathbf{I}_k$ gives $(\Delta\mathbf{Q}_{\mathrm{amb}})^\ell = (\mathbf{B}^\top)^\dagger\mathbf{D}_\tau^\ell\mathbf{B}^\top$ for $\ell\ge 1$, and hence $\mathbf{a}^\top(\Delta\mathbf{Q}_{\mathrm{amb}})^\ell\mathbf{c} = \mathbb{E}[\tau(U)^\ell]$, exactly as in the compressed setting
(Theorem \ref{thm:moment_equivalence}). In finite samples, however, the population
$\mathbf{M}_{ZX|t}$ has rank exactly $k$, so the empirical $(k+3)\times(k+3)$
matrix is full-rank only because of noise. Its $d-k = 3$ smallest singular
values are of order $\sigma_{\min}\sim 1/\sqrt{n_t}$, and inverting it
amplifies noise from those null directions by the same factor, contaminating
the moment sequence consumed by the Hankel pencil.

\paragraph{Empirical comparison.}
Figure \ref{fig:three_way_heatmaps} reports the three-way comparison on the same
$(k, N)$ grid used in Section \ref{sec:experiments_eigenvalues}, with all panels on
a shared logarithmic color scale. The no-truncation variant (right) is
uniformly the worst of the three, often by an order of magnitude over the
truncated baseline (center). At $k=3$, $N=25{,}000$, the spectral,
truncated, and no-truncation estimators achieve median absolute errors of
$0.026$, $0.45$, and $0.90$ respectively. Figure \ref{fig:three_way_convergence}
shows the convergence behavior at $k=3$: the spectral estimator follows the
predicted $\mathcal{O}(n^{-1/2})$ rate, the truncated baseline is noisy and
roughly flat at much higher error, and the no-truncation variant remains
higher still and non-monotone. The oscillations are driven by occasional
near-collinearity in the empirical $(k+3)\times(k+3)$ moment matrix, which
spikes the inversion error in particular trials rather than averaging out
with sample size. Figure \ref{fig:hist_base_full} extends the eigenspace
distributional view of Section \ref{sec:experiments_eigenspaces}: where the
truncated baseline of Figure \ref{fig:hist_base} retains weak shoulders near the
true effects $\{-2, 0, 2\}$, the no-truncation variant displays mass smeared
across the entire interval $[-4, 4]$. Table \ref{tab:three_way} reports the
underlying medians numerically; the Full\,/\,Spectral ratio at
$N = 25{,}000$ ranges from $6.2\times$ at $k = 2$ to $35.3\times$ at $k = 3$,
and stays between $12$ and $32\times$ for $k\in\{4, 5, 6\}$.

\paragraph{Interpretation.}
The picture isolates two distinct contributions of the spectral framework's
SVD step. First, it adaptively selects the $k$-dimensional signal subspace,
which separates it from a fixed first-$k$ truncation; this is the
comparison reported in the main paper. Second, by working with the projected
$k\times k$ matrices throughout, it sidesteps the ill-conditioned ambient
$d\times d$ inversion that drives the no-truncation variant. The first
effect is what separates the green panel from the red panel in
Figure \ref{fig:three_way_heatmaps}; the second is what separates the red panel
from the orange. Both effects vanish at the population level by
Theorem \ref{thm:moment_equivalence}, so the entire empirical gap is a
finite-sample story driven by how the noise along the $d-k$ null directions
of $\widehat{\mathbf M}_{ZX|t}$ propagates through each method. In
overcomplete proxy regimes, then, the truncation built into the original
method is best understood as a noise-control step rather than a structural
feature of the algorithm.

\begin{figure}[!htbp]
    \centering
    \includegraphics[width=1.0\textwidth]{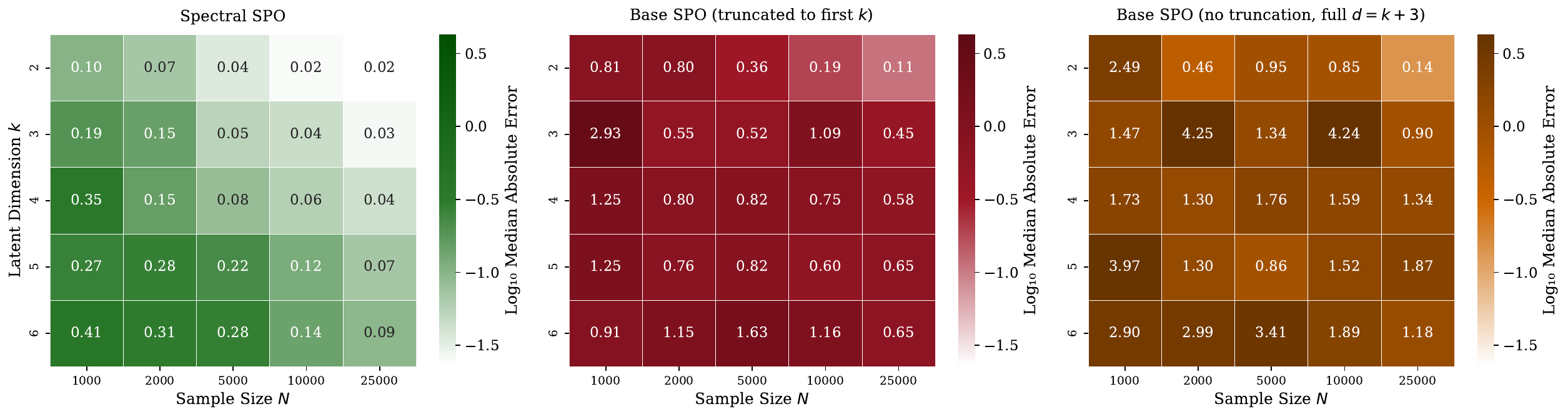}
    \caption{
    Three-way heatmap comparison on the $(k, N)$ grid of
    Section \ref{sec:experiments_eigenvalues}, with all panels on a shared
    logarithmic color scale. \textbf{Left:} Spectral SPO. \textbf{Center:}
    base SPO truncated to the first $k$ proxy coordinates. \textbf{Right:}
    base SPO applied to the full $d = k+3$ moment matrices without
    truncation. The no-truncation variant is uniformly worse than even the
    truncated baseline.
    }
    \label{fig:three_way_heatmaps}
\end{figure}

\begin{figure}[!htbp]
    \centering
    \includegraphics[width=1.0\textwidth]{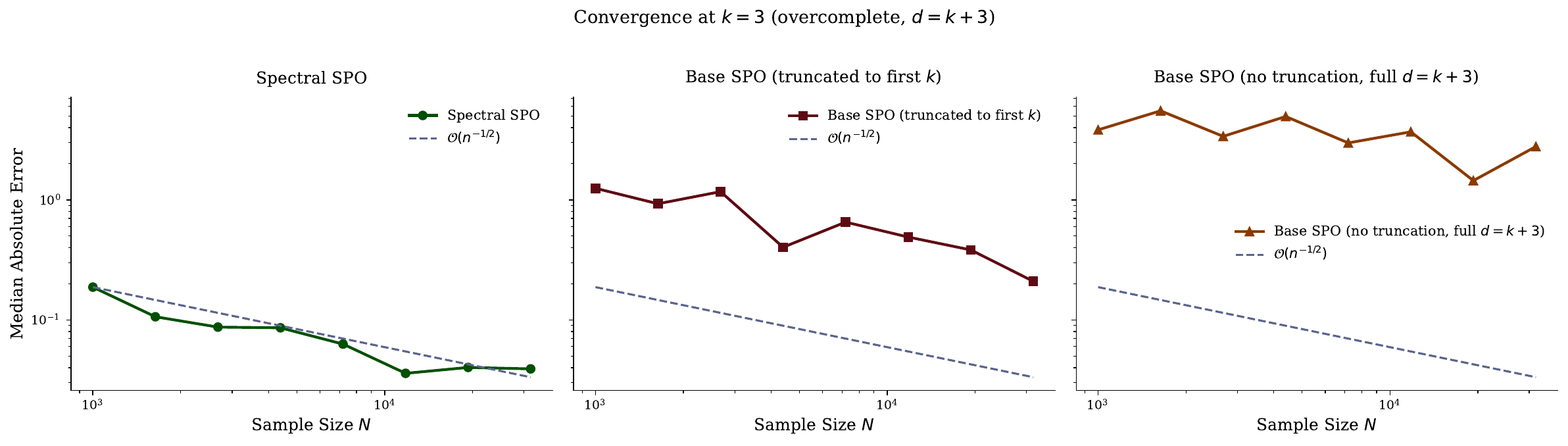}
    \caption{
    Convergence at $k = 3$ and $d_z = d_x = k + 3$. The spectral estimator
    follows the predicted $\mathcal{O}(n^{-1/2})$ rate; base SPO truncated to
    the first $k$ coordinates is noisy and roughly flat at much higher
    error; base SPO without truncation is higher still and non-monotone.
    }
    \label{fig:three_way_convergence}
\end{figure}

\begin{figure}[!htbp]
    \centering
    \includegraphics[width=1.0\textwidth]{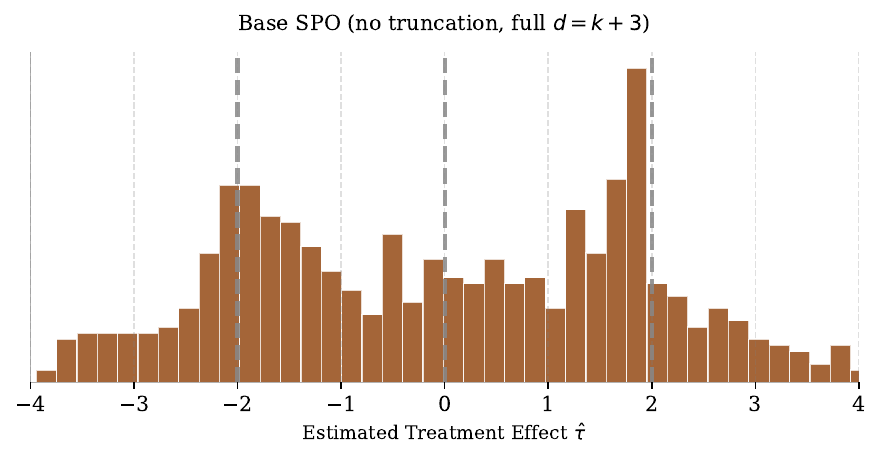}
    \caption{
    Base SPO without truncation, eigenvalue estimates across 300 trials
    with $k = 3$, $N = 5000$, $\tau \in \{-2, 0, 2\}$. The empirical
    distribution is smeared across $[-4, 4]$, in contrast to the truncated
    variant, which retains weak shoulders near the
    true support.
    }
    \label{fig:hist_base_full}
\end{figure}

\begin{figure}[!htbp]
    \centering
    \includegraphics[width=1.0\textwidth]{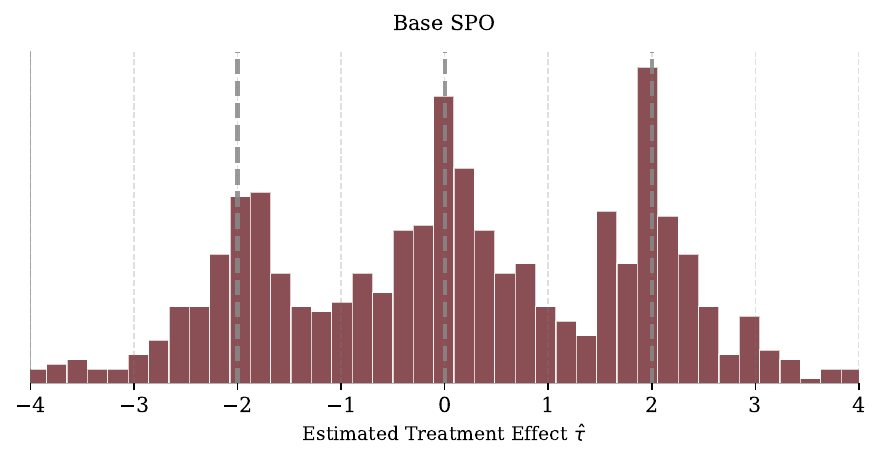}
    \caption{
    Base SPO eigenvalue estimates across 300 trials. \textcolor{red}{As shown earlier in the main paper and reproduced here alongside Fig. \ref{fig:hist_base_full} for direct comparison}, truncation retains discernible shoulders near the true effects $\{-2,0,2\}$.}
    \label{fig:hist}
\end{figure}

\begin{table}[htbp]
  \centering
  \caption{
  Three-way median absolute eigenvalue error in the overcomplete proxy
  regime $d_z = d_x = k+3$. Lower is better.
  \emph{Spec} = Spectral estimator;
  \emph{Trunc} = Base SPO truncated to the first $k$ proxy coordinates;
  \emph{Full} = Base SPO applied to the full $(k+3)\times(k+3)$ moment
  matrices without truncation. Bold entries indicate the best method at
  each sample size; the final column gives the Full\,/\,Spectral ratio at
  $N = 25{,}000$.
  }
  \label{tab:three_way}
  \setlength{\tabcolsep}{4pt}
  \renewcommand{\arraystretch}{1.12}
  \begin{tabular}{c ccc ccc ccc c}
    \toprule
    & \multicolumn{3}{c}{$N=1{,}000$}
    & \multicolumn{3}{c}{$N=5{,}000$}
    & \multicolumn{3}{c}{$N=25{,}000$}
    & \cellcolor{gray!10}{} \\
    \cmidrule(lr){2-4}
    \cmidrule(lr){5-7}
    \cmidrule(lr){8-10}
    $k$
    & Spec & Trunc & Full
    & Spec & Trunc & Full
    & Spec & Trunc & Full
    & \cellcolor{gray!10}{Gain} \\
    \midrule
    \rowcolor{gray!3}
    $2$ & \textbf{0.102} & 0.810 & 2.486 & \textbf{0.035} & 0.357 & 0.951 & \textbf{0.023} & 0.110 & 0.140 & \cellcolor{green!9}{$6.2\times$} \\
    $3$ & \textbf{0.191} & 2.931 & 1.465 & \textbf{0.054} & 0.524 & 1.336 & \textbf{0.026} & 0.445 & 0.905 & \cellcolor{green!21}{$35.3\times$} \\
    \rowcolor{gray!3}
    $4$ & \textbf{0.348} & 1.248 & 1.731 & \textbf{0.083} & 0.815 & 1.759 & \textbf{0.042} & 0.584 & 1.340 & \cellcolor{green!20}{$31.7\times$} \\
    $5$ & \textbf{0.273} & 1.253 & 3.970 & \textbf{0.219} & 0.820 & 0.864 & \textbf{0.069} & 0.653 & 1.868 & \cellcolor{green!18}{$27.2\times$} \\
    \rowcolor{gray!3}
    $6$ & \textbf{0.408} & 0.907 & 2.896 & \textbf{0.283} & 1.631 & 3.414 & \textbf{0.094} & 0.654 & 1.184 & \cellcolor{green!12}{$12.6\times$} \\
    \bottomrule
  \end{tabular}
\end{table}

\end{document}